\documentclass{article} 
\usepackage{iclr2027_conference,times}

\usepackage{amsmath,amsfonts,bm}

\def\eqref#1{equation~\ref{#1}}

\def\1{\bm{1}}

\DeclareMathAlphabet{\mathsfit}{\encodingdefault}{\sfdefault}{m}{sl}
\SetMathAlphabet{\mathsfit}{bold}{\encodingdefault}{\sfdefault}{bx}{n}

\usepackage{graphicx}
\usepackage[table]{xcolor}
\usepackage{booktabs}
\usepackage{multirow}

\usepackage{amsthm}

\usepackage{algorithm}
\usepackage{algorithmic}

\usepackage{url}
\usepackage{hyperref}

\newtheorem{theorem}{Theorem}
\newtheorem{lemma}[theorem]{Lemma}

\newtheorem{assumption}{Assumption}

\newtheorem{remark}{Remark}

\newcommand{\accup}{\textcolor{green!60!black}{$\uparrow$}}
\newcommand{\commdown}{\textcolor{red}{$\downarrow$}}

\title{Latent Information Sharing for Accelerating Federated Learning}

\author{%
Seungjun Lee\textsuperscript{1},
Ensieh Khazaei\textsuperscript{2},
Dimitrios Hatzinakos\textsuperscript{2}\\
Baturalp Buyukates\textsuperscript{3},
Sunwoo Lee\textsuperscript{1}\thanks{Corresponding author.}\\[0.5em]
\textsuperscript{1}Inha University\\
\textsuperscript{2}University of Toronto\\
\textsuperscript{3}University of Birmingham
}

\iclrfinalcopy 
\begin{document}

\maketitle
\lhead{Preprint}

\begin{abstract}
Federated learning (FL) is a communication-efficient distributed learning paradigm.
However, client drift remains one of the most critical challenges, hindering the efficient training of a global model.
In this study, we propose a novel latent information sharing scheme that directly mitigates data heterogeneity across clients.
Our theoretical and empirical results show that sharing a small amount of hidden-layer activations significantly improves training efficiency while preserving convergence guarantees and data privacy.
Furthermore, we compare our method with existing FL approaches designed to address client drift, including FedProx, SCAFFOLD, FedPVR, FedProto, and SplitFed, and demonstrate superior model accuracy under a fixed round budget without incurring excessive communication overhead.
Overall, this work presents a promising new knowledge aggregation scheme and provides a comprehensive analysis of the impact of activation sharing on federated optimization.
\end{abstract}

\section {Introduction} \label{sec:1}

While federated learning (FL) enables privacy-preserving and communication-efficient aggregation across multiple data holders, referred to as clients, its training efficiency largely depends on the degree of data heterogeneity among local datasets.
Severe data heterogeneity leads to client drift, one of the most well-known challenges in FL, which significantly hinders convergence.
To enable practical large-scale FL applications, it is crucial to address client drift without sacrificing the benefits of communication-efficient distributed learning.

To mitigate data heterogeneity, prior works such as FedProx~\cite{li2020federated} and FedDyn~\cite{durmus2021federated} introduce proximal regularization to stabilize local updates.
While these regularization approaches alleviate client drift to some extent, they do not fundamentally resolve discrepancies across clients, as they operate independently without explicit coordination.
Other studies address data heterogeneity by reducing the stochastic variance of gradients, such as SCAFFOLD~\cite{karimireddy2020scaffold}, FedVARP~\cite{jhunjhunwala2022fedvarp}, and FedMDVR~\cite{zhang2025federated}.
Although these methods effectively alleviate client drift, they often incur substantial memory and communication overhead, limiting their practicality in large-scale deployments.
FedLAMA~\cite{lee2023layer} proposes a structural approach based on layer-wise adaptive aggregation to implicitly mitigate client discrepancy.
However, it still relies on parameter-space coordination and does not explicitly align intermediate representations across clients.
These limitations motivate exploring alternative forms of collaboration beyond parameter-space coordination.

In this study, we propose \textit{Activation Sharing Federated Learning} (ASFL), a latent information sharing approach that tackles client drift from a novel representation-based perspective. Building upon the conventional model aggregation scheme, our approach additionally exchanges intermediate activations from a critical hidden layer.
By directly mitigating data heterogeneity through this exchange, our method accelerates convergence, thereby achieving higher model accuracy within a fixed round budget.
Specifically, our study addresses two key questions: (i) which layer generates the most informative latent representations for exchange and (ii) how can we effectively leverage remote latent information without incurring prohibitive memory / communication overhead?
Furthermore, we provide a theoretical analysis of the impact of activation sharing on the convergence properties of FedAvg for non-convex and smooth objective functions.
To bridge the gap between theory and practice, we also discuss how to exploit activation sharing with minimal additional communication cost by reusing remote activations multiple times.

Our extensive empirical study substantiates the practicality of activation sharing, demonstrating its superiority not only in convergence speed but also in system efficiency.
On representative benchmarks, including CIFAR-10/100~\cite{krizhevsky2009learning}, FEMNIST~\cite{caldas2018leaf}, and SVHN~\cite{netzer2011reading}, under non-IID settings, ASFL consistently achieves higher accuracy than state-of-the-art FL methods, including FedProx~\cite{li2020federated}, SCAFFOLD~\cite{karimireddy2020scaffold}, FedPVR~\cite{li2023effectiveness}, FedProto~\cite{tan2022fedproto}, FedImpro~\cite{tang2024fedimpro}, and SplitFed~\cite{thapa2022splitfed}.
These results indicate that ASFL effectively mitigates client drift by allowing each client to learn from the latent representations of other clients, thereby achieving higher accuracy within a fixed round budget.
In addition, our wall-clock time analysis shows that ASFL reduces the per-round training time by up to $33\%$ compared to the other FL methods under the same number of local steps per round.
Finally, we empirically show that, under model inversion attacks, ASFL exhibits privacy leakage comparable to that of existing representation-sharing methods such as SplitFed.   
\section {Background} \label{sec:2}

\textbf{Federated optimization} --
We consider federated minimization problems defined as follows:
\begin{align}
    \underset{\mathbf{x} \in \mathbb{R}^d}{\min}\left[F(\mathbf{x}) := \frac{1}{m} \sum_{i=1}^{m} F_i(\mathbf{x}) \right] \label{eq:objective},
\end{align}
where $m$ is the number of local models and $F_i(\mathbf{x}) = \mathop{\mathbb{E}}_{\xi_i \sim D_i}[ F_i(\mathbf{x}, \xi_i)]$
is the local objective function associated with local data distribution $D_i$.
FedAvg \cite{mcmahan2017communication} is a pioneering distributed learning method designed for federated optimization, utilizing local SGD with periodic model aggregations.
It has been widely adopted to train global models without explicitly sharing decentralized data.
However, it suffers from slow convergence when data heterogeneity is high.

\textbf{Client drift} -- 
Client drift is a well-known phenomenon in FL. 
In federated optimization, each client performs local updates using only its own data, which are often drawn from heterogeneous (non-IID) distributions. 
As a result, local updates become biased with respect to the global objective, leading to a mismatch between the expected local gradient and the true global gradient:
\begin{align}
    \mathbb{E}[g_t^i] = \nabla F_i(\mathbf{x}_t^i) \neq \nabla F(\mathbf{x}_t^i),
\end{align}
where $g_t^i$, $\nabla F_i(\mathbf{x}_t^i)$, and $\nabla F(\mathbf{x}_t^i)$ denote the stochastic gradient, the local full-batch gradient, and the global full-batch gradient, respectively.
This deviation becomes more pronounced as the degree of data heterogeneity increases or the number of local update steps grows, ultimately slowing down the convergence of the global model.
To mitigate this issue, a variety of approaches have been proposed in the literature such as FedProx, SCAFFOLD, and SILVER~\cite{oko2024silver}.
\section {Related work} \label{sec:3}

\textbf{Variance reduction methods} --
Recent federated optimization methods mitigate client drift using variance reduction techniques, including FedProx, SCAFFOLD, FedPVR, FedVARP, and SILVER. 
Although effective, these methods typically introduce additional memory and communication overhead to maintain auxiliary states such as control or tracking variables, which becomes increasingly burdensome for large-scale models. 
FedProto~\cite{tan2022fedproto} shares class-wise prototypes computed by averaging embeddings within each class and regularizes local training toward the aggregated global prototypes to reduce cross-client discrepancies in class representations.
MOON~\cite{li2021model} reduces client drift using a model-contrastive objective, at the cost of additional memory and computation overhead for maintaining auxiliary models and computing contrastive losses. 
In this work, we compare ASFL with these representative variance-reduced FL methods in terms of both model accuracy under a fixed round budget and communication overhead.

\textbf{Split learning methods} --
Split Federated Learning (SFL) integrates Split Learning (SL)~\cite{vepakomma2018split} with Federated Learning (FL) to alleviate the computational burden on the client side.
The most prominent difference between SFL and conventional FL is that SFL transmits intermediate activations from a specific `cut layer' to a central server.
The server then executes the forward and backward passes for the remaining layers.
This split training mechanism not only reduces client-side computational overhead but also mitigates client drift, as the server-side sub-model is updated using aggregated knowledge from all participating clients.
While SplitFed aggregates latent representations on the server and trains a subset of network layers there, our proposed method requires neither data centralization across clients nor server-side computation.
We will also compare our proposed activation sharing method with SplitFed in terms of model accuracy and system efficiency.

\section {Method} \label{sec:4}

\begin{figure}[t]
\begin{center}
\includegraphics[width=\columnwidth]{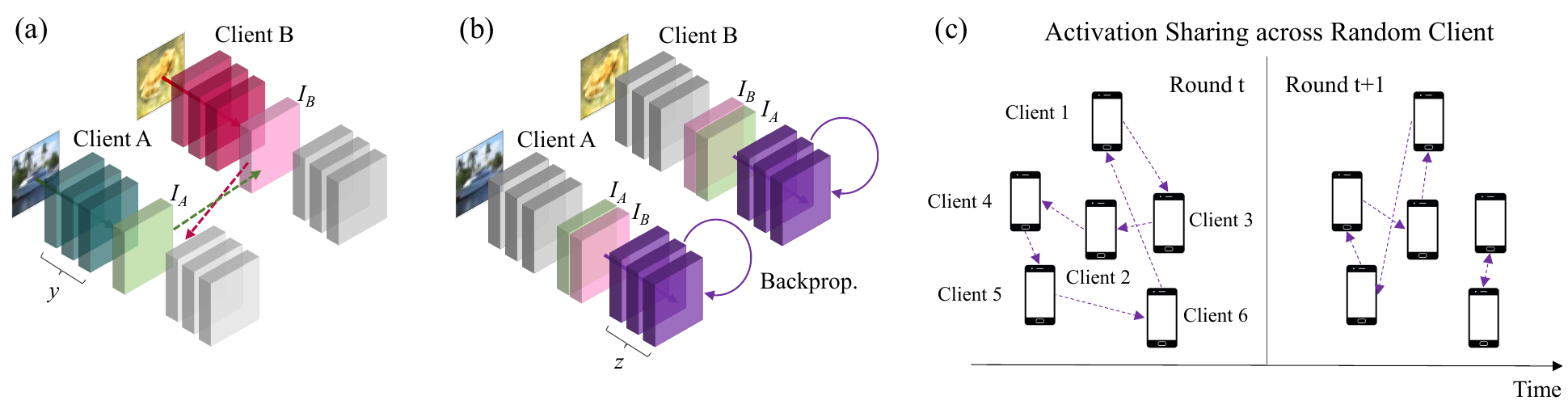}
\end{center}
\vspace{-0.2cm}
\caption{
    A schematic illustration of latent information sharing. Figure (a) shows the \textit{Latent Information Sharing} step. The activation outputs from $\mathbf{y}$, $I_A$ and $I_B$, are exchanged between the two clients. Figure (b) shows the \textit{Tail Round} in which the two clients run backpropagation only within $\mathbf{z}$ using both $I_A$ and $I_B$. Figure (c) shows the activation sharing across random clients.
}
\label{fig:schematic}
\vspace{-0.4cm}
\end{figure}

\subsection {Latent information sharing}
In this section, we describe how latent information is shared across clients via two steps, \textit{Activation Sharing} and \textit{Tail Round}.
Figure~\ref{fig:schematic} provides a schematic illustration of this process.

\textbf{Activation sharing} -- We define \textit{Latent Information} as the activation output of a specific hidden layer in a neural network.
This information inherently captures local data patterns.
Therefore, we expect that sharing it across clients will alleviate data heterogeneity.
First, a given neural network $x$ is partitioned into two consecutive groups of layers, denoted as $\mathbf{x} := \{\mathbf{y}, \mathbf{z}\}$, where $\mathbf{y}$ consists of the input-side layers and $\mathbf{z}$ consists of the remaining output-side layers.

Then, we define the latent information $I$ as $\tau' / 2$ activation output batches from $y$:
\begin{align}
    I_i := \left\{ \mathbf{y}(\mathcal{B}_{i,t}) \right\}_{t=1}^{\tau' / 2},
\end{align}
where $\tau'$ is the number of local steps with remote activations, $\{\mathcal{B}_{i,t}\}_{t=1}^{\tau' / 2}$ denotes $\tau' / 2$ mini-batches randomly sampled from the local dataset of client $i$.
Once $I_i$ is ready, randomly paired clients exchange their $I_i$ and corresponding labels through inter-process communications.
We call this communication \textit{Activation Sharing}.

\noindent \textbf{Tail round} -- 
Once client $i$ receives $I_j$ ($i \neq j$), it performs local training using the combined set $I_i \cup I_j$, updating only the subnetwork $\mathbf{z}$. 
Specifically, client $i$ trains $\mathbf{z}$ for $\tau'$ iterations using mini-batches randomly sampled from $I_i \cup I_j$. 
Since $I$ consists of the activation outputs from the last layer of $\mathbf{y}$, both $I_i$ and $I_j$ can be directly fed into $\mathbf{z}$. 
Backpropagation is restricted to the layers within $\mathbf{z}$. 
We refer to these steps as a \textit{tail round}.

Algorithm~\ref{alg:LIS} presents the pseudocode for the \textit{Latent Information Sharing} method. 
Each client $i$ first computes $I_i$ using its local data and shares it with a randomly selected client $j$ (Line~3, \textit{activation sharing}). 
Subsequently, each client $i$ trains $\mathbf{z}$ for $\tau'$ iterations using activations sampled from $I_i \cup I_j$ (Lines~5--6, \textit{tail round}).

Algorithm~\ref{alg:FedAvg} outlines FedAvg integrated with the proposed \textit{Latent Information Sharing} method. 
At the beginning of each communication round, the model is first updated via Algorithm~\ref{alg:LIS} (Line~4). 
Consequently, each round consists of $\tau'$ tail update steps followed by $\tau$ full update steps.

\begin{algorithm}[t]
\caption{Activation Sharing}
\label{alg:LIS}
\begin{algorithmic}[1]
\small
\REQUIRE $\tau'$: the total number of local update steps within a tail round, $b$: the number of activation batches to be exchanged.
\STATE {Partition $x$ into $y$ and $z$.}
\STATE {Each client $i \in [m]$ calculates $I_i = {\mathbf{y}(\mathcal{B}_{i,t})}_{t=1}^{\tau' - b}$.}
\STATE {Each client $i$ sends $b$ mini-batch drawn from $I_i$ to a random client $j \neq i$.}
\STATE {Each client $i$ receives $b$ remote samples, $I_k$ from a random client $k \neq i$.}
\FOR{$t \to \tau'$}
    \STATE {$\mathcal{\hat{B}} \leftarrow \textrm{Random}(I_i \cup I_k)$.}
    \STATE {$\mathbf{z}_{i,t+1} = \mathbf{z}_{i,t} - \eta \nabla f(\mathbf{z}, \mathcal{\hat{B}})$}
\ENDFOR
\STATE {Return $\mathbf{x} = \{\mathbf{y}, \mathbf{z}_{i, \tau'}\}$.}
\end{algorithmic}
\end{algorithm}

\begin{algorithm}[t]
\caption{FedAvg with Activation Sharing}
\label{alg:FedAvg}
\begin{algorithmic}[1]
\small
\REQUIRE $\tau$: the total number of local updates within a communication round, $\tau'$: the total number of local updates within a tail round, $R$: the total number of communication rounds, $\mathcal{N}$: the total list of clients, $a$: the number of active clients per round.
\FOR{$r \to R$}
    \STATE {$\mathcal{A} \leftarrow \textrm{Random}_a(\mathcal{N})$}
    \STATE {Server sends global model $\mathbf{x}$ to clients in $\mathcal{A}$.}
    \STATE {$\mathbf{x}_{i,0} \leftarrow \textrm{Latent Information Sharing}(\mathbf{x})$ (Algorithm~\ref{alg:LIS})}
    \FOR{$k \to \tau$}
        \STATE {Each client $i \in \mathcal{A}$ gets a mini-batch $\mathcal{B}_{i,k}$ from its local dataset.}
        \STATE {$\mathbf{x}_{i,k+1} = \mathbf{x}_{i,k} - \eta \nabla f(\mathbf{x}_{i,k}, \mathcal{B}_{i,k})$}
    \ENDFOR
    \STATE {Server aggregates the local models $\mathbf{x}_{i,\tau}, i \in \mathcal{A}$ to $\mathbf{x}$.}
\ENDFOR
\STATE {Return $\mathbf{x}$}
\end{algorithmic}
\end{algorithm}

\subsection{Controlling the degree of information sharing}
\textit{Activation Sharing} is governed by two critical factors that determine its impact on local training: 
(i) which layer's latent information is shared, and 
(ii) the volume of activation outputs being shared.

\subsubsection{Selecting the sharing layer via cost-variance trade-off}\label{sec:tradeoff}
The tail round exposes the sub-model $\mathbf{z}$ to latent information from other clients. 
These remote activations can be utilized to update $\mathbf{z}$, thereby reducing the variance of stochastic gradients. 
Sharing activations at earlier layers allows a larger portion of the model to benefit from this variance reduction.
However, this incurs higher computational overhead, as backpropagation must traverse more layers during the tail rounds. 
To address this trade-off, we propose a practical layer selection scheme that balances variance reduction and system efficiency.

\textbf{Step 1. System efficiency curve} -- 
Let $k$ denote the partition index that splits the model into two sub-models: $\mathbf{y}$, comprising layers $1$ to $k-1$, and $\mathbf{z}$, comprising layers $k$ to $L$ (where $L$ is the total number of layers). 
The computational cost of the backward pass through layer $i$ is defined as:
\begin{align}
    \mathcal{C}_{\text{comp}}(i) &:= N_{\text{out}}^i \cdot N_{\text{conn}}^i, \nonumber \\
    \hat{\mathcal{C}}_{\text{comp}}(k) &= \sum_{i=k}^{L} \mathcal{C}_{\text{comp}}(i), \nonumber
\end{align}
where $N_{\text{out}}^i$ is the number of output activations at layer $i$, and $N_{\text{conn}}^i$ is the number of input connections per output unit. 
Since backpropagation proceeds from layer $L$ down to $k$, the total cost is accumulated over the layers in $\mathbf{z}$.

The additional communication cost of activation sharing is quantified as:
\begin{align}
    \mathcal{C}_{\text{comm}}(k) := N_{\text{out}}^{k-1},
\end{align}
where $N_{\text{out}}^{k-1}$ denotes the number of activations at the boundary layer $k-1$. 
To integrate these measures, we normalize them separately and compute their product to represent the total system overhead for a given $k$:
\begin{align}
    \mathcal{C}(k) = \tilde{\mathcal{C}}_{\text{comp}}(k) \cdot \tilde{\mathcal{C}}_{\text{comm}}(k), \label{eq:ck}
\end{align}
where $\tilde{\mathcal{C}}_{\text{comp}}(k)$ and $\tilde{\mathcal{C}}_{\text{comm}}(k)$ are the normalized computation and communication costs, respectively. 
Notably, $\mathcal{C}(k)$ monotonically decreases as $k$ increases from $1$ to $L$.

\textbf{Step 2. Variance reduction curve} -- 
We quantify the potential benefit of variance reduction by varying $k$ as follows:
\begin{align}
\mathcal{V}(k) := \sum_{i=1}^{k-1} \mathrm{Var}(g^i), \label{eq:vk}
\end{align}
where $g^i$ denotes the mini-batch gradient of the loss function with respect to the parameters of layer $i$. 
In practice, we approximate the true gradient variance using the empirical variance of gradients within each mini-batch.

\textbf{Step 3. Finding the intersection} -- 
Finally, we identify the optimal layer $k$ where the normalized $\mathcal{C}(k)$ and $\mathcal{V}(k)$ curves intersect. 
This ``sweet spot" represents a configuration that maximizes the number of layers benefiting from variance reduction while maintaining overall system efficiency. 
Figure~\ref{fig:cv} illustrates these $\mathcal{C}(\cdot)$ and $\mathcal{V}(\cdot)$ curves across four popular machine learning benchmarks.

\begin{figure}[t]
\begin{center}
\includegraphics[width=1\columnwidth]{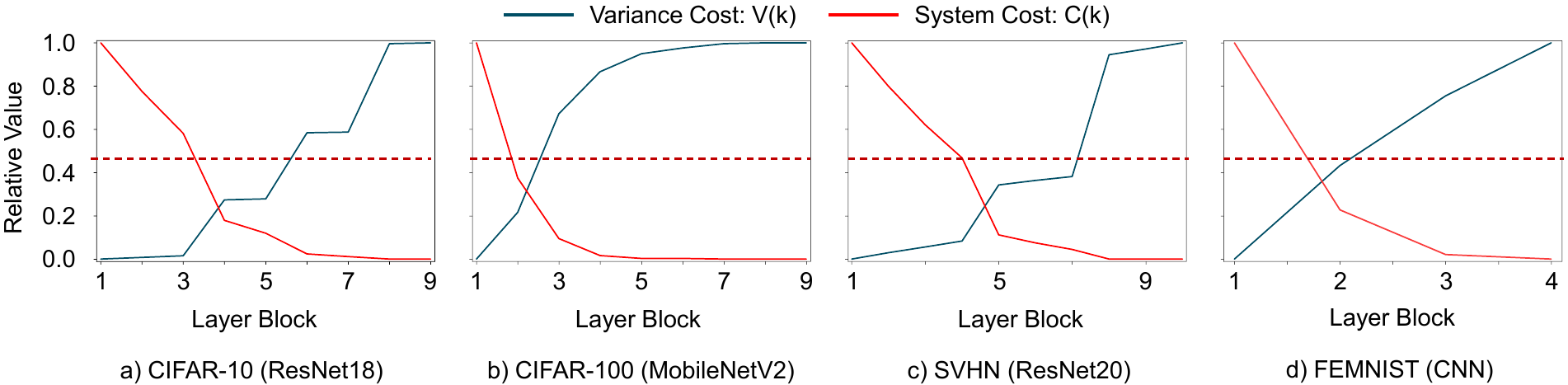}
\vspace{-0.5cm}
\end{center}
\caption{
    Layer selection results. The green curve represents the cumulative variance of the input-side sub-model y, computed using Eq.~\ref{eq:vk}, while the red curve represents the system efficiency computed using Eq.~\ref{eq:ck}. The red horizontal dashed line indicates the midpoint of the efficiency range. Across various benchmarks, the intersection of the two curves indicates a balanced trade-off between system efficiency and variance reduction.
}
\label{fig:cv}
\vspace{-0.2cm}
\end{figure}

\subsubsection{Communication-efficient activation sharing} \label{sec:comm-efficient-sharing}
Depending on the input dimensionality and layer connectivity, communicating intermediate activations can incur substantial communication overhead.
However, in our study, we find that only a small amount of remote activations significantly improves the FL performance.
We let each client send and receive only a single mini-batch in our empirical study, and it achieves superior performance compared to the other FL methods (See Table~\ref{tab:comparison}).

Here, we describe the detailed activation sharing mechanism.
In Algorithm~\ref{alg:LIS}, every client $i \in \mathcal{A}$ prepare $\tau' - b$ local activation batches, $I_i$, and then sends $b$ batches drawn from $I_i$ to a random client $j$, where $\mathcal{A}$ is a set of active clients randomly selected from the total $m$ clients.
Then it receives the same amount of remote activations $I_k$ from client $k$ randomly sampled from $\mathcal{A}$.
After finishing the communications, the tail round is performed using mini-batches randomly extracted from the activation pool $\mathcal{P} = I_i \cup I_k$.

As $b$ increases, the communication cost for sharing activations increases while each client will enjoy the potentially reduced variance of updates.
Since this introduces a trade-off, the value of $b$ should be carefully adjusted to minimize the additional communication cost while maximizing the reduction in variance.
We conducted an ablation study to evaluate the impact of the number of remote batches $b$ and present the results in the Appendix (Table~\ref{tab:batches}).

\subsection {Convergence analysis} \label{sec:theory}
We show that ASFL guarantees convergence for non-convex and smooth optimization problems:
\begin{equation}
    \min_{\mathbf{x} \in \mathbb{R}^d} F(\mathbf{x}) := \frac{1}{m} \sum_{i=1}^m F_i(\mathbf{x}), \nonumber
\end{equation}
where $m$ denotes the total number of clients, and $F_i(\mathbf{x}) := \mathbb{E}_{\xi_i \sim \mathcal{D}_i}[f(\mathbf{x}; \xi_i)]$ represents the local objective function associated with the data distribution $\mathcal{D}_i$ of client $i$. Here, $f(\mathbf{x}; \xi_i)$ is the loss function evaluated on a sample $\xi_i$ drawn from $\mathcal{D}_i$.
Each communication round consists of $H$ local steps, and $H = \tau + \tau'$, where $\tau$ and $\tau'$ are the number of full-model steps and tail steps, respectively.
Consequently, the ratio of full-model steps within each round is defined as $\rho = \frac{\tau}{H}$.
In each communication round, the fraction of remote data in the total data consumed is denoted by $\alpha$.
Our analysis relies on the following assumptions, which hold for all $i \in [m]$.

\begin{assumption}[Lipschitz continuous gradients]\label{asm:smooth}
$\|\nabla F_i(\mathbf{a}) - \nabla F_i(\mathbf{b})\| \le \mathcal{L}\|\mathbf{a} - \mathbf{b}\|$ for some $\mathcal{L} > 0$ and all $\mathbf{a}, \mathbf{b} \in \mathbb{R}^d$.
\end{assumption}
\begin{assumption}[Unbiased local gradients]\label{asm:unbiased}
$\mathbb{E}_{\xi_i \sim D_i}[\nabla f(\mathbf{x}, \xi_i)] = \nabla F_i(\mathbf{x})$.
\end{assumption}
\begin{assumption}[Bounded local and global variances]\label{asm:var}
There exist two constants $\sigma_x > 0$ and $\sigma_G > 0$, such that the local gradient variance is bounded by $\mathbb{E}[\| \nabla f(\mathbf{x}, \xi_i) - \nabla F_i(\mathbf{x}) \|^2] \leq \sigma_x^2, \forall i \in [m]$, and the global variability is bounded by $\| \nabla F_i(\mathbf{x}) - \nabla F(\mathbf{x}) \|^2 \leq \sigma_G^2, \forall i \in [m]$.
\end{assumption}
\begin{assumption}[Bounded local variance w.r.t.\ a partial model]\label{asm:var_z}
$\mathbb{E}_{\xi_i \sim D_i}\|\nabla f(\mathbf{z}, \xi_i) - \nabla F_i(\mathbf{z})\|^2 \le \sigma_z^2$ for some $\sigma_z > 0$.
\end{assumption}

Then, we present our convergence analysis results.
See the Appendix for proofs.
\begin{lemma}[framework] \label{lemma:framework}
Under assumption 1 $\sim$ 4, if the learning rate $\eta \leq \min \{\frac{1}{4\mathcal{L}\tau}, \frac{1}{24\mathcal{L}\tau'}\}$, the loss reduction of a single communication round is bounded as follows.
\begin{align}
    \mathbb{E}[F(\mathbf{x}_{t+1})] &\leq F(\mathbf{x}_t) - \frac{\eta\tau}{2} \left\| \nabla F(\mathbf{x}_t) \right\|^2 - \frac{\eta\tau'}{4} \left\| \nabla_z F(\mathbf{x}_t) \right\|^2 \nonumber \\
    &\qquad + \frac{\eta\mathcal{L}^2}{a} \left( \sum_{i\in\mathcal{A}}\sum_{j=\tau'}^{H-1} \mathbb{E} \left[ \left\| \mathbf{x}_{t,j}^i - \mathbf{x}_t \right\|^2 \right] + \sum_{i\in\mathcal{A}}\sum_{j=0}^{\tau'-1} \mathbb{E} \left[ \left\| \mathbf{z}_{t,j}^i - \mathbf{z}_t \right\|^2 \right]\right) \nonumber \\
    &\qquad + \frac{2\eta^2\mathcal{L}}{a} \left( \tau \sigma_x^2 + \tau' \sigma_z^2 + \tau' \left( 4\alpha(1-\alpha) + 3\tau' \left( \frac{m-a}{m-1} + (1-\alpha)^2 \right) \right) \sigma_G^2 \right). \nonumber
\end{align}
\end{lemma}

\begin{lemma}[Partial model discrepancy] \label{lemma:partial}
Suppose that $\mathcal{A}$ is sampled uniformly without replacement and each $i \in \mathcal{A}$ selects $k$ uniformly from $\mathcal{A} \setminus \{i\}$. For any step size satisfying $\eta \leq \frac{1}{3\sqrt{2}\mathcal{L}\tau'}$, we have, for any $j = 0, \dots, \tau'$,
\begin{align}
    \frac{1}{a} \sum_{i \in \mathcal{A}} \mathbb{E}[\| \mathbf{z}_{t,j}^{i} - \mathbf{z}_t \|^2] \leq 2\eta^2 \tau' \sigma_z^2 + 2\eta^2 \tau' \left( 4\alpha(1-\alpha) + 9\tau' \gamma \right) \sigma_G^2 + 18\eta^2 (\tau')^2 \| \nabla_z F(\mathbf{x}_t) \|^2, \nonumber
\end{align}
where $\gamma = \alpha^2 + (1-\alpha)^2$.
\end{lemma}

\begin{theorem}\label{thm:main}
Under Assumptions 1--4, suppose that $\mathcal{A}$ is a set of $a$ clients sampled uniformly at random without replacement from $[m]$, and each client $i \in \mathcal{A}$ selects a client $k$ uniformly at random from $\mathcal{A} \setminus \{i\}$. Let $H = \tau + \tau'$ and $\rho = \tau / H \geq \frac{1}{2}$. If $\eta \leq \frac{1}{24\mathcal{L}\tau}$, then ASFL satisfies:
\begin{align}
    \frac{1}{T}\sum_{t=0}^{T-1} \mathbb{E}\left[ \left\| \nabla F(\mathbf{x}_t) \right\|^2 \right] \leq \frac{4\left( F(\mathbf{x}_0) - F^* \right)}{\eta\tau T} + \frac{8\eta\mathcal{L}}{a\tau}A + \frac{8\eta^2\mathcal{L}^2}{\tau}B, \nonumber
\end{align}
where $\gamma = \alpha^2 + (1-\alpha)^2$, and
\begin{align}
    A &:= \tau\sigma_x^2 + \tau'\sigma_z^2 + \tau' \left( 4\alpha(1-\alpha) + 3\tau' \left( \frac{m-a}{m-1} + (1-\alpha)^2 \right) \right) \sigma_G^2, \nonumber \\
    B &:= \tau^2\sigma_x^2 + \tau'(3\tau+\tau')\sigma_z^2 + \left( 9\tau^3 + \tau'(3\tau+\tau') \left( 4\alpha(1-\alpha) + 9\tau'\gamma \right) \right) \sigma_G^2. \nonumber
\end{align}
\end{theorem}

\begin{remark}[Convergence rate]
Every term on the right-hand side of Theorem~\ref{thm:main}, except for the first term, is proportional to $\eta$ or $\eta^2$, and the first term decays as $\mathcal{O}(1/(\eta\tau T))$.
Thus, the bound contains no residual term that is independent of $\eta$, and ASFL converges to a stationary point.
In particular, with $\eta = \sqrt{\frac{a}{\tau T}}$, which satisfies $\eta \leq \frac{1}{24\mathcal{L}\tau}$ for $T \geq 576\mathcal{L}^2\tau a$, we have
\begin{align}
    \frac{1}{T}\sum_{t=0}^{T-1} \mathbb{E}\left[ \left\| \nabla F(\mathbf{x}_t) \right\|^2 \right] \leq \frac{4\left( F(\mathbf{x}_0) - F^* \right) + 8\mathcal{L}A/\tau}{\sqrt{a\tau T}} + \frac{8a\mathcal{L}^2 B}{\tau^2 T} = \mathcal{O}\left( \frac{1}{\sqrt{aT}} + \frac{a}{T} \right), \nonumber
\end{align}
where $\tau$ and $\tau'$ are treated as constants.
That is, when $T$ is sufficiently large, the first term dominates, and ASFL achieves a linear speedup with respect to the number of active clients $a$.
When $\rho = 1$, i.e., no tail steps are performed, all terms involving $\tau'$ vanish, and the bound reduces to the form of local SGD with $\tau$ local steps (FedAvg).
\end{remark}

\begin{remark}[Effect of activation sharing]
Activation sharing introduces no bias because each client selects its peer uniformly at random, so the deviation $\nabla_z F_k(\cdot) - \nabla_z F_i(\cdot)$ becomes zero in expectation (see (\ref{eq:cancel})).
Our result instead shows that the activation sharing reduces the client drift.
Since $\gamma = 1 - 2\alpha(1-\alpha)$, the heterogeneity coefficient of the tail steps in $B$ becomes $9\tau' - (18\tau'-4)\,\alpha(1-\alpha)$, which decreases from $9\tau'$ to $\frac{9\tau'}{2} + 1$ as $\alpha$ decreases from $\alpha = 1$ to $\alpha = \frac{1}{2}$.
The cost is an additional variance from the random peer selection, $4\alpha(1-\alpha) + 3\tau'(1-\alpha)^2$ in $A$.
However, this term is scaled by $\frac{1}{a}$, whereas the drift term is not.
Thus, a smaller $\alpha \in [\frac{1}{2}, 1]$ is beneficial when the client drift dominates (e.g., highly non-IID environments).
Since any $\alpha < 1$ reduces the drift, a larger $\alpha$ retains a part of the benefit at a lower communication cost, which is well aligned with our empirical results in Table~\ref{tab:comparison}.
\end{remark}

\subsection{Differential privacy for activation sharing}

Although ASFL does not share raw data, exchanging latent representations  $I_i = \left\{ \mathbf{y}(\mathcal{B}_{i,t}) \right\}_{t=1}^{\tau' / 2}$ between paired clients may still leak sensitive information.
To mitigate this risk, one can consider incorporating differential privacy (DP) to ASFL, following prior work~\cite{thapa2022splitfed}.
Specifically, we adopt PixelDP~\cite{lecuyer2018certified} by inserting a noise layer into $\mathbf{y}$, so that the receiving module $\mathbf{z}$ operates on privatized activations rather than the raw outputs of $\mathbf{y}$. By the \emph{post-processing immunity} property of DP~\cite{dwork2014algorithmic}, any downstream computation on the perturbed activations preserves the same privacy guarantee.

The sensitivity of the subnetwork $\mathbf{y}$, which governs the required noise magnitude, is defined as:
\begin{equation}
\triangle^{\mathbf{y}}_{p,q} \;=\; \max_{\xi \neq \xi'} \frac{\bigl\|\mathbf{y}(\xi) - \mathbf{y}(\xi')\bigr\|_q}{\|\xi - \xi'\|_p},\nonumber
\end{equation}
where $\xi$ and $\xi'$ are two input samples of the corresponding client, and $p$ and $q$ are the degrees of input and output norms.
In addition, $R$ defines the radius of the threat model such that any adversarial perturbation confined within $B_p(R):= \{\delta' \in \mathbb{R}^{n} : \|\delta'\|_p \leq R\}$ is 
guaranteed to be masked by the injected noise and provides $(\epsilon, \delta)$-DP guarantees. 

We consider two standard mechanisms to privatize the shared activations.
The Laplace mechanism achieves $(\epsilon,0)$-DP by adding noise scaled to $\triangle^{\mathbf{y}}{p,q}/\epsilon$, while the Gaussian mechanism ensures $(\epsilon,\delta)$-DP with variance $\sigma^2 \propto (\triangle^{\mathbf{y}}{p,q}/\epsilon)^2$.
In both cases, the transmitted privatized activations satisfy differential privacy w.r.t. the sender’s local data.
\section {Experiments} \label{sec:5}

\noindent \textbf{Experimental settings} -- 
We evaluate ASFL on four benchmarks: CIFAR-10 (ResNet18~\cite{he2016deep}), CIFAR-100 (MobileNetV2~\cite{sandler2018mobilenetv2}), SVHN (ResNet20), and FEMNIST (CNN). 
Experiments involve 100 total clients, with 10 randomly activated per round. 
For non-IID settings, we use a Dirichlet distribution with $\alpha = 0.1$. 
Our framework is based on PyTorch 2.7.0 and runs on a server with two NVIDIA RTX 4070 GPUs. 
See Appendix~\ref{appendix:settings} for detailed settings.

\begin{table}[t!]
\vspace{-0.2cm}
\caption{Accuracy and communication cost comparison across datasets.}
\label{tab:comparison}
\begin{center}
\setlength{\tabcolsep}{2pt}
\scriptsize
\begin{tabular}{lcccccccc}
\toprule
\multirow{2}{*}{\textbf{Method}} &
\multicolumn{2}{c}{\textbf{CIFAR-10}} &
\multicolumn{2}{c}{\textbf{CIFAR-100}} &
\multicolumn{2}{c}{\textbf{SVHN}} &
\multicolumn{2}{c}{\textbf{FEMNIST}} \\
\cmidrule(lr){2-3} \cmidrule(lr){4-5} \cmidrule(lr){6-7} \cmidrule(lr){8-9}
& \textbf{Acc. (\%)} \accup & \textbf{Comm (MB)} \commdown
& \textbf{Acc.} \accup & \textbf{Comm} \commdown
& \textbf{Acc.} \accup & \textbf{Comm} \commdown
& \textbf{Acc.} \accup & \textbf{Comm} \commdown \\
\midrule
FedAvg & $67.70 \pm 0.33$ & 42.63 & $48.31 \pm 0.68$ & 9.20 & $87.13 \pm 0.16$ & 1.04 & $83.54 \pm 0.19$ & 1.54 \\
FedProx & $67.91 \pm 0.53$ & 42.63 & $47.61 \pm 0.27$ & 9.20 & $88.27 \pm 0.90$ & 1.04 & $83.23 \pm 0.29$ & 1.54 \\
FedProto & $69.57 \pm 1.58$ & 42.65 & $48.08 \pm 0.87$ & 9.69 & $87.99 \pm 0.15$ & 1.04 & $84.11 \pm 0.23$ & 1.60 \\
MOON & $69.35 \pm 0.83$ & 42.63 & $46.49 \pm 0.21$ & 9.20 & $86.10 \pm 0.83$ & 1.04 & $84.03 \pm 0.12$ & 1.54 \\
SCAFFOLD & $71.93 \pm 1.14$ & 85.26 & $48.53 \pm 0.34$ & 18.40 & $87.36 \pm 0.37$ & 2.08 & $84.13 \pm 0.20$ & 3.08 \\
FedPVR & $69.40 \pm 0.97$ & 42.65 & $48.48 \pm 0.70$ & 9.69 & $87.92 \pm 0.47$ & 1.04 & $83.57 \pm 0.26$ & 1.60 \\
FedVARP & $69.24 \pm 1.43$ & 42.63 & $48.71 \pm 0.55$ & 9.20 & $87.92 \pm 0.40$ & 1.04 & $82.88 \pm 0.25$ & 1.54 \\
FedImpro & $62.80 \pm 0.41$ & 45.13 & $47.65 \pm 0.66$ & 21.70 & $85.83 \pm 1.64$ & 1.67 & $83.74 \pm 0.24$ & 3.02 \\
FedMDVR & $71.74 \pm 2.04$ & 85.26 & $46.83 \pm 0.10$ & 18.40 & $86.60 \pm 1.49$ & 2.08 & $84.37 \pm 0.44$ & 3.08 \\
SplitFed & $68.16 \pm 0.85$ & 257.45 & $48.42 \pm 0.40$ & 320.01 & $86.92 \pm 0.05$ & 40.11 & $83.67 \pm 0.14$ & 2.93 \\
\rowcolor{gray!30} \textbf{ASFL (Ours)} & $\mathbf{75.27 \pm 0.88}$ & 50.63 & $\mathbf{50.56 \pm 0.32}$ & 13.20 & $\mathbf{89.10 \pm 0.31}$ & 2.04 & $\mathbf{84.85 \pm 0.33}$ & 1.78 \\
\bottomrule
\end{tabular}
\end{center}
\vspace{-0.3cm}
\end{table}

\begin{figure}[t]
\begin{center}
\includegraphics[width=0.9\columnwidth]{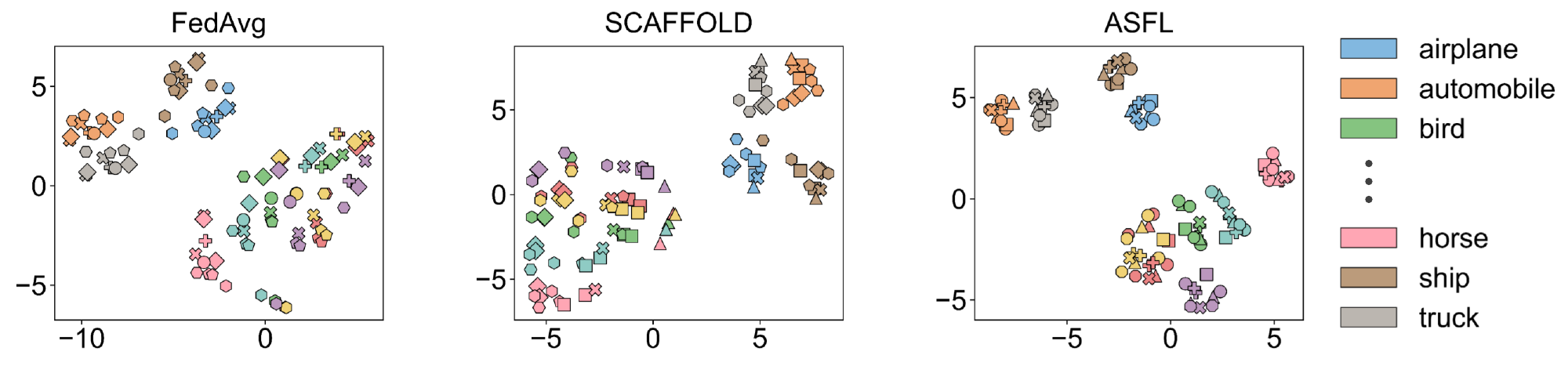}
\vspace{-0.2cm}
\end{center}
\caption{
t-SNE visualization of per-client class centroids based on penultimate features from 10 clients on CIFAR-10. Colors denote classes and marker shapes indicate clients.
}
\label{fig:tsne}
\vspace{-0.4cm}
\end{figure}

\subsection{Comparative study}
\textbf{Federated Learning benchmarks} --
Table~\ref{tab:comparison} compares ASFL with SOTA FL methods, where PixelDP is not applied to the transmitted activations to isolate the effect of activation sharing on optimization efficiency.
For FedMDVR, whose uplink and downlink communication costs differ, Table~\ref{tab:comparison} reports the larger of the two as a worst-case estimate.
Drift-mitigation methods, including SCAFFOLD, FedProto, MOON, FedPVR, and FedVARP, improve accuracy over FedAvg on most benchmarks, but SCAFFOLD and FedMDVR nearly double communication cost through control variate aggregation, and the other four offer relatively limited gains.
SplitFed achieves higher accuracy by exchanging activations at every update step, but at the cost of substantially increased communication cost.
In contrast, ASFL attains higher accuracy across all benchmarks while introducing only marginal additional communication cost.
E.g., on CIFAR-10, ASFL improves accuracy by approximately $7.5\%$ over FedAvg while exchanging 8 MB more activations per round.
Extra experimental analyses supporting these findings are provided in Appendix~\ref{appendix:supplementary}.

\textbf{t-SNE analysis} --
Figure~\ref{fig:tsne} provides a t-SNE visualization of per-client class centroids in the latent space. 
Under FedAvg, representations of the same class are widely scattered, whereas SCAFFOLD shows only limited concentration. 
In contrast, ASFL forms noticeably tighter and more coherent clusters across different clients, suggesting superior cross-client representation alignment. 
Overall, these results demonstrate that activation sharing effectively alleviates data heterogeneity by promoting consistent representations, directly leading to the higher validation accuracy shown in Table~\ref{tab:comparison}.

\begin{table}[t]
\vspace{-0.2cm}
\caption{Per-client round wall-clock time for local computation, normalized to FedAvg on each dataset. Bars share a $0$--$4\times$ scale, and vertical ticks mark the FedAvg baseline ($1\times$).}
\label{tab:local_training_time}
\begin{center}
\setlength{\tabcolsep}{3pt}
\scriptsize
\definecolor{timebaraccent}{HTML}{168A87}
\newcommand{\timebar}[2][gray!55]{%
  \begingroup
  \setlength{\unitlength}{3mm}%
  \makebox[10mm][r]{$#2\times$}\hspace{1mm}%
  \begin{picture}(4,0.7)(0,-0.15)
    \put(0,0){\color{gray!12}\rule{4\unitlength}{1.1mm}}
    \put(0,0){\color{#1}\rule{#2\unitlength}{1.1mm}}
    \put(1,-0.1){\color{gray!75}\line(0,1){0.55}}
  \end{picture}%
  \endgroup
}
\renewcommand{\arraystretch}{1.2}
\begin{tabular}{lcccc}
\toprule
\textbf{Method} & \textbf{CIFAR-10} & \textbf{CIFAR-100} & \textbf{SVHN} & \textbf{FEMNIST} \\
\midrule
FedProto & \timebar{1.02} & \timebar{1.04} & \timebar{1.04} & \timebar{1.26} \\
MOON & \timebar{1.38} & \timebar{1.34} & \timebar{1.32} & \timebar{1.84} \\
SCAFFOLD & \timebar{2.06} & \timebar{1.77} & \timebar{1.43} & \timebar{2.27} \\
FedPVR & \timebar{1.00} & \timebar{1.04} & \timebar{1.01} & \timebar{1.10} \\
FedVARP & \timebar{1.07} & \timebar{1.02} & \timebar{1.03} & \timebar{1.18} \\
FedImpro & \timebar{1.38} & \timebar{1.92} & \timebar{1.17} & \timebar{3.67} \\
SplitFed & \timebar{1.00} & \timebar{1.00} & \timebar{1.02} & \timebar{1.19} \\
\textbf{ASFL (Ours)} & {\boldmath\timebar[timebaraccent]{0.69}} & {\boldmath\timebar[timebaraccent]{0.77}} & {\boldmath\timebar[timebaraccent]{0.67}} & {\boldmath\timebar[timebaraccent]{0.94}} \\
\bottomrule
\end{tabular}
\end{center}
\vspace{-0.2cm}
\end{table}

\begin{table}[t]
\caption{Performance of ASFL with Pixel DP in accuracy (\%). The radius of threat model, $R=0.1$.}
\label{tab:noise}
\begin{center}
\scriptsize
\begin{tabular}{clcccc}  
    \toprule
    \textbf{Noise type} & \textbf{Noise level} & \textbf{CIFAR-10} & \textbf{CIFAR-100} & \textbf{SVHN} & \textbf{FEMNIST}   \\
    \midrule
    \multirow{4}{*}{Laplace} & $\epsilon=2.0$ & $74.69 \pm 0.97$ & $49.09 \pm 0.59$ & $88.49 \pm 0.34$ &  $85.84 \pm 0.19$  \\
    & $\epsilon=1.0$ & $74.42 \pm 0.97$ & $47.85 \pm 0.58$ & $88.88 \pm 0.40$ &  $85.76 \pm 0.19$  \\
    & $\epsilon=0.5$ & $74.14 \pm 0.47$ & $47.30 \pm 0.82$ & $88.80 \pm 0.38$ &  $85.29 \pm 0.19$  \\
    & $\epsilon=0.2$ & $69.83 \pm 0.97$ & $45.11 \pm 1.02$ & $87.07 \pm 0.41$ &  $84.77 \pm 0.31$  \\
    \midrule
    \multirow{3}{*}{Gaussian} & $\delta=0.05$ & $73.02 \pm 0.84$ & $46.18 \pm 0.39$ & $87.88 \pm 0.59$ &  $85.27 \pm 0.08$  \\
    & $\delta=0.01$ & $71.07 \pm 0.64$ & $45.56 \pm 1.02$ & $86.67 \pm 0.80$ &  $84.88 \pm 0.30$  \\
    & $\delta=0.001$ & $68.55 \pm 0.70$ & $45.19 \pm 0.05$ & $85.95 \pm 0.81$ &  $84.96 \pm 0.08$  \\
    \bottomrule
\end{tabular}
\end{center}
\vspace{-0.6cm}
\end{table}

\subsection{System efficiency analysis}
\label{app:local_training_time}

\textbf{Per-client round wall-clock time} --
We compare the per-client round wall-clock time for local computation and report the results in Table~\ref{tab:local_training_time}, normalized to FedAvg on each dataset.
All methods use the same total local update budget per communication round (See Appendix~\ref{appendix:main-comparison-protocol}).
ASFL divides this budget equally between tail-round updates and standard local updates ($\tau'=\tau$), whereas methods without a tail round devote the entire budget to standard local training.
We synchronize CUDA at both the start and the end of each timed region.
For ASFL, the measured local computation time includes both the training loop and the forward passes used to generate shared activations, which we time separately and sum.

ASFL achieves the lowest local computation wall-clock time on all four datasets under the same total local update budget, even after including shared activation generation.
The tail round contributes to this reduction through partial training: its updates train only part of the model, reducing backward-pass computation per update.
Thus, allocating part of the fixed local update budget to the tail round provides a computational advantage while preserving the total number of local updates.

\textbf{Learning curves against communication cost} --
We present an end-to-end learning curve comparison in Figure~\ref{fig:e2e}.
The x-axis is the communication cost with respect to the transferred data size in Bytes.
The total number of communication rounds is fixed across all the methods.
ASFL achieves the highest accuracy compared to several baselines.
In particular, we observe that ASFL has a considerably cheaper communication cost within the same number of communication rounds compared to SplitFed and SCAFFOLD.
These results clearly demonstrate that ASFL makes a practical trade-off between the additional communication cost and the faster convergence.

\begin{figure}[ht!]
\begin{center}
\includegraphics[width=1.0\columnwidth]{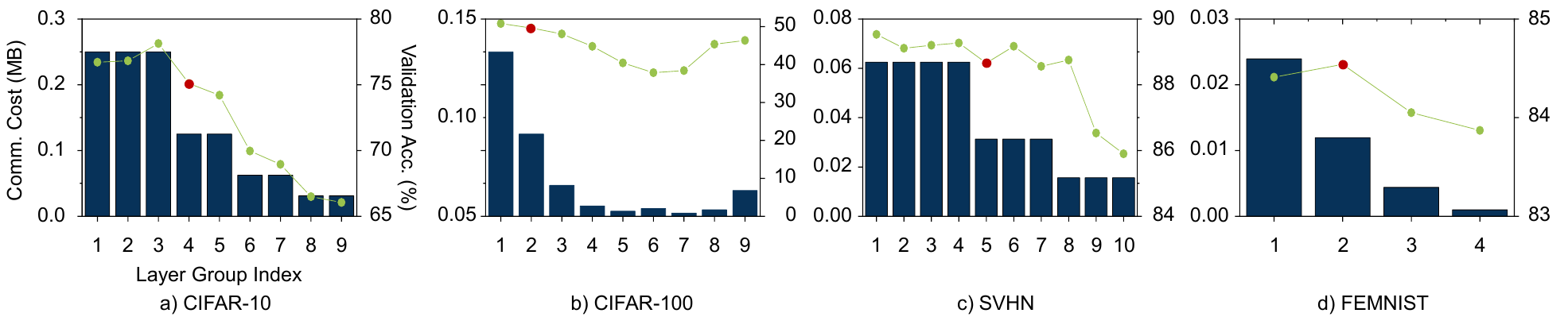}
\vspace{-0.7cm}
\end{center}
\caption{
    Ablation study on selection of activation sharing layers.
    The left and right y-axes represent the communication cost and the validation accuracy, respectively. The x-axis denotes the layer index.
}
\label{fig:ablation}
\vspace{-0.4cm}
\end{figure}

\subsection{ASFL with differential privacy}
We evaluate the privacy--utility trade-off of ASFL with PixelDP under Laplace and Gaussian noise across varying privacy budgets (Table~\ref{tab:noise}).
As expected, stronger privacy guarantees (e.g., smaller $\epsilon$ for Laplace or smaller $\delta$ for Gaussian) lead to larger injected noise, which in turn degrades model performance. 
However, even under moderate noise levels, ASFL continues to outperform other FL methods shown in Table~\ref{tab:comparison} that do not even employ differential privacy techniques.
This highlights that ASFL provides a superior privacy-utility trade-off, making it a highly practical strategy for FL preserving privacy.

\textbf{Label DP} --
Labels shared in ASFL are exposed and can be protected with existing label DP techniques.
To verify that ASFL remains trainable under label DP, we applied $K$-ary randomized response (RR), which keeps each label with probability $\frac{e^{\epsilon_\ell}}{e^{\epsilon_\ell}+K-1}$ and otherwise replaces it with one of the other $K-1$ classes uniformly at random.
ASFL maintains competitive accuracy under label DP (Table~\ref{tab:label_dp}).

\textbf{Additional privacy experiments} --
We further evaluate privacy through reconstruction and membership inference attacks.
PixelDP reduces the quality of inputs reconstructed from shared activations under learning-based and DRAG inversion attacks.
With gradient perturbation, ASFL achieves higher accuracy than SplitFed while yielding similar membership inference attack metrics under the evaluated settings.
Further privacy analyses and empirical results are provided in Appendix~\ref{appendix:privacy}.

\begin{table}[t]
\caption{Performance of ASFL under PixelDP with and without label DP in terms of accuracy (\%). The threat model radius is $R=0.1$. ``--'' in the $\epsilon_\ell$ column denotes no label DP.}
\label{tab:label_dp}
\begin{center}
\setlength{\tabcolsep}{4pt}
\renewcommand{\arraystretch}{1.15}
\scriptsize
\begin{tabular}{lccccccc}
\toprule
\textbf{Noise type}
& $\epsilon$
& $\delta$
& $\epsilon_\ell$
& \textbf{CIFAR-10}
& \textbf{CIFAR-100}
& \textbf{SVHN}
& \textbf{FEMNIST} \\
\midrule
\multirow{3}{*}{Laplacian}
& \multirow{3}{*}{$2.0$}
& \multirow{3}{*}{$0$}
& -- & $74.69 \pm 0.97$ & $49.09 \pm 0.59$ & $88.49 \pm 0.34$ & $85.84 \pm 0.19$ \\
& & & $5.0$ & $74.39 \pm 0.63$ & $46.87 \pm 0.22$ & $88.80 \pm 0.49$ & $85.63 \pm 0.16$ \\
& & & $2.0$ & $68.80 \pm 0.66$ & $41.39 \pm 0.87$ & $86.52 \pm 0.65$ & $84.64 \pm 0.43$ \\
\midrule
\multirow{3}{*}{Gaussian}
& \multirow{3}{*}{$0.5$}
& \multirow{3}{*}{$0.05$}
& -- & $73.02 \pm 0.84$ & $46.18 \pm 0.39$ & $87.88 \pm 0.59$ & $85.27 \pm 0.08$ \\
& & & $5.0$ & $72.93 \pm 0.38$ & $45.46 \pm 0.26$ & $88.00 \pm 0.21$ & $85.10 \pm 0.22$ \\
& & & $2.0$ & $70.92 \pm 0.28$ & $43.33 \pm 0.06$ & $86.50 \pm 0.39$ & $85.19 \pm 0.28$ \\
\bottomrule
\end{tabular}
\end{center}
\end{table}

\begin{figure}[t]
\begin{center}
\includegraphics[width=0.9\columnwidth]{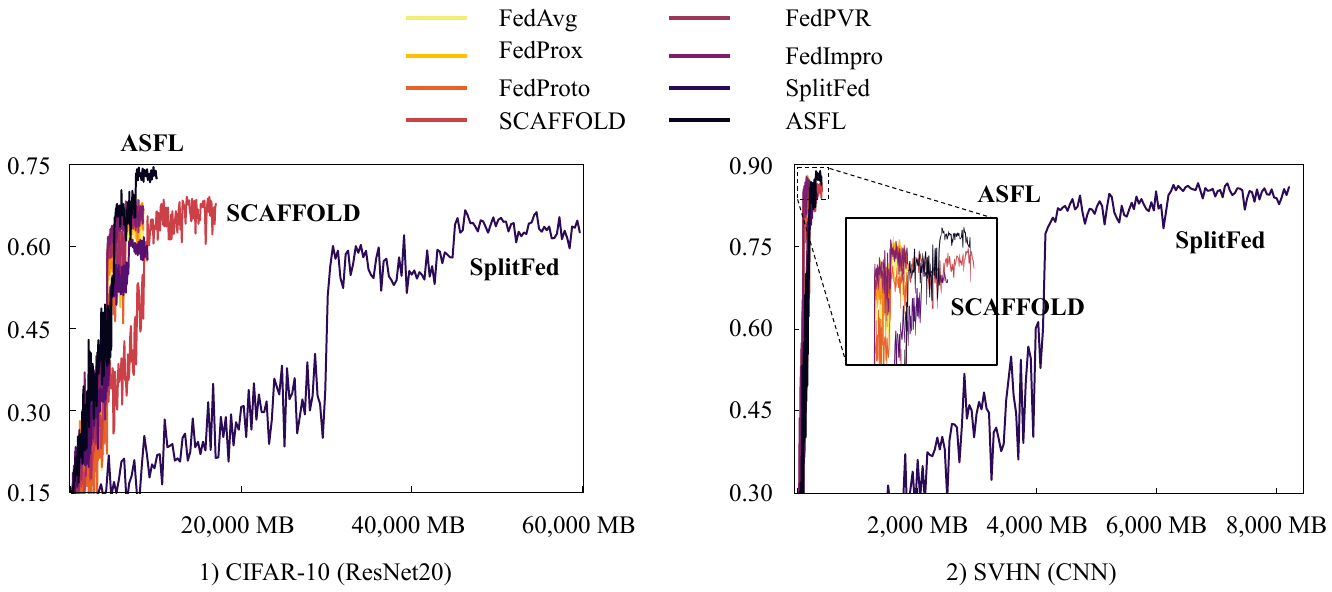}
\vspace{-0.3cm}
\end{center}
\caption{
    End-to-end learning curves against communication cost (transferred data size in Bytes).
}
\label{fig:e2e}
\end{figure}

\subsection {Ablation study}

\textbf{Effect of exchange layer selection} -- Figure~\ref{fig:ablation} shows the communication cost and model accuracy under different exchange layer settings.
The communication cost is calculated by adding the size of model parameters and activations.
The red dot indicates the layer automatically selected by our proposed method.
When activations are shared at layers closer to the output, both the communication cost and accuracy tend to decrease.
One exception is CIFAR-100, where both communication cost and accuracy increase at layers $8$ and $9$.
Nevertheless, our method consistently identifies a practical layer that significantly reduces communication cost while marginally degrading accuracy.

\textbf{Effect of the number of exchanged activations} --
We conduct extra experiments with varying numbers of exchanged activations.
Due to the limited space, we present the results in Appendix~\ref{appendix:activation}.

\section {Conclusion} \label{sec:conclusion}
We propose ASFL, an activation sharing framework for FL that mitigates data heterogeneity across clients.
Our theoretical analysis provides insights into its convergence behavior, and extensive experiments demonstrate its effectiveness across diverse FL benchmarks.
While activation sharing may raise privacy concerns, we investigate this aspect using PixelDP, a recently proposed representative DP method, and provide empirical insights.
Overall, our results suggest that activation sharing is a practical and effective approach for achieving fast and accurate FL across a wide range of applications.
As the present study focuses on a practical gossip-based communication pattern, future work could extend this framework by aggregating activations at a central server and exploring alternative strategies for their utilization.



\clearpage
\subsection*{AI use statement}
In this work, we used generative AI tools solely to assist with writing in English and to provide assistance with writing, modifying, and debugging experimental code.
We did not use these tools for any other purposes.
All AI-assisted text was reviewed and revised by the authors, and all AI-assisted code was manually inspected and tested for correctness.

\subsection*{Reproducibility statement}
To facilitate reproducibility, Section~\ref{appendix:settings} provides experimental settings and implementation details for all experiments reported in the main manuscript.
These include the hardware and software environment, client sampling and data partitioning procedures, optimization settings, and allocation of local update budgets across methods.
Tables~\ref{tab:hyperparams} and~\ref{tab:baseline-hyperparams} report the dataset-specific training settings and method-specific hyperparameters, respectively.
We plan to publicly release our source code upon publication.

\bibliography{iclr2027_conference}
\bibliographystyle{iclr2027_conference}

\clearpage
\appendix

\section {Appendix}

The appendix is structured as follows:
\begin{itemize}
\item Section~\ref{appendix:settings} summarizes experimental settings corresponding to all the experimental results reported in the main manuscript
\item Section~\ref{appendix:theory} provides detailed proofs of the proposed lemma, theorem, and corollary.
\item Section~\ref{appendix:communication-pattern} provides additional details on the gossip-style communication pattern used to transfer latent-space information among clients.
\item Section~\ref{appendix:privacy} presents additional experimental results and analyses regarding privacy.
\item Section~\ref{appendix:supplementary} provides supplementary experimental analyses that further support the empirical findings reported in the main manuscript.
\end{itemize}

\subsection{Experimental details}
\label{appendix:settings}

This section summarizes the experimental details required to reproduce and interpret the results reported in the main manuscript, including training settings, activation-sharing configurations, cost estimation, and gradient-variance measurement.

\subsubsection{Main comparison protocol}
\label{appendix:main-comparison-protocol}
\begin{table}[ht!]
\caption{Training settings used for the main comparison in Table~\ref{tab:comparison}. Each column corresponds to the dataset--backbone pair used in the main experiments.}
\label{tab:hyperparams}
\begin{center}
\footnotesize
\setlength{\tabcolsep}{4.5pt}
\renewcommand{\arraystretch}{1.08}
\begin{tabular}{lcccc}
    \toprule
    \textbf{Hyperparameter} & \textbf{CIFAR-10} & \textbf{CIFAR-100} & \textbf{SVHN} & \textbf{FEMNIST} \\
    & \textbf{(ResNet18)} & \textbf{(MobileNetV2)} & \textbf{(ResNet20)} & \textbf{(CNN)} \\
    \midrule
    Learning rate & 0.3 & 0.05 & 0.1 & 0.01 \\
    Batch size & 64 & 64 & 32 & 20 \\
    Weight decay & 1e-4 & 8e-4 & 1e-4 & 1e-4 \\
    LR decay rounds & 100, 150 & 100, 150 & 100, 150 & -- \\
    Communication rounds & 200 & 200 & 200 & 200 \\
    Total local update budget & 32 & 80 & 40 & 12 \\
    ASFL tail updates ($\tau'$) & 16 & 40 & 20 & 6 \\
    ASFL standard local updates ($\tau$) & 16 & 40 & 20 & 6 \\
    \bottomrule
\end{tabular}
\end{center}
\end{table}

We use a unified experimental protocol for the main comparison in Table~\ref{tab:comparison}.
Experiments are run on two NVIDIA RTX 4070 GPUs using PyTorch 2.7.0 with CUDA 11.8 and MPI-based distributed execution.
Each run uses a total of 100 clients, with 10 clients sampled uniformly at random in every communication round.
Client data are partitioned using a Dirichlet distribution with concentration parameter $\alpha=0.1$, which creates a strongly non-IID label distribution across clients.
All methods are optimized with mini-batch SGD with momentum 0.9.
Table~\ref{tab:hyperparams} reports the training settings for each dataset--backbone pair.
For ASFL, the local computation budget in each communication round is divided equally between the tail round and the subsequent standard local training.
Specifically, we set $\tau'=\tau$, where $\tau'$ denotes the number of tail-round updates and $\tau$ denotes the number of standard local updates.
For methods without a tail round, the same total local update budget is used entirely for standard local training.
In each communication round, ASFL shares activations corresponding to one mini-batch, regardless of the local update budget.
For the baseline methods, we use the same settings reported in Table~\ref{tab:hyperparams}, except that the learning rate is selected by grid search.

\begin{table}[ht!]
\caption{Method-specific hyperparameters used in the main comparison.}
\label{tab:baseline-hyperparams}
\begin{center}
\footnotesize
\setlength{\tabcolsep}{2pt}
\renewcommand{\arraystretch}{1.08}
\begin{tabular}{llcccc}
    \toprule
    \textbf{Method} & \textbf{Hyperparameter} & \textbf{CIFAR-10} & \textbf{CIFAR-100} & \textbf{SVHN} & \textbf{FEMNIST} \\
    & & \textbf{(ResNet18)} & \textbf{(MobileNetV2)} & \textbf{(ResNet20)} & \textbf{(CNN)} \\
    \midrule
    FedProx & $\mu$ (proximal coefficient)
    & 0.001 & 0.001 & 0.001 & 0.001 \\
    FedProto & $\lambda$ (prototype regularization coefficient)
    & 1.0 & 1.0 & 1.0 & 1.0 \\
    MOON & $\mu$ (contrastive coefficient), $T$ (temperature)
    & 1.0, 0.5 & 1.0, 0.5 & 0.1, 0.5 & 1.0, 0.5 \\
    FedMDVR & $\alpha$ (control-variate update coefficient)
    & 5.0 & 9.0 & 5.0 & 8.0 \\
    \bottomrule
\end{tabular}
\end{center}
\end{table}

Table~\ref{tab:baseline-hyperparams} summarizes the method-specific hyperparameters.
FedProx uses $\mu$ to penalize deviations of local model parameters from the current global model, while FedProto uses $\lambda$ to encourage alignment between local representations and their corresponding global class prototypes.
MOON uses a contrastive loss, weighted by $\mu$ and evaluated at temperature $T$, to align local representations with those of the current global model and separate them from those of the client's previous local model.
In FedMDVR, $\alpha$ denotes the control-variate update coefficient.

\subsubsection{Activation-sharing positions}
ASFL defines candidate activation-sharing positions at the level of backbone-specific layer blocks, using the same indexing convention as the x-axis of Figure~\ref{fig:cv}.
Each layer block is defined as the smallest architectural unit that preserves the backbone's block structure and ends after a nonlinearity.
The candidate position is indexed by the layer block at which the tail-side partial update starts.
For a selected layer block $b$, ASFL forwards a mini-batch through the preceding blocks and exchanges the activation at the input of block $b$.
The selected block is therefore included in the partial training range.
The input boundary is excluded from the candidate set because it corresponds to exchanging the original mini-batch input rather than an intermediate activation.

\begin{table}[ht!]
\caption{Candidate activation-sharing positions for each dataset--backbone pair. The number of candidate positions excludes the input boundary but includes the final classifier position.}
\label{tab:layer-groups}
\begin{center}
\footnotesize
\setlength{\tabcolsep}{6pt}
\begin{tabular}{llcc}
    \toprule
    \textbf{Dataset} & \textbf{Backbone} & \textbf{No. of candidate positions} & \textbf{Selected layer block} \\
    \midrule
    CIFAR-10 & ResNet18 & 9 & 4 \\
    CIFAR-100 & MobileNetV2 & 9 & 2 \\
    SVHN & ResNet20 & 10 & 5 \\
    FEMNIST & CNN & 4 & 2 \\
    \bottomrule
\end{tabular}
\end{center}
\end{table}

Table~\ref{tab:layer-groups} summarizes the candidate positions and selected layer blocks for each dataset--backbone pair.
The selected layer block is chosen according to the cost--variance trade-off shown in Figure~\ref{fig:cv}.

\subsubsection{Activation-sharing quantity}
The number of shared batches controls the size of the peer activation pool exchanged during the tail round.
Table~\ref{tab:batches} summarizes the effect of this quantity on accuracy and communication cost.
In this ablation, the activation-sharing position is fixed to the selected layer block for each dataset--backbone pair, and only the number of shared batches is varied.
One shared batch denotes the activations generated from one mini-batch, with the mini-batch size set according to Table~\ref{tab:hyperparams}.

As described in the communication-efficient activation sharing scheme in Algorithm~\ref{alg:LIS}, the received peer activation information $I_j$ is reused when constructing the random mini-batches for the tail round.
Thus, the number of shared batches in Table~\ref{tab:batches} controls the amount of peer activation information exchanged through the directed gossip communication pattern, rather than the total number of tail-round mini-batches used to update the tail subnetwork.

\subsubsection{Communication and computational cost}
For each candidate activation-sharing position, we compute the communication cost as the number of activation elements exchanged at that position. Specifically, for a candidate layer block $b$, the transmitted activation is the input to block $b$, and its communication cost is measured as the product of its tensor dimensions.

For the computational cost, we estimate the backpropagation cost of the tail subnetwork starting from each candidate layer block.
The estimate focuses on the convolutional and linear layers included in the tail subnetwork, which are the main parameterized computational components of the backbone.
We use this layer-wise estimate to compare the relative cost of tail-side partial training across candidate positions. Figure~\ref{fig:cv} visualizes the resulting normalized cost curve together with the cumulative gradient-variance curve for layer-block selection.

\subsubsection{Gradient variance measurement}
This subsection explains how we estimate the gradient-variance term in the cost--variance trade-off of Section~\ref{sec:tradeoff}.
For the variance-reduction curve $\mathcal{V}(k)=\sum_{l=1}^{k-1}\mathrm{Var}(g^l)$, we estimate each layer-block term under the client-sampling setting in Section~\ref{appendix:main-comparison-protocol}, using one local update step on each active client $i \in \mathcal{A}$.

For each active client, we compute per-sample gradients of the cross-entropy loss with respect to the trainable parameters of each layer block.
The parameters in a block are vectorized into a single gradient vector, and the block-wise variance is measured as the average squared deviation of the per-sample gradients from their mini-batch mean:
\begin{align*}
    \mathrm{Var}_i(g^l)
    = \frac{1}{B} \sum_{j=1}^{B} \left\| g_{i,j}^{l} - \bar{g}_{i}^{l} \right\|_2^2.
\end{align*}
This corresponds to the trace of the per-sample gradient covariance for layer block $l$ on client $i$.
We average this quantity over the active clients to obtain $\mathrm{Var}(g^l)$. We then construct $\mathcal{V}(k)$ as a prefix sum over the candidate layer blocks before block $k$, excluding the input boundary, and apply min--max normalization for the layer selection shown in Figure~\ref{fig:cv}.
\clearpage
\subsection {Convergence analysis of partial model updates with activation sharing} \label{appendix:theory}

To make this appendix self-contained, we provide the problem definition and assumptions that are provided in the main text.
We consider non-convex and smooth optimization problems:
\begin{align}
    \min_{\mathbf{x} \in \mathbb{R}^d} F(\mathbf{x}) := \frac{1}{m} \sum_{i=1}^m F_i(\mathbf{x}), \nonumber
\end{align}
where $F_i(\mathbf{x}) = \mathbb{E}_{\xi_i \sim D_i}[f(\mathbf{x}, \xi_i)]$ is the local loss function associated with the local data distribution $D_i$ of client $i$ and $m$ is the number of clients.
Our analysis is based on the following assumptions.

\textbf{Assumption 1} (Lipschitz continuous gradients).
    \textit{There exists a constant $\mathcal{L}>0$, such that $\| \nabla F_i(\mathbf{a}) - \nabla F_i(\mathbf{b}) \| \leq \mathcal{L} \| \mathbf{a} - \mathbf{b} \|, \forall \mathbf{a},\mathbf{b} \in \mathbb{R}^d, \textrm{ and } i\in[m]$.
    }

\textbf{Assumption 2} (Unbiased local gradients).
    \textit{The local gradient estimator is unbiased such that $\mathbb{E}_{\xi_i \sim D_i}[\nabla f(\mathbf{a}, \xi_i)] = \nabla F_i(\mathbf{a}), \forall i \in [m]$.
    }

\textbf{Assumption 3} (Bounded local and global variances).
    \textit{There exist two constants $\sigma_x > 0$ and $\sigma_G > 0$, such that the local gradient variance is bounded by $\mathbb{E}[\| \nabla f(\mathbf{x}, \xi_i) - \nabla F_i(\mathbf{x}) \|^2] \leq \sigma_x^2, \forall i \in [m]$, and the global variability is bounded by $\| \nabla F_i(\mathbf{x}) - \nabla F(\mathbf{x}) \|^2 \leq \sigma_G^2, \forall i \in [m]$.
    }

\textbf{Assumption 4} (Bounded local variance of a partial model).
    \textit{There exists a constant $\sigma_z > 0$ such that the local gradient variance is bounded by $\mathbb{E}[\| \nabla f(\mathbf{z}, \xi_i) - \nabla F_i(\mathbf{z}) \|^2] \leq \sigma_z^2, \forall i \in [m]$.
    }

Each communication round consists of $H$ local steps, and $H = \tau + \tau'$, where $\tau$ and $\tau'$ are the number of full-model steps and tail steps, respectively.
Consequently, the ratio of full-model steps within each round is defined as $\rho = \frac{\tau}{H}$.

Here, we consider client $i$'s $\tau'$ local updates in a tail round use not only its local activations but also remote activations received from a random client $k$.
Without loss of generality, in the tail round, a fraction $\alpha$ of the data consists of activations from local data, while the remaining fraction $1 - \alpha$ consists of activations from remote data.
Due to the linearity of expectation over the mini-batch sampling, the expected gradient with respect to $z$, computed from random mini-batches of these activations, is given by
\begin{align}
\mathbb{E}[ \tilde{p}_{t,j}^i ] = \alpha \nabla_z F_i(\mathbf{x}_{t,j}^i) + (1 - \alpha)\nabla_z F_k(\mathbf{x}_{t,j}^i), \label{eq:expected_p}
\end{align}
where the expectation is taken over the mini-batch sampling from the merged pool of local and remote activations, given the sampled client $k$.
Thus, the global update at round $t$ across $a$ randomly activated clients is defined as 
\begin{align}
    \tilde{\Delta}_t := \frac{1}{a} \sum_{i \in \mathcal{A}} \left( \sum_{j=0}^{\tau' - 1} \tilde{p}_{t,j}^i + \sum_{j=0}^{\tau - 1} g_{t,j}^i \right).
\end{align}

We first define two auxiliary vectors for analyzing tail rounds:
\begin{align}
    \mathbf{d}_i &:= \nabla_z F_i(\mathbf{x}_t) - \nabla_z F(\mathbf{x}_t) \nonumber \\
    \mathbf{h}_t^i &:= \alpha \mathbf{d}_i + (1 - \alpha) \mathbf{d}_k,\nonumber    
\end{align}
where $\alpha$ is the fraction of local activations.
By the definition, $\sum_{i=1}^m \mathbf{d}_i = 0$.

\clearpage
\begin{lemma}[Mixing heterogeneity]\label{lemma:mixed_het}
    For each $i \in \mathcal{A}$ and its randomly selected client $k$, under Assumption 3, the following condition is true:
\begin{align}
    \mathbb{E}\left[ \frac{1}{a} \sum_{i \in \mathcal{A}} \| \mathbf{h}_t^i \|^2 \right] \leq \gamma \sigma_G^2,
\end{align}
where $\gamma := \alpha^2 + (1 - \alpha)^2$.
\end{lemma}
\begin{proof}
\begin{align}
    \mathbb{E}\left[ \frac{1}{a} \sum_{i \in \mathcal{A}} \| \mathbf{h}_t^i \|^2 \right] &= \mathbb{E}\left[ \frac{1}{a} \sum_{i \in \mathcal{A}} \| \alpha \mathbf{d}_i + (1 - \alpha)\mathbf{d}_k \|^2 \right] \nonumber \\
    &=\frac{1}{m(m-1)} \sum_{l \neq l'} \left\| \alpha \mathbf{d}_l + (1 - \alpha) \mathbf{d}_{l'} \right\|^2 \nonumber \\
    &= \frac{1}{m(m-1)} \sum_{l \neq l'} \left( \alpha^2 \| \mathbf{d}_l \|^2 + (1 - \alpha)^2 \| \mathbf{d}_{l'} \|^2 + 2\alpha (1 - \alpha) \langle \mathbf{d}_l, \mathbf{d}_{l'} \rangle \right) \nonumber \\
    &\leq \frac{1}{m(m-1)} \sum_{l \neq l'} \left( \alpha^2 \| \mathbf{d}_l \|^2 + (1 - \alpha)^2 \| \mathbf{d}_{l'} \|^2 \right), \label{eq:lemma4}
\end{align}
where (\ref{eq:lemma4}) holds because the right-most term on the right-hand side is negative since:
\begin{align}
    \frac{1}{m(m-1)} \sum_{l \neq l'} \langle \mathbf{d}_l, \mathbf{d}_{l'} \rangle  &= \frac{1}{m(m - 1)} \left( \left\| \sum_{l=1}^m \mathbf{d}_l \right\|^2 - \sum_{l=1}^m \left\| \mathbf{d}_i \right\|^2 \right) \nonumber \\
    &= - \frac{1}{m(m - 1)} \sum_{l=1}^m \left\| \mathbf{d}_i \right\|^2, \label{eq:zo}
\end{align}
where (\ref{eq:zo}) holds because $\sum_{i=1}^m \mathbf{d}_i = 0$.

Then, we can further bound as follows:
\begin{align}
    \mathbb{E}\left[ \frac{1}{a} \sum_{i \in \mathcal{A}} \| \mathbf{h}_t^i \|^2 \right] &\leq \frac{1}{m(m-1)} \sum_{l \neq l'} \left( \alpha^2 \| \mathbf{d}_l \|^2 + (1 - \alpha)^2 \| \mathbf{d}_{l'} \|^2 \right) \nonumber \\
    &= \frac{\alpha^2}{m(m-1)} \sum_{l \neq l'} \left\| \mathbf{d}_l \right\|^2 + \frac{(1 - \alpha)^2}{m(m-1)} \sum_{l \neq l'} \left\| \mathbf{d}_{l'} \right\|^2 \nonumber \\
    &= \frac{\alpha^2}{m} \sum_{l=1}^m \left\| \mathbf{d}_l \right\|^2 + \frac{(1 - \alpha)^2}{m} \sum_{l'=1}^m \left\| \mathbf{d}_{l'} \right\|^2 \label{eq:sum} \\
    &\leq \left(\alpha^2 + (1 - \alpha)^2 \right) \sigma_G^2 = \gamma \sigma_G^2, \label{eq:lemma4}
\end{align}
where (\ref{eq:sum}) holds because $\sum_{l \neq l'} \| \mathbf{d}_l \|^2 = \sum_{l\neq l'} \| \mathbf{d}_{l'} \|^2 = (m-1) \sum_{l=1}^m \| \mathbf{d}_l \|^2$ and (\ref{eq:lemma4}) holds based on Assumption 3.
\end{proof}

\clearpage
\begin{lemma}[The sum of mixed gradients in the tail round] \label{lemma:local}
Under Assumptions 1--4 and the sampling scheme of Lemma~\ref{lemma:mixed_het}, the following condition holds:
\begin{align}
    \mathbb{E}\left[ \left\| \sum_{i\in \mathcal{A}} \sum_{j=0}^{\tau' - 1} \mathbb{E}[\tilde{p}_{t,j}^i] \right\|^2 \right] &\leq 
    6a \tau' \mathcal{L}^2 \mathbb{E}\left[ \sum_{i\in\mathcal{A}} \sum_{j=0}^{\tau'-1} \left\| \mathbf{z}_{t,j}^i - \mathbf{z}_t \right\|^2 \right] \nonumber \\
    &\qquad + 3a(\tau')^2 \left( \frac{m-a}{m-1} + (1-\alpha)^2 \right)\sigma_G^2 \nonumber \\
    &\qquad + 3 a^2 (\tau')^2 \mathbb{E}\left[ \left\| \nabla_z F(\mathbf{x}_t) \right\|^2 \right]. \nonumber
\end{align}
\end{lemma}
\begin{proof}
Let us define $\mathbf{g}_t^i := \alpha \nabla_z F_i(\mathbf{x}_t) + (1 - \alpha) \nabla_z F_k(\mathbf{x}_t)$.
Then, by adding and subtracting $\mathbf{g}_t^i$ and $\nabla_z F(\mathbf{x}_t)$, we have:
\begin{align}
    \mathbb{E}\left[ \left\| \sum_{i\in \mathcal{A}} \sum_{j=0}^{\tau' - 1} \mathbb{E}[\tilde{p}_{t,j}^i] \right\|^2 \right]
    &= \mathbb{E}\left[ \left\| \sum_{i\in\mathcal{A}} \sum_{j=0}^{\tau'-1} \mathbb{E}\left[ \Tilde{p}_{t,j}^i \right] - \mathbf{g}_t^i +\mathbf{g}_t^i -\nabla_z F(\mathbf{x}_t) + \nabla_z F(\mathbf{x}_t) \right\|^2 \right] \nonumber \\
    &\leq 3 \underbrace{\mathbb{E}\left[ \left\| \sum_{i\in\mathcal{A}} \sum_{j=0}^{\tau'-1} \mathbb{E}\left[ \Tilde{p}_{t,j}^i \right] - \mathbf{g}_t^i \right\|^2 \right]}_{T_1} + 3\underbrace{\mathbb{E}\left[ \left\| \sum_{i\in\mathcal{A}} \sum_{j=0}^{\tau'-1} \mathbf{g}_t^i -\nabla_z F(\mathbf{x}_t) \right\|^2 \right]}_{T_2} \nonumber \\
    &\qquad + 3 a^2 (\tau')^2 \mathbb{E}\left[ \left\| \nabla_z F(\mathbf{x}_t) \right\|^2 \right], \label{eq:jensen}
\end{align}
where (\ref{eq:jensen}) holds due to the convexity of $\ell_2$ norm and Jensen's inequality.

\textbf{Bounding $T_1$}.

Since $\mathbb{E}[\tilde{p}_{t,j}^i] - \mathbf{g}_t^i = \alpha (\nabla_z F_i(\mathbf{x}_{t,j}^i) - \nabla_z F_i(\mathbf{x}_t)) + (1 - \alpha)(\nabla_z F_k(\mathbf{x}_{t,j}^i) - \nabla_z F_k(\mathbf{x}_t))$, we have:
\begin{align}
    T_1 &= \mathbb{E}\left[ \left\| \sum_{i\in\mathcal{A}} \sum_{j=0}^{\tau'-1} \mathbb{E}\left[ \Tilde{p}_{t,j}^i \right] - \mathbf{g}_t^i \right\|^2 \right] \nonumber \\
    & \leq 2\mathbb{E}\left[ \left\| \sum_{i\in\mathcal{A}} \sum_{j=0}^{\tau'-1} \alpha (\nabla_z F_i(\mathbf{x}_{t,j}^i) - \nabla_z F_i(\mathbf{x}_t)) \right\|^2 \right] \nonumber \\
    &\qquad + 2\mathbb{E}\left[ \left\| \sum_{i\in\mathcal{A}} \sum_{j=0}^{\tau'-1} (1 - \alpha)(\nabla_z F_k(\mathbf{x}_{t,j}^i) - \nabla_z F_k(\mathbf{x}_t)) \right\|^2 \right] \nonumber \\
    &\leq 2a \tau' \alpha^2 \mathbb{E}\left[ \sum_{i\in\mathcal{A}} \sum_{j=0}^{\tau'-1} \left\| \nabla_z F_i(\mathbf{x}_{t,j}^i) - \nabla_z F_i(\mathbf{x}_t) \right\|^2 \right] \nonumber \\
    &\qquad + 2a \tau' (1 - \alpha)^2 \mathbb{E}\left[ \sum_{i\in\mathcal{A}} \sum_{j=0}^{\tau'-1} \left\| \nabla_z F_k(\mathbf{x}_{t,j}^i) - \nabla_z F_k(\mathbf{x}_t) \right\|^2 \right] \nonumber \\
    &\leq 2a \tau' \mathcal{L}^2 \mathbb{E}\left[ \sum_{i\in\mathcal{A}} \sum_{j=0}^{\tau'-1} \left\| \mathbf{z}_{t,j}^i - \mathbf{z}_t \right\|^2 \right]. \label{eq:T_1}
\end{align}

\textbf{Bounding $T_2$}.

Let $\mathbb E_k[\cdot]$ denote the expectation over $k$ conditioned on $\mathbf{x}_t$ and $\mathcal{A}$.
By the tower property, $\mathbb{E}[\cdot]=\mathbb{E}[\mathbb{E}_k[\cdot]]$, which gives
\begin{align}
    T_2 &= \mathbb{E}\left[ \left\| \sum_{i\in\mathcal{A}} \sum_{j=0}^{\tau'-1} \mathbf{g}_t^i -\nabla_z F(\mathbf{x}_t) \right\|^2 \right] \nonumber \\
    &= \mathbb{E}\left[ \left\| \sum_{i\in \mathcal{A}} \sum_{j=0}^{\tau' - 1} \alpha \nabla_z F_i(\mathbf{x}_t) + (1-\alpha) \nabla_z F_k(\mathbf{x}_t) - \nabla_z F(\mathbf{x}_t) \right\|^2 \right] \nonumber \\
    &= (\tau')^2 \mathbb{E}\left[ \mathbb{E}_k \left[ \left\| \sum_{i\in \mathcal{A}} \mathbf{h}_t^i \right\|^2 \right] \right] = (\tau')^2 \mathbb{E}\left[ \underbrace{\left\| \mathbb{E}_k\left[ \sum_{i\in \mathcal{A}} \mathbf{h}_t^i \right] \right\|^2}_{T_3} + \underbrace{\mathbb{E}_k \left[ \left\|  \sum_{i \in \mathcal{A}} \mathbf{h}_t^i - \mathbb{E}_k \left[ \sum_{i\in \mathcal{A}} \mathbf{h}_t^i \right] \right\|^2 \right]}_{T_4} \right] \label{eq:T_2_frame}
\end{align}
When bounding $T_3$ and $T_4$, we consider two randomness, one for sampling activation clients, $\mathcal{A}$, and the other for sampling client $k$ by client $i$. We first consider the expectation over $k$ as follows:

\textbf{Bounding $T_3$}.
\begin{align}
    T_3 &= \left\| \mathbb{E}_k\left[ \sum_{i\in \mathcal{A}} \mathbf{h}_t^i \right] \right\|^2 = \left\| \mathbb{E}_k\left[ \sum_{i\in \mathcal{A}} \left( \alpha \mathbf{d}_i + (1 - \alpha)\mathbf{d}_k \right) \right] \right\|^2 \nonumber \\
    &= \left\| \alpha \sum_{i \in \mathcal{A}} \mathbf{d}_i + \mathbb{E}_k \left[ (1 - \alpha) \sum_{i\in \mathcal{A}} \mathbf{d}_k \right] \right\|^2 \nonumber \\
    &= \left\| \alpha \sum_{i \in \mathcal{A}} \mathbf{d}_i + (1-\alpha) \sum_{i \in \mathcal{A}} \frac{1}{a-1} \sum_{l \in \mathcal{A}\setminus\{i\}} \mathbf{d}_l \right\|^2 \nonumber \\
    &= \left\| \alpha \sum_{i \in \mathcal{A}} \mathbf{d}_i +  (1-\alpha) \sum_{i\in \mathcal{A}} \mathbf{d}_i \right\|^2 = \left\| \sum_{i\in \mathcal{A}} \mathbf{d}_i \right\|^2. \label{eq:t_3}
\end{align}

\textbf{Bounding $T_4$}.
\begin{align}
    T_4 &= \mathbb{E}_k \left[ \left\|  \sum_{i \in \mathcal{A}} \mathbf{h}_t^i - \mathbb{E}_k\left[ \sum_{i\in \mathcal{A}} \mathbf{h}_t^i \right] \right\|^2 \right] \nonumber\\
    &=\mathbb{E}_k \left[ \left\| \sum_{i\in \mathcal{A}} \alpha \mathbf{d}_i + (1-\alpha)\mathbf{d}_k - \mathbb{E}_k \left[ \sum_{i\in \mathcal{A}} \alpha \mathbf{d}_i + (1-\alpha)\mathbf{d}_k \right] \right\|^2 \right] \nonumber \\
    &=\mathbb{E}_k \left[ \left\| \sum_{i\in \mathcal{A}} (1-\alpha)\mathbf{d}_k - \mathbb{E}_k \left[ \sum_{i\in \mathcal{A}} (1-\alpha)\mathbf{d}_k \right] \right\|^2 \right] \nonumber \\
    &= (1 - \alpha)^2 \mathbb{E}_k \left[ \left\| \sum_{i\in \mathcal{A}} \left( \mathbf{d}_k - \mathbb{E}_k \left[ \mathbf{d}_k \right] \right) \right\|^2 \right] \nonumber \\
    &= (1 - \alpha)^2 \sum_{i\in \mathcal{A}} \mathbb{E}_k \left[ \left\| \mathbf{d}_k - \mathbb{E}_k \left[ \mathbf{d}_k \right] \right\|^2 \right] \label{eq:indep} \\
    &\leq (1 - \alpha)^2 \sum_{i\in \mathcal{A}} \mathbb{E}_k \left[ \left\| \mathbf{d}_k \right\|^2 \right], \label{eq:dk}
\end{align}
where (\ref{eq:indep}) holds because $k$ is independent across $i$ and (\ref{eq:dk}) holds because $\mathbb{E}[\|X - \mathbb{E}[X] \|^2] \leq \mathbb{E}[\| X \|^2]$ with $X = \mathbf{d}_k$.

By rewriting the expectation over $k$, we have:
\begin{align}
    T_4 &\leq (1 - \alpha)^2 \sum_{i\in\mathcal{A}} \frac{1}{(a-1)} \sum_{l \in \mathcal{A}\setminus \{i\}} \left\| \mathbf{d}_l \right\|^2 \nonumber \\
    &= (1 - \alpha)^2 \sum_{i\in\mathcal{A}} \left\| \mathbf{d}_i \right\|^2. \label{eq:t_4}
\end{align}

By plugging in (\ref{eq:t_3}) and (\ref{eq:t_4}) into (\ref{eq:T_2_frame}), we have:
\begin{align}
    T_2 &\leq (\tau')^2 \mathbb{E} \left[ \left\| \sum_{i\in \mathcal{A}} \mathbf{d}_i \right\|^2 + (1 - \alpha)^2 \sum_{i\in \mathcal{A}} \left\| \mathbf{d}_i \right\|^2 \right] \nonumber \\
    & = (\tau')^2 \mathbb{E}\left[ \mathbb{E}_{\mathcal{A}} \left[ \left\| \sum_{i\in \mathcal{A}} \mathbf{d}_i \right\|^2 \right] + \mathbb{E}_{\mathcal{A}} \left[ (1 - \alpha)^2 \sum_{i\in \mathcal{A}} \| \mathbf{d}_i \|^2 \right] \right] \nonumber \\
    & = (\tau')^2 \mathbb{E} \left[ \frac{a}{m} \sum_{i=1}^m \left\| \mathbf{d}_i \right\|^2 + \frac{a(a-1)}{m(m-1)} \sum_{i\neq j} \langle \mathbf{d}_i, \mathbf{d}_j \rangle  + \frac{a}{m} (1 - \alpha)^2 \sum_{i=1}^m \| \mathbf{d}_i \|^2 \right] \nonumber \\
    &\leq (\tau')^2 \mathbb{E} \left[ \frac{a}{m} \sum_{i=1}^m \left\| \mathbf{d}_i \right\|^2 + \frac{a(a-1)}{m(m-1)} \sum_{i\neq j} \langle \mathbf{d}_i, \mathbf{d}_j \rangle \right] + a(\tau')^2 (1 - \alpha)^2 \sigma_G^2 \label{eq:assumption3} \\
    &= (\tau')^2 \mathbb{E} \left[ \frac{a}{m} \sum_{i=1}^m \left\| \mathbf{d}_i \right\|^2 + \frac{a(a-1)}{m(m-1)} \left( \left\| \sum_{i=1}^m \mathbf{d}_i \right\|^2 - \sum_{i=1}^m \left\|\mathbf{d}_i \right\|^2 \right) \right] + a(\tau')^2 (1 - \alpha)^2 \sigma_G^2 \nonumber \\
    &= (\tau')^2 \mathbb{E} \left[ \frac{a}{m} \sum_{i=1}^m \left\| \mathbf{d}_i \right\|^2 - \frac{a(a-1)}{m(m-1)} \sum_{i=1}^m \left\|\mathbf{d}_i \right\|^2 \right] + a(\tau')^2 (1 - \alpha)^2 \sigma_G^2 \nonumber \\
    &= (\tau')^2 \mathbb{E} \left[ \frac{a(m-a)}{m(m-1)} \sum_{i=1}^m \left\|\mathbf{d}_i \right\|^2 \right] + a(\tau')^2 (1 - \alpha)^2 \sigma_G^2 \nonumber \\
    &\leq \frac{a(m-a)}{(m-1)} (\tau')^2 \sigma_G^2 + a(\tau')^2 (1 - \alpha)^2 \sigma_G^2 \nonumber \\
    & = a(\tau')^2 \left( \frac{m-a}{m-1} + (1-\alpha)^2 \right)\sigma_G^2, \label{eq:T_2}
\end{align}
where (\ref{eq:assumption3}) holds based on Assumption 3.

Finally, by plugging in (\ref{eq:T_1}) and (\ref{eq:T_2}) into (\ref{eq:jensen}), we have:
\begin{align}
    \mathbb{E}\left[ \left\| \sum_{i\in \mathcal{A}} \sum_{j=0}^{\tau' - 1} \mathbb{E}[\tilde{p}_{t,j}^i] \right\|^2 \right] &\leq 
    6a \tau' \mathcal{L}^2 \mathbb{E}\left[ \sum_{i\in\mathcal{A}} \sum_{j=0}^{\tau'-1} \left\| \mathbf{z}_{t,j}^i - \mathbf{z}_t \right\|^2 \right] \nonumber \\
    &\qquad + 3a(\tau')^2 \left( \frac{m-a}{m-1} + (1-\alpha)^2 \right)\sigma_G^2 \nonumber \\
    &\qquad + 3 a^2 (\tau')^2 \mathbb{E}\left[ \left\| \nabla_z F(\mathbf{x}_t) \right\|^2 \right]. \nonumber
\end{align}
\end{proof}

\clearpage
\begin{lemma}[Variance of Mixed Gradient Estimator]\label{lemma:mixed_var}
Let $\tilde{p}_i$ be a stochastic gradient defined as
\begin{align}
    \tilde{p}^i =
        \begin{cases}
            p^i & \text{with probability } \alpha, \\
            p^k & \text{with probability } 1 - \alpha, 
        \end{cases} \nonumber
\end{align}
where $p^i$ and $p^k$ are stochastic gradients satisfying
\begin{align}
\mathbb{E}[p^i] = \nabla_z F_i(\mathbf{x}), \quad \mathbb{E}[p^k] = \nabla_z F_k(\mathbf{x}).    
\end{align}
Assume that $\mathrm{Var}(p^i), \mathrm{Var}(p^k) \le \sigma_z^2$.
Then, the variance of $\tilde{p}_i$ satisfies
\[
\mathrm{Var}(\tilde{p}_i)
\le
\sigma_z^2 + 4\alpha(1-\alpha)\sigma_G^2.
\]
\end{lemma}
\begin{proof}
\begin{align}
    \mathrm{Var}(\tilde{p}^i) &= \mathbb{E}[ \| \tilde{p}^i - \mathbb{E}[ \tilde{p}^i ] \|^2] \nonumber \\
    &= \alpha \mathbb{E}[\| p^i - \mathbb{E}[\tilde{p}^i] \|^2] + (1 - \alpha) \mathbb{E}[\| p^k - \mathbb{E}[\tilde{p}^i] \|^2] \nonumber \\
    &= \alpha \mathbb{E}[\| p^i - \mathbb{E}[p^i] + \mathbb{E}[p^i] - \mathbb{E}[\tilde{p}^i] \|^2] \nonumber \\
    & \qquad + (1 - \alpha) \mathbb{E}[\| p^k - \mathbb{E}[p^k] + \mathbb{E}[p^k] - \mathbb{E}[\tilde{p}^i] \|^2] \nonumber \\
    &= \alpha \mathbb{E}[\| p^i - \mathbb{E}[p^i] \|^2] + \alpha \| \mathbb{E}[p^i] - \mathbb{E}[\tilde{p}^i] \|^2 + 2\alpha \mathbb{E}[ \langle p^i - \mathbb{E}[p^i], \mathbb{E}[p^i] - \mathbb{E}[\tilde{p}^i] \rangle ]  \nonumber \\
    & \qquad + (1 - \alpha) \mathbb{E}[\| p^k - \mathbb{E}[p^k] \|^2] + (1 - \alpha) \| \mathbb{E}[p^k] - \mathbb{E}[\tilde{p}^i] \|^2 + 2(1 - \alpha) \mathbb{E}[\langle p^k - \mathbb{E}[p^k], \mathbb{E}[p^k] - \mathbb{E}[\tilde{p}^i] \rangle ]\nonumber
\end{align}
Since $\mathbb{E}[p^i - \mathbb{E}[p^i]] = 0$ for all $i \in [N]$, the above variance is simplified as follows.
\begin{align}
    \mathrm{Var}(\tilde{p}^i) &= \alpha \mathbb{E}[\| p^i - \mathbb{E}[p^i] \|^2] + \alpha \| \mathbb{E}[p^i] - \mathbb{E}[\tilde{p}^i] \|^2 \nonumber \\
    & \qquad + (1 - \alpha) \mathbb{E}[\| p^k - \mathbb{E}[p^k] \|^2] + (1 - \alpha) \| \mathbb{E}[p^k] - \mathbb{E}[\tilde{p}^i] \|^2 \nonumber \\
    &= \alpha \mathrm{Var}(p^i) + (1 - \alpha) \mathrm{Var}(p^k) \nonumber \\
    & \qquad + \alpha \| \mathbb{E}[p^i] - \mathbb{E}[\tilde{p}^i] \|^2 + (1 - \alpha) \| \mathbb{E}[p^k] - \mathbb{E}[\tilde{p}^i] \|^2. \label{eq:var}
\end{align}
By the definition of $\mathbb{E}[\tilde{p}^i] = \alpha \mathbb{E}[p^i] + (1 - \alpha)\mathbb{E}[p^k]$, the following equations hold:
\begin{align}
    \mathbb{E}[p^i] - \mathbb{E}[\tilde{p}^i] &= (1 - \alpha)(\mathbb{E}[p^i] - \mathbb{E}[p^k]) \nonumber \\
    \mathbb{E}[p^k] - \mathbb{E}[\tilde{p}^i] &= - \alpha(\mathbb{E}[p^i] - \mathbb{E}[p^k]). \nonumber
\end{align}
By plugging in the above terms, (\ref{eq:var}) can be re-written as follows:
\begin{align}
    \mathrm{Var}(\tilde{p}^i) &= \alpha \mathrm{Var}(p^i) + (1 - \alpha) \mathrm{Var}(p^k) \nonumber \\
    &\qquad + \alpha(1-\alpha)^2 \| \mathbb{E}[p^i] - \mathbb{E}[p^k] \|^2 + (1-\alpha)\alpha^2 \| \mathbb{E}[p^i] - \mathbb{E}[p^k] \|^2 \nonumber \\
    &= \alpha \mathrm{Var}(p^i) + (1 - \alpha) \mathrm{Var}(p^k) + \alpha(1-\alpha) \| \mathbb{E}[p^i] - \mathbb{E}[p^k] \|^2 \nonumber \\
    &= \alpha \mathrm{Var}(p^i) + (1 - \alpha) \mathrm{Var}(p^k) + \alpha(1-\alpha) \| \mathbb{E}[p^i] - \nabla_z F(\mathbf{x}) + \nabla_z F(\mathbf{x}) - \mathbb{E}[p^k] \|^2 \nonumber \\
    &\leq \alpha \mathrm{Var}(p^i) + (1 - \alpha) \mathrm{Var}(p^k) + 2\alpha(1-\alpha) \| \mathbb{E}[p^i] - \nabla_z F(\mathbf{x}) \|^2 \nonumber \\
    &\qquad + 2\alpha(1-\alpha) \| \mathbb{E}[p^k] - \nabla_z F(\mathbf{x}) \|^2 \nonumber \\
    &\leq \alpha \mathrm{Var}(p^i) + (1 - \alpha) \mathrm{Var}(p^k) + 4\alpha(1-\alpha)\sigma_G^2 \nonumber \\
    &\leq \sigma_z^2 + 4\alpha(1-\alpha)\sigma_G^2 \nonumber
\end{align}
\end{proof}

\clearpage
\begin{lemma}[framework] \label{lemma:framework}
Under assumption 1 $\sim$ 4, if the learning rate $\eta \leq \min \{\frac{1}{4\mathcal{L}\tau}, \frac{1}{24\mathcal{L}\tau'}\}$, the loss reduction of a single communication round is bounded as follows.
\begin{align}
    \mathbb{E}[F(\mathbf{x}_{t+1})] &\leq F(\mathbf{x}_t) - \frac{\eta\tau}{2} \left\| \nabla F(\mathbf{x}_t) \right\|^2 - \frac{\eta\tau'}{4} \left\| \nabla_z F(\mathbf{x}_t) \right\|^2 \nonumber \\
    &\qquad + \frac{\eta\mathcal{L}^2}{a} \left( \sum_{i\in\mathcal{A}}\sum_{j=\tau'}^{H-1} \mathbb{E} \left[ \left\| \mathbf{x}_{t,j}^i - \mathbf{x}_t \right\|^2 \right] + \sum_{i\in\mathcal{A}}\sum_{j=0}^{\tau'-1} \mathbb{E} \left[ \left\| \mathbf{z}_{t,j}^i - \mathbf{z}_t \right\|^2 \right]\right) \nonumber \\
    &\qquad + \frac{2\eta^2\mathcal{L}}{a} \left( \tau \sigma_x^2 + \tau' \sigma_z^2 + \tau' \left( 4\alpha(1-\alpha) + 3\tau' \left( \frac{m-a}{m-1} + (1-\alpha)^2 \right) \right) \sigma_G^2 \right). \nonumber
\end{align}
\end{lemma}
\begin{proof}
Based on Assumption 1, taking expectation of $F(\mathbf{x}_{t+1})$, we have:
\begin{align}
    \mathbb{E}[F(\mathbf{x}_{t+1})] &\leq F(\mathbf{x}_t) + \langle \nabla F(\mathbf{x}_t), \mathbb{E}[\mathbf{x}_{t+1} - \mathbf{x}_t] \rangle + \frac{\mathcal{L}}{2} \mathbb{E}[\| \mathbf{x}_{t+1} - \mathbf{x}_t \|^2] \nonumber \\
    &= F(\mathbf{x}_t) + \langle \nabla F(\mathbf{x}_t), \mathbb{E}[ -\eta \tilde{\Delta}_t] \rangle + \frac{\mathcal{L}}{2} \mathbb{E}[\| \eta \tilde{\Delta}_t \|^2] \nonumber\\
    &= F(\mathbf{x}_t) + \frac{\mathcal{L}}{2} \mathbb{E}[\| \eta \tilde{\Delta}_t \|^2] \nonumber \\
    & \qquad + \langle \nabla F(\mathbf{x}_t), \mathbb{E}[ -\eta \tilde{\Delta}_t + \eta\tau \nabla F(\mathbf{x}_{t}) + \eta \tau' \nabla_z F(\mathbf{x}_t) - \eta\tau \nabla F(\mathbf{x}_{t}) - \eta \tau' \nabla_z F(\mathbf{x}_t)] \rangle \nonumber \\
    &= F(\mathbf{x}_t) + \frac{\mathcal{L}}{2} \mathbb{E}[\| \eta \tilde{\Delta}_t \|^2] \nonumber \\
    & \qquad -\eta \tau \| \nabla F(\mathbf{x}_t) \|^2 + \langle \nabla F(\mathbf{x}_t), \mathbb{E}\left[ -\frac{\eta}{a}\sum_{i \in \mathcal{A}} \sum_{j=\tau'}^{H - 1} g_{t,j}^i + \eta\tau \nabla F(\mathbf{x}_t) \right] \rangle \nonumber \\
    & \qquad + \langle \nabla F(\mathbf{x}_t), -\eta \tau' \nabla_z F(\mathbf{x}_t) \rangle + \langle \nabla F(\mathbf{x}_t), \mathbb{E}\left[ -\frac{\eta}{a}\sum_{i \in \mathcal{A}} \sum_{j=0}^{\tau' - 1} \tilde{p}_{t,j}^i + \eta\tau' \nabla_z F(\mathbf{x}_t) \right] \rangle \nonumber \\
    &= F(\mathbf{x}_t) -\eta \tau \| \nabla F(\mathbf{x}_t) \|^2 -\eta \tau' \| \nabla_z F(\mathbf{x}_t) \|^2 + \frac{\mathcal{L}}{2} \underbrace{ \mathbb{E}[\| \eta \tilde{\Delta}_t \|^2] }_{T_3} \nonumber \\
    & \qquad + \underbrace{ \langle \nabla F(\mathbf{x}_t), \mathbb{E}\left[ -\frac{\eta}{a}\sum_{i \in \mathcal{A}} \sum_{j=\tau'}^{H - 1} \left( g_{t,j}^i - \nabla F_i(\mathbf{x}_t) \right) \right] \rangle }_{T_1} \nonumber \\
    & \qquad + \underbrace{ \langle \nabla F(\mathbf{x}_t), \mathbb{E}\left[ -\frac{\eta}{a}\sum_{i \in \mathcal{A}} \sum_{j=0}^{\tau' - 1} \left( \tilde{p}_{t,j}^i - \nabla_z F_i(\mathbf{x}_t) \right) \right] \rangle }_{T_2}, \label{eq:frame2}
\end{align}
where $\tilde{p}_{t,j}^i$ and $\nabla_z F$ are regarded as vectors in $\mathbb{R}^d$ with zeros in the coordinates of $\mathbf{y}$, so that $\langle \nabla F(\mathbf{x}_t), \nabla_z F(\mathbf{x}_t) \rangle = \| \nabla_z F(\mathbf{x}_t) \|^2$, and (\ref{eq:frame2}) uses $\mathbb{E}[\frac{1}{a}\sum_{i \in \mathcal{A}} \nabla F_i(\mathbf{x}_t)] = \nabla F(\mathbf{x}_t)$, which holds since $\mathcal{A}$ is sampled uniformly at random.
Now, we will bound $T_1$, $T_2$, and $T_3$, separately.

\textbf{Bounding $T_1$.}
\begin{align}
    T_1 &= \langle \nabla F(\mathbf{x}_t), \mathbb{E}\left[ -\frac{\eta}{a}\sum_{i \in \mathcal{A}} \sum_{j=\tau'}^{H - 1} \left( g_{t,j}^i - \nabla F_i(\mathbf{x}_t) \right) \right] \rangle \nonumber \\
    &= \langle \nabla F(\mathbf{x}_t), \mathbb{E}\left[ -\frac{\eta}{a}\sum_{i \in \mathcal{A}} \sum_{j=\tau'}^{H - 1} \left( \nabla F_i (\mathbf{x}_{t,j}^i) - \nabla F_i(\mathbf{x}_t) \right) \right] \rangle \label{eq:T1_unbiased} \\
    &= \langle \sqrt{\eta \tau} \nabla F(\mathbf{x}_t), \mathbb{E}\left[ -\frac{\sqrt{\eta}}{a \sqrt{\tau}}\sum_{i \in \mathcal{A}} \sum_{j=\tau'}^{H - 1} \left( \nabla F_i (\mathbf{x}_{t,j}^i) - \nabla F_i(\mathbf{x}_t) \right) \right] \rangle \nonumber \\
    &= \frac{\eta \tau}{2} \| \nabla F(\mathbf{x}_t) \|^2 + \frac{\eta}{2a^2\tau} \left\| \mathbb{E} \left[ \sum_{i \in \mathcal{A}} \sum_{j=\tau'}^{H - 1} \left( \nabla F_i (\mathbf{x}_{t,j}^i) - \nabla F_i(\mathbf{x}_t) \right) \right] \right\|^2 \nonumber \\
    &\qquad -\frac{\eta}{2a^2\tau} \left\| \mathbb{E}\left[ \sum_{i\in \mathcal{A}} \sum_{j=\tau'}^{H - 1} \nabla F_i (\mathbf{x}_{t,j}^{i}) \right] \right\|^2 \label{eq:t1_long} \\
    &\leq \frac{\eta \tau}{2} \| \nabla F(\mathbf{x}_t) \|^2 + \frac{\eta \mathcal{L}^2}{2a} \sum_{i \in \mathcal{A}} \sum_{j=\tau'}^{H - 1} \mathbb{E} \left[ \left\| \mathbf{x}_{t,j}^i - \mathbf{x}_t \right\|^2 \right] -\frac{\eta}{2a^2\tau} \left\| \mathbb{E}\left[ \sum_{i\in \mathcal{A}} \sum_{j=\tau'}^{H - 1} \nabla F_i (\mathbf{x}_{t,j}^{i}) \right] \right\|^2, \label{eq:t1_bound}
\end{align}
where (\ref{eq:T1_unbiased}) holds because $\mathbb{E}[g_{t,j}^i \mid \mathbf{x}_{t,j}^i] = \nabla F_i(\mathbf{x}_{t,j}^i)$ by Assumption 2, (\ref{eq:t1_long}) follows from $\langle \mathbf{u}, \mathbf{v} \rangle = \frac{1}{2}\|\mathbf{u}\|^2 + \frac{1}{2}\|\mathbf{v}\|^2 - \frac{1}{2}\|\mathbf{u} - \mathbf{v}\|^2$ applied to the two deterministic vectors and $\mathbb{E}[\sum_{i \in \mathcal{A}} \nabla F_i(\mathbf{x}_t)] = a \nabla F(\mathbf{x}_t)$, and (\ref{eq:t1_bound}) follows from Jensen's inequality, the Cauchy-Schwarz inequality, and Assumption 1.

\textbf{Bounding $T_2$.}

For each $i \in \mathcal{A}$ with its selected client $k$, we decompose
\begin{align}
    \mathbb{E}[\tilde{p}_{t,j}^i] - \nabla_z F_i(\mathbf{x}_t) = r_{t,j}^i + (1-\alpha) \left( \nabla_z F_k(\mathbf{x}_t) - \nabla_z F_i(\mathbf{x}_t) \right), \nonumber
\end{align}
where $r_{t,j}^i := \alpha \big( \nabla_z F_i(\mathbf{x}_{t,j}^i) - \nabla_z F_i(\mathbf{x}_t) \big) + (1-\alpha) \big( \nabla_z F_k(\mathbf{x}_{t,j}^i) - \nabla_z F_k(\mathbf{x}_t) \big)$ satisfies $\| r_{t,j}^i \| \leq \mathcal{L} \| \mathbf{z}_{t,j}^i - \mathbf{z}_t \|$ by Assumption 1.
Since each $i \in \mathcal{A}$ selects $k$ uniformly from $\mathcal{A} \setminus \{i\}$, every client in $\mathcal{A}$ is counted exactly $a - 1$ times, and thus
\begin{align}
    \mathbb{E}\left[ \sum_{i \in \mathcal{A}} \left( \nabla_z F_k(\mathbf{x}_t) - \nabla_z F_i(\mathbf{x}_t) \right) \right] = \mathbb{E}\left[ \sum_{i \in \mathcal{A}} \frac{1}{a-1} \sum_{k \in \mathcal{A} \setminus \{i\}} \nabla_z F_k(\mathbf{x}_t) - \sum_{i \in \mathcal{A}} \nabla_z F_i(\mathbf{x}_t) \right] = 0. \label{eq:cancel}
\end{align}
Therefore,
\begin{align}
    T_2 &= \langle \nabla_z F(\mathbf{x}_t), \mathbb{E}\left[ -\frac{\eta}{a}\sum_{i \in \mathcal{A}} \sum_{j=0}^{\tau' - 1} \left( \mathbb{E}[\tilde{p}_{t,j}^i] - \nabla_z F_i(\mathbf{x}_t) \right) \right] \rangle \nonumber \\
    &= \langle \nabla_z F(\mathbf{x}_t), \mathbb{E}\left[ -\frac{\eta}{a}\sum_{i \in \mathcal{A}} \sum_{j=0}^{\tau' - 1} r_{t,j}^i \right] \rangle \label{eq:T2_cancel} \\
    &\leq \frac{\eta \tau'}{2} \left\| \nabla_z F(\mathbf{x}_t) \right\|^2 + \frac{\eta}{2\tau' a^2} \mathbb{E}\left[ \left\| \sum_{i \in \mathcal{A}} \sum_{j=0}^{\tau' - 1} r_{t,j}^i \right\|^2 \right] \label{eq:young} \\
    &\leq \frac{\eta \tau'}{2} \left\| \nabla_z F(\mathbf{x}_t) \right\|^2 + \frac{\eta \mathcal{L}^2}{2a} \sum_{i \in \mathcal{A}} \sum_{j=0}^{\tau' - 1} \mathbb{E}\left[ \left\| \mathbf{z}_{t,j}^i - \mathbf{z}_t \right\|^2 \right], \label{eq:t2_bound}
\end{align}
where the first equality uses the tower property and the fact that $\tilde{p}_{t,j}^i$ has nonzero entries only in the coordinates of $\mathbf{z}$, (\ref{eq:T2_cancel}) follows from (\ref{eq:cancel}), (\ref{eq:young}) follows from Young's inequality and Jensen's inequality, and (\ref{eq:t2_bound}) follows from the Cauchy-Schwarz inequality and $\| r_{t,j}^i \| \leq \mathcal{L} \| \mathbf{z}_{t,j}^i - \mathbf{z}_t \|$.

\textbf{Bounding $T_3$.}
\begin{align}
    T_3 &= \mathbb{E}[\| \eta \tilde{\Delta}_t \|^2]
    = \frac{\eta^2}{a^2} \mathbb{E}\left[ \left\| \sum_{i \in \mathcal{A}} \left( \sum_{j=0}^{\tau' - 1} \tilde{p}_{t,j}^i + \sum_{j=\tau'}^{H - 1} g_{t,j}^i \right) \right\|^2 \right] \nonumber \\
    &\leq \frac{4\eta^2}{a^2} \mathbb{E} \left[\left \| \sum_{i\in\mathcal{A}} \sum_{j=\tau'}^{H-1} \left( g_{t,j}^i - \nabla F_i(\mathbf{x}_{t,j}^{i}) \right) \right\|^2 \right] + \frac{4\eta^2}{a^2} \underbrace{\mathbb{E} \left[\left \| \sum_{i\in\mathcal{A}} \sum_{j=\tau'}^{H-1} \nabla F_i(\mathbf{x}_{t,j}^{i}) \right\|^2 \right]}_{T_4} \nonumber \\
    &\qquad + \frac{4\eta^2}{a^2} \underbrace{\mathbb{E} \left[\left \| \sum_{i\in\mathcal{A}} \sum_{j=0}^{\tau'-1} \left( \tilde{p}_{t,j}^i - \mathbb{E}[\tilde{p}_{t,j}^i] \right) \right\|^2 \right]}_{T_5} + \frac{4\eta^2}{a^2} \mathbb{E} \left[\left \| \sum_{i\in\mathcal{A}} \sum_{j=0}^{\tau'-1} \mathbb{E}[\tilde{p}_{t,j}^i] \right\|^2 \right] \label{eq:t3_frame}
\end{align}
where the first inequality follows from $\| \sum_{l=1}^4 \mathbf{u}_l \|^2 \leq 4 \sum_{l=1}^4 \| \mathbf{u}_l \|^2$.

\textbf{Bounding $T_4$}.
\begin{align}
    T_4 &= \mathbb{E} \left[\left \| \sum_{i\in\mathcal{A}} \sum_{j=\tau'}^{H-1} \nabla F_i(\mathbf{x}_{t,j}^{i}) \right\|^2 \right] \nonumber \\
    &= \left\| \mathbb{E}\left[ \sum_{i\in\mathcal{A}} \sum_{j=\tau'}^{H-1} \nabla F_i(\mathbf{x}_{t,j}^i) \right] \right\|^2 + \mathbb{E}\left[ \left \| \sum_{i\in \mathcal{A}} \sum_{j=\tau'}^{H-1} \nabla F_i(\mathbf{x}_{t,j}^i) - \mathbb{E}\left[ \sum_{i\in \mathcal{A}} \sum_{j=\tau'}^{H-1} \nabla F_i(\mathbf{x}_{t,j}^i) \right] \right\|^2 \right] \nonumber
\end{align}
Note that $\mathbb{E}[\| X - \mathbb{E}[X] \|^2] = \mathbb{E}[\| X - c \|^2] - \| \mathbb{E}[X] - c \|^2 \leq \mathbb{E}[\| X - c \|^2] $ for any constant $c$.
By setting $X = \sum_{i\in\mathcal{A}} \sum_{j=\tau'}^{H-1} \nabla F_i(\mathbf{x}_{t,j}^i)$ and $c = \tau \sum_{i\in \mathcal{A}} \nabla F_i(\mathbf{x}_t)$, $T_4$ can be further bounded as follows:
\begin{align}
    T_4 &\leq \left\| \mathbb{E}\left[ \sum_{i\in\mathcal{A}} \sum_{j=\tau'}^{H-1} \nabla F_i(\mathbf{x}_{t,j}^i) \right] \right\|^2 + \mathbb{E}\left[ \left\| \sum_{i\in\mathcal{A}} \sum_{j=\tau'}^{H-1} \nabla F_i(\mathbf{x}_{t,j}^i) - \tau \sum_{i\in \mathcal{A}} \nabla F_i(\mathbf{x}_t) \right\|^2 \right] \nonumber \\
    &= \left\| \mathbb{E}\left[ \sum_{i\in\mathcal{A}} \sum_{j=\tau'}^{H-1} \nabla F_i(\mathbf{x}_{t,j}^i) \right] \right\|^2 + \mathbb{E}\left[ \left\| \sum_{i\in\mathcal{A}} \sum_{j=\tau'}^{H-1} \left( \nabla F_i(\mathbf{x}_{t,j}^i) - \nabla F_i(\mathbf{x}_t) \right) \right\|^2 \right] \nonumber \\
    &\leq \left\| \mathbb{E}\left[ \sum_{i\in\mathcal{A}} \sum_{j=\tau'}^{H-1} \nabla F_i(\mathbf{x}_{t,j}^i) \right] \right\|^2 + a\tau \sum_{i\in\mathcal{A}} \sum_{j=\tau'}^{H-1} \mathbb{E}\left[ \left\| \left( \nabla F_i(\mathbf{x}_{t,j}^i) - \nabla F_i(\mathbf{x}_t) \right) \right\|^2 \right] \nonumber \\
    &\leq \left\| \mathbb{E}\left[ \sum_{i\in\mathcal{A}} \sum_{j=\tau'}^{H-1} \nabla F_i(\mathbf{x}_{t,j}^i) \right] \right\|^2 + a\tau \mathcal{L}^2 \sum_{i\in\mathcal{A}} \sum_{j=\tau'}^{H-1} \mathbb{E}\left[ \left\| \mathbf{x}_{t,j}^i - \mathbf{x}_t \right\|^2 \right], \label{eq:t4}
\end{align}
where (\ref{eq:t4}) holds based on Assumption 1.

\textbf{Bounding $T_5$}.

\begin{align}
    T_5 &= \mathbb{E} \left[\left \| \sum_{i\in\mathcal{A}} \sum_{j=0}^{\tau'-1} \left( \tilde{p}_{t,j}^i - \mathbb{E}[\tilde{p}_{t,j}^i] \right) \right\|^2 \right] \nonumber \\
    &= \sum_{i\in\mathcal{A}} \sum_{j=0}^{\tau'-1} \mathbb{E} \left[\left \| \tilde{p}_{t,j}^i - \mathbb{E}[\tilde{p}_{t,j}^i] \right\|^2 \right] + \sum_{(i,j) \neq (i',j')} \langle \tilde{p}_{t,j}^i - \mathbb{E}[\tilde{p}_{t,j}^i], \tilde{p}_{t,j'}^{i'} - \mathbb{E}[\tilde{p}_{t,j'}^{i'}] \rangle \label{eq:zero} \\
    &= \sum_{i\in\mathcal{A}} \sum_{j=0}^{\tau'-1} \mathbb{E} \left[\left \| \tilde{p}_{t,j}^i - \mathbb{E}[\tilde{p}_{t,j}^i] \right\|^2 \right]. \nonumber
\end{align}
For the cross terms in (\ref{eq:zero}), when $i \neq i'$, the cross terms are zeroed out since the clients are independently sampled and thus $\mathbb{E}[\tilde{p}_{t,j}^i - \mathbb{E}[\tilde{p}_{t,j}^i]] = 0$.
When $i=i'$ and $j<j'$, the cross terms also go to zero since the later noise has zero mean conditioned on the past.

Then, by adding and subtracting full batch gradient terms and expanding the terms, we have:
\begin{align}
    T_5 &= \sum_{i\in\mathcal{A}} \sum_{j=0}^{\tau'-1} \mathbb{E} \left[\left \| \tilde{p}_{t,j}^i - \mathbb{E}[\tilde{p}_{t,j}^i] \right\|^2 \right] \nonumber \\
    &=\sum_{i\in\mathcal{A}} \sum_{j=0}^{\tau'-1} \left( \underbrace{\alpha \mathbb{E} \left[\left \| \tilde{p}_{t,j}^i - \mathbb{E}[\tilde{p}_{t,j}^i] \right\|^2  | \text{ client } i \right]}_{A} + \underbrace{(1 - \alpha) \mathbb{E} \left[\left \| \tilde{p}_{t,j}^i - \mathbb{E}[\tilde{p}_{t,j}^i] \right\|^2  | \text{ client } k \right]}_{B} \right). \label{eq:ab}
\end{align}

\begin{align}
    A &= \alpha \left( \mathbb{E}\left[ \left\| \tilde{p}_{t,j}^i - \nabla_z F_i(\mathbf{x}_{t,j}^i) + \nabla_z F_i(\mathbf{x}_{t,j}^i) - \mathbb{E}\left[ \tilde{p}_{t,j}^i \right] \right\|^2 | \text{ client } i \right] \right) \nonumber \\
    &=\alpha \mathbb{E} \left[ \left\| \tilde{p}_{t,j}^i - \nabla_z F_i(\mathbf{x}_{t,j}^i) \right\|^2 | \text{ client } i \right] + \alpha \mathbb{E} \left[ \left\| \nabla_z F_i(\mathbf{x}_{t,j}^i) - \mathbb{E}\left[ \tilde{p}_{t,j}^i \right] \right\|^2  | \text{ client } i \right] \nonumber \\
    &\qquad + 2\alpha \mathbb{E}\left[ \langle (\tilde{p}_{t,j}^i - \nabla_z F_i(\mathbf{x}_{t,j}^i)), (\nabla_z F_i(\mathbf{x}_{t,j}^i) - \mathbb{E}\left[ \tilde{p}_{t,j}^i \right]) \rangle  | \text{ client } i \right] \nonumber \\
    &=\alpha \mathbb{E} \left[ \left\| \tilde{p}_{t,j}^i - \nabla_z F_i(\mathbf{x}_{t,j}^i) \right\|^2  | \text{ client } i \right] + \alpha \mathbb{E} \left[ \left\| \nabla_z F_i(\mathbf{x}_{t,j}^i) - \mathbb{E}\left[ \tilde{p}_{t,j}^i \right] \right\|^2  | \text{ client } i \right] \label{eq:Azo} \\
    &\leq \alpha \sigma_z^2 + \alpha \mathbb{E}\left[ \left\| \nabla_z F_i(\mathbf{x}_{t,j}^i) - \left( \alpha \nabla_z F_i(\mathbf{x}_{t,j}^i) + (1-\alpha) \nabla_z F_k(\mathbf{x}_{t,j}^i) \right) \right\|^2 \right] \nonumber \\
    &= \alpha \sigma_z^2 + \alpha (1 - \alpha)^2 \mathbb{E}\left[ \left\| \nabla_z F_i(\mathbf{x}_{t,j}^i) - \nabla_z F_k(\mathbf{x}_{t,j}^i) \right\|^2 \right] \nonumber\\
    &\leq \alpha \sigma_z^2 + 4\alpha(1 - \alpha)^2 \sigma_G^2, \label{eq:A}
\end{align}
where (\ref{eq:Azo}) is because $\mathbb{E} [\tilde{p}_{t,j}^i | \text{ client } i]$ is $\nabla_z F_i(\mathbf{x}_{t,j}^i)$, and thus the cross term is zeroed out.
\begin{align}
    B &= (1 -\alpha) \left( \mathbb{E}\left[ \left\| \tilde{p}_{t,j}^i - \nabla_z F_k(\mathbf{x}_{t,j}^i) + \nabla_z F_k(\mathbf{x}_{t,j}^i) - \mathbb{E}\left[ \tilde{p}_{t,j}^i \right] \right\|^2 | \text{ client } k \right] \right) \nonumber \\
    &= (1-\alpha) \mathbb{E} \left[ \left\| \tilde{p}_{t,j}^i - \nabla_z F_k(\mathbf{x}_{t,j}^i) \right\|^2 | \text{ client } k \right] + (1-\alpha) \mathbb{E} \left[ \left\| \nabla_z F_k(\mathbf{x}_{t,j}^i) - \mathbb{E}\left[ \tilde{p}_{t,j}^i \right] \right\|^2  | \text{ client } k \right] \nonumber \\
    &\qquad + 2(1-\alpha) \mathbb{E}\left[ \langle (\tilde{p}_{t,j}^i - \nabla_z F_k(\mathbf{x}_{t,j}^i)), (\nabla_z F_k(\mathbf{x}_{t,j}^i) - \mathbb{E}\left[ \tilde{p}_{t,j}^i \right]) \rangle  | \text{ client } k \right] \nonumber \\
    &=(1-\alpha) \mathbb{E} \left[ \left\| \tilde{p}_{t,j}^i - \nabla_z F_k(\mathbf{x}_{t,j}^i) \right\|^2  | \text{ client } k \right] + (1-\alpha) \mathbb{E} \left[ \left\| \nabla_z F_k(\mathbf{x}_{t,j}^i) - \mathbb{E}\left[ \tilde{p}_{t,j}^i \right] \right\|^2  | \text{ client } k \right] \label{eq:Bzo} \\
    &\leq (1-\alpha) \sigma_z^2 + (1-\alpha) \mathbb{E}\left[ \left\| \nabla_z F_k(\mathbf{x}_{t,j}^i) - \left( \alpha \nabla_z F_i(\mathbf{x}_{t,j}^i) + (1-\alpha) \nabla_z F_k(\mathbf{x}_{t,j}^i) \right) \right\|^2 \right] \nonumber \\
    &= (1-\alpha) \sigma_z^2 + \alpha^2 (1 - \alpha) \mathbb{E}\left[ \left\| \nabla_z F_i(\mathbf{x}_{t,j}^i) - \nabla_z F_k(\mathbf{x}_{t,j}^i) \right\|^2 \right] \nonumber\\
    &\leq (1-\alpha) \sigma_z^2 + 4\alpha^2(1 - \alpha) \sigma_G^2, \label{eq:B}
\end{align}
where (\ref{eq:Bzo}) holds because $\mathbb{E} [\tilde{p}_{t,j}^i | \text{ client } k]$ is $\nabla_z F_k(\mathbf{x}_{t,j}^i)$, and thus the cross term is zeroed out.

By plugging in (\ref{eq:A}) and (\ref{eq:B}) into (\ref{eq:ab}), we have:
\begin{align}
    T_5 &\leq \sum_{i\in\mathcal{A}} \sum_{j=0}^{\tau'-1} \left( \sigma_z^2 + 4\alpha(1-\alpha) \sigma_G^2 \right). \label{eq:t5}
\end{align}

Now, putting (\ref{eq:t4}) and (\ref{eq:t5}) into (\ref{eq:t3_frame}), we bound $T_3$ as follows:
\begin{align}
    T_3 & \leq \frac{4\eta^2}{a^2} \mathbb{E} \left[\left \| \sum_{i\in\mathcal{A}} \sum_{j=\tau'}^{H-1} \left( g_{t,j}^i - \nabla F_i(\mathbf{x}_{t,j}^{i}) \right) \right\|^2 \right] \nonumber \\
    &\qquad + \frac{4\eta^2}{a^2} \left( \left\| \mathbb{E}\left[ \sum_{i\in\mathcal{A}} \sum_{j=\tau'}^{H-1} \nabla F_i(\mathbf{x}_{t,j}^i) \right] \right\|^2 + a\tau \mathcal{L}^2 \sum_{i\in\mathcal{A}} \sum_{j=\tau'}^{H-1} \mathbb{E}\left[ \left\| \mathbf{x}_{t,j}^i - \mathbf{x}_t \right\|^2 \right] \right) \nonumber \\
    &\qquad + \frac{4\eta^2}{a^2} \left( \sum_{i\in\mathcal{A}} \sum_{j=0}^{\tau'-1} \left( \sigma_z^2 + 4\alpha(1-\alpha) \sigma_G^2 \right) \right) + \frac{4\eta^2}{a^2} \mathbb{E} \left[\left \| \sum_{i\in\mathcal{A}} \sum_{j=0}^{\tau'-1} \mathbb{E}[\tilde{p}_{t,j}^i] \right\|^2 \right]. \nonumber
\end{align}
The first term on the right-hand side can be simplified as follows:
\begin{align}
    \frac{4\eta^2}{a^2} \mathbb{E}\left[\left\| \sum_{i\in \mathcal{A}} \sum_{j=\tau'}^{H-1} \left( g_{t,j}^i - \nabla F_i (\mathbf{x}_{t,j}^i) \right) \right\|^2\right] &\leq \frac{4\eta^2}{a^2} \sum_{i\in \mathcal{A}} \sum_{j=\tau'}^{H-1} \mathbb{E} \left[ \left\| g_{t,j}^i - \nabla F_i(\mathbf{x}_{t,j}^i) \right\|^2 \right] \label{eq:var} \\
    &\leq \frac{4\eta^2 \tau}{a} \sigma_x^2. \nonumber
\end{align}
In (\ref{eq:var}), conditioned on $\{\mathbf{x}_{t,\tau'}^i\}_{i\in\mathcal{A}}$, the noise sequences $\{g_{t,j}^i - \nabla F_i(\mathbf{x}_{t,j}^i)\}_{j\ge\tau'}$ of different clients are independent and each has zero mean, so the cross terms with $i \neq i' \in \mathcal{A}$ vanish.
Similarly, for the same client $i$ and $\tau' \leq j < j'$, since the mini-batch at step $j'$ is sampled independently of the history up to $\mathbf{x}_{t,j'}^i$, Assumption 2 gives $\mathbb{E}[g_{t,j'}^i \mid \text{history up to } \mathbf{x}_{t,j'}^i] = \nabla F_i(\mathbf{x}_{t,j'}^i)$, while $g_{t,j}^i - \nabla F_i(\mathbf{x}_{t,j}^i)$ is fixed under this conditioning. Thus, the cross terms with $j \neq j'$ also vanish.
Therefore, the cross terms are all zeroed out, and (\ref{eq:var}) holds.

Now, $T_3$ can be further bounded by the above condition and Lemma~\ref{lemma:local} as follows:
\begin{align}
    T_3 &\leq \frac{4\eta^2 \tau}{a} \sigma_x^2 \nonumber \\
    &\qquad + \frac{4\eta^2}{a^2} \left( \left\| \mathbb{E}\left[ \sum_{i\in\mathcal{A}} \sum_{j=\tau'}^{H-1} \nabla F_i(\mathbf{x}_{t,j}^i) \right] \right\|^2 + a\tau \mathcal{L}^2 \sum_{i\in\mathcal{A}} \sum_{j=\tau'}^{H-1} \mathbb{E}\left[ \left\| \mathbf{x}_{t,j}^i - \mathbf{x}_t \right\|^2 \right] \right) \nonumber \\
    &\qquad + \frac{4\eta^2}{a^2} \left( \sum_{i\in\mathcal{A}} \sum_{j=0}^{\tau'-1} \left( \sigma_z^2 + 4\alpha(1-\alpha) \sigma_G^2 \right) \right) \nonumber \\
    &\qquad + \frac{4\eta^2}{a^2} \left( 6a\tau' \mathcal{L}^2 \mathbb{E}\left[ \sum_{i\in\mathcal{A}} \sum_{j=0}^{\tau' - 1} \left\| \mathbf{z}_{t,j}^i - \mathbf{z}_t \right\|^2 \right] + 3a(\tau')^2 \left( \frac{m-a}{m-1} + (1-\alpha)^2 \right) \sigma_G^2 + 3a^2(\tau')^2 \mathbb{E}\left[ \left\| \nabla_z F(\mathbf{x}_t) \right\|^2 \right] \right). \nonumber
\end{align}
After rearranging the terms, we have:
\begin{align}
    T_3 &\leq \frac{4\eta^2}{a}\left(\tau\sigma_x^2 + \tau'\sigma_z^2 + \tau'\left(4\alpha(1-\alpha) + 3\tau'\left(\frac{m-a}{m-1} + (1-\alpha)^2\right)\right)\sigma_G^2\right) \nonumber \\
    &\qquad + \frac{4\eta^2}{a^2} \left\| \mathbb{E}\left[ \sum_{i\in\mathcal{A}} \sum_{j=\tau'}^{H-1} \nabla F_i(\mathbf{x}_{t,j}^i) \right] \right\|^2 + 12\eta^2 (\tau')^2 \mathbb{E}\left\| \nabla_z F(\mathbf{x}_t) \right\|^2 \nonumber \\
    &\qquad + \frac{4\eta^2\mathcal{L}^2}{a} \left( \tau \sum_{i \in \mathcal{A}} \sum_{j=\tau'}^{H-1} \mathbb{E}\left\| \mathbf{x}_{t,j}^i - \mathbf{x}_t \right\|^2 + 6\tau' \sum_{i \in \mathcal{A}} \sum_{j=0}^{\tau' - 1} \mathbb{E}\left\| \mathbf{z}_{t,j}^i - \mathbf{z}_t \right\|^2 \right). \label{eq:t3}
\end{align}

Putting $T_1$, $T_2$, and $T_3$ bounds into (\ref{eq:frame2}), we have:
\begin{align}
\mathbb{E}[F(\mathbf{x}_{t+1})] &\leq F(\mathbf{x}_t) - \eta \tau \left\| \nabla F(\mathbf{x}_t) \right\|^2 - \eta \tau' \left\| \nabla_z F(\mathbf{x}_t) \right\|^2 \nonumber \\
& +\frac{\eta \tau}{2} \left\| \nabla F(\mathbf{x}_t) \right\|^2 + \frac{\eta \mathcal{L}^2}{2a} \sum_{i\in\mathcal{A}}\sum_{j=\tau'}^{H-1} \mathbb{E}\left[\left\| \mathbf{x}_{t,j}^i - \mathbf{x}_t \right\|^2\right] - \frac{\eta}{2a^2\tau} \left\| \mathbb{E}\left[ \sum_{i\in\mathcal{A}}\sum_{j=\tau'}^{H-1} \nabla F_i(\mathbf{x}_{t,j}^i) \right] \right\|^2 \nonumber \\
& + \frac{\eta \tau'}{2} \left\| \nabla_z F(\mathbf{x}_t) \right\|^2 + \frac{\eta\mathcal{L}^2}{2a} \sum_{i\in\mathcal{A}}\sum_{j=0}^{\tau' - 1} \mathbb{E}\left[ \left\| \mathbf{z}_{t,j}^i - \mathbf{z}_t \right\|^2 \right] \nonumber \\
& + \frac{2\eta^2\mathcal{L}}{a} \left( \tau \sigma_x^2 + \tau' \sigma_z^2 + \tau' \left( 4\alpha(1-\alpha) + 3\tau' \left( \frac{m-a}{m-1} + (1-\alpha)^2 \right) \right) \sigma_G^2 \right) \nonumber\\
& + \frac{2\eta^2\mathcal{L}}{a^2} \left\| \mathbb{E}\left[ \sum_{i\in\mathcal{A}}\sum_{j=\tau'}^{H-1} \nabla F_i(\mathbf{x}_{t,j}^i) \right] \right\|^2 + 6\eta^2 \mathcal{L}(\tau')^2 \left\| \nabla_z F(\mathbf{x}_t) \right\|^2 \nonumber \\
& + \frac{2\eta^2 \mathcal{L}^3}{a} \left( \tau \sum_{i\in\mathcal{A}}\sum_{j=\tau'}^{H - 1} \mathbb{E} \left[ \left\| \mathbf{x}_{t,j}^i - \mathbf{x}_t \right\|^2 \right] + 6\tau' \sum_{i\in\mathcal{A}}\sum_{j=0}^{\tau' - 1} \mathbb{E}\left[ \left\| \mathbf{z}_{t,j}^i - \mathbf{z}_t \right\|^2 \right] \right). \nonumber
\end{align}

Then, by rearranging the terms:
\begin{align}
    \mathbb{E}[F(\mathbf{x}_{t+1})] &\leq F(\mathbf{x}_t) - \frac{\eta\tau}{2} \left\| \nabla F(\mathbf{x}_t) \right\|^2 - \frac{\eta \tau'}{2} \left(1 - 12\eta \mathcal{L}\tau' \right) \left\| \nabla_z F(\mathbf{x}_t) \right\|^2 \nonumber \\
    &\qquad -\frac{\eta}{2a^2\tau} \left(1 - 4\eta \mathcal{L}\tau \right) \left\| \mathbb{E}\left[ \sum_{i\in\mathcal{A}}\sum_{j=\tau'}^{H-1} \nabla F_i(\mathbf{x}_{t,j}^i) \right] \right\|^2 \nonumber \\
    &\qquad + \frac{\eta\mathcal{L}^2}{2a} \left( \left(1 + 4\eta\mathcal{L}\tau \right) \sum_{i\in\mathcal{A}}\sum_{j=\tau'}^{H-1} \mathbb{E} \left[ \left\| \mathbf{x}_{t,j}^i - \mathbf{x}_t \right\|^2 \right] + \left(1 + 24\eta\mathcal{L}\tau' \right) \sum_{i\in\mathcal{A}}\sum_{j=0}^{\tau' - 1} \mathbb{E} \left[ \left\| \mathbf{z}_{t,j}^i - \mathbf{z}_t \right\|^2 \right] \right) \nonumber \\
    &\qquad + \frac{2\eta^2\mathcal{L}}{a} \left( \tau \sigma_x^2 + \tau' \sigma_z^2 + \tau' \left( 4\alpha(1-\alpha) + 3\tau' \left( \frac{m-a}{m-1} + (1-\alpha)^2 \right) \right) \sigma_G^2 \right). \nonumber
\end{align}

Finally, if $\eta \leq \min \{\frac{1}{4\mathcal{L}\tau}, \frac{1}{24\mathcal{L}\tau'} \}$, we have:
\begin{align}
    \mathbb{E}[F(\mathbf{x}_{t+1})] &\leq F(\mathbf{x}_t) - \frac{\eta\tau}{2} \left\| \nabla F(\mathbf{x}_t) \right\|^2 - \frac{\eta\tau'}{4} \left\| \nabla_z F(\mathbf{x}_t) \right\|^2 \nonumber \\
    &\qquad + \frac{\eta\mathcal{L}^2}{a} \left( \sum_{i\in\mathcal{A}}\sum_{j=\tau'}^{H-1} \mathbb{E} \left[ \left\| \mathbf{x}_{t,j}^i - \mathbf{x}_t \right\|^2 \right] + \sum_{i\in\mathcal{A}}\sum_{j=0}^{\tau'-1} \mathbb{E} \left[ \left\| \mathbf{z}_{t,j}^i - \mathbf{z}_t \right\|^2 \right]\right) \nonumber \\
    &\qquad + \frac{2\eta^2\mathcal{L}}{a} \left( \tau \sigma_x^2 + \tau' \sigma_z^2 + \tau' \left( 4\alpha(1-\alpha) + 3\tau' \left( \frac{m-a}{m-1} + (1-\alpha)^2 \right) \right) \sigma_G^2 \right). \nonumber
\end{align}
\end{proof}

\clearpage

\begin{lemma}[Partial model discrepancy] \label{lemma:partial}
Suppose that $\mathcal{A}$ is sampled uniformly without replacement and each $i \in \mathcal{A}$ selects $k$ uniformly from $\mathcal{A} \setminus \{i\}$. For any step size satisfying $\eta \leq \frac{1}{3\sqrt{2}\mathcal{L}\tau'}$, we have, for any $j = 0, \dots, \tau'$,
\begin{align}
    \frac{1}{a} \sum_{i \in \mathcal{A}} \mathbb{E}[\| \mathbf{z}_{t,j}^{i} - \mathbf{z}_t \|^2] \leq 2\eta^2 \tau' \sigma_z^2 + 2\eta^2 \tau' \left( 4\alpha(1-\alpha) + 9\tau' \gamma \right) \sigma_G^2 + 18\eta^2 (\tau')^2 \| \nabla_z F(\mathbf{x}_t) \|^2, \nonumber
\end{align}
where $\gamma = \alpha^2 + (1-\alpha)^2$.
\end{lemma}

\begin{proof}

For any client $i \in [m]$, let $k$ denote the client selected by $i$ at round $t$.
Since only $\mathbf{z}$ is updated during the tail steps, we have $\mathbf{x}^i_{t,j} = (\mathbf{y}_t, \mathbf{z}^i_{t,j})$ with $\mathbf{z}^i_{t,0} = \mathbf{z}_t$.
Throughout this proof, the inner expectation $\mathbb{E}[\tilde{p}^i_{t,j}]$ is taken over the mini-batch sampling from the merged pool of local and remote activations, conditioned on $\mathbf{z}^i_{t,j}$ and $k$.
The outer expectation $\mathbb{E}[\cdot]$ is conditioned on $\mathbf{x}_t$, so $\nabla F(\mathbf{x}_t)$ and $\nabla_z F(\mathbf{x}_t)$ are deterministic.
Then, for $j = 0, \dots, \tau'-1$, we have
\begin{align}
    \mathbb{E}[\| \mathbf{z}_{t,j+1}^{i} - \mathbf{z}_t \|^2] &= \mathbb{E}[ \| \mathbf{z}_{t,j}^{i} - \mathbf{z}_t -\eta \tilde{p}_{t,j}^{i} \|^2 ] \nonumber \\
    &= \mathbb{E}[\| \mathbf{z}_{t,j}^{i} - \mathbf{z}_t -\eta ( \tilde{p}_{t,j}^{i} - \mathbb{E}[\tilde{p}_{t,j}^i] ) - \eta \mathbb{E}[\tilde{p}_{t,j}^i] \|^2] \nonumber \\
    &= \eta^2 \mathbb{E}[\| \tilde{p}^i_{t,j} - \mathbb{E}[\tilde{p}_{t,j}^i] \|^2] + \mathbb{E}[\| \mathbf{z}_{t,j}^{i} - \mathbf{z}_t -\eta \mathbb{E}[\tilde{p}_{t,j}^i] \|^2] \nonumber \\
    &\qquad - 2\eta\, \mathbb{E}[\langle \tilde{p}^i_{t,j} - \mathbb{E}[\tilde{p}_{t,j}^i], \mathbf{z}_{t,j}^{i} - \mathbf{z}_t -\eta \mathbb{E}[\tilde{p}_{t,j}^i] \rangle] \nonumber \\
    &= \eta^2 \mathbb{E}[\| \tilde{p}^i_{t,j} - \mathbb{E}[\tilde{p}_{t,j}^i] \|^2] + \mathbb{E}[\| \mathbf{z}_{t,j}^{i} - \mathbf{z}_t -\eta \mathbb{E}[\tilde{p}_{t,j}^i] \|^2], \label{eq:zeroout2}
\end{align}
where (\ref{eq:zeroout2}) holds because $\mathbf{z}^i_{t,j} - \mathbf{z}_t - \eta \mathbb{E}[\tilde{p}^i_{t,j}]$ is fixed given $\mathbf{z}^i_{t,j}$ and $k$, and $\tilde{p}^i_{t,j} - \mathbb{E}[\tilde{p}^i_{t,j}]$ has zero mean over the mini-batch sampling.

We then further bound the right-hand side of (\ref{eq:zeroout2}) as follows:
\begin{align}
    \mathbb{E}[\| \mathbf{z}_{t,j+1}^{i} - \mathbf{z}_t \|^2]
    &\leq \eta^2 \left( \sigma_z^2 + 4\alpha(1-\alpha) \sigma_G^2 \right) + \mathbb{E}[\| \mathbf{z}_{t,j}^{i} - \mathbf{z}_t -\eta \mathbb{E}[\tilde{p}_{t,j}^i] \|^2] \label{eq:var_mixed} \\
    &\leq \eta^2 \left( \sigma_z^2 + 4\alpha(1-\alpha) \sigma_G^2 \right) + \left(1 + \frac{1}{2\tau'}\right) \mathbb{E}[\| \mathbf{z}_{t,j}^{i} - \mathbf{z}_t \|^2] + \underbrace{ (1 + 2\tau')\eta^2 \mathbb{E}\left[ \left\| \mathbb{E}[\tilde{p}_{t,j}^i] \right\|^2 \right] }_{T_6}, \label{eq:split2}
\end{align}
where (\ref{eq:var_mixed}) follows from Lemma 5, whose condition is satisfied by Assumption 4, and (\ref{eq:split2}) follows from $\| \mathbf{u} + \mathbf{v} \|^2 \leq (1 + \frac{1}{c}) \| \mathbf{u} \|^2 + (1 + c) \| \mathbf{v} \|^2$ with $c = 2\tau'$.

The $T_6$ can be bounded as follows:
\begin{align}
    T_6 &\leq 3\tau' \eta^2 \mathbb{E}\left[ \left\| \alpha \nabla_z F_i(\mathbf{x}_{t,j}^i) + (1-\alpha)\nabla_z F_k(\mathbf{x}_{t,j}^i) \right\|^2 \right] \nonumber \\
    &\leq 9\tau' \eta^2 \mathcal{L}^2 \mathbb{E}[\| \mathbf{z}_{t,j}^i - \mathbf{z}_t \|^2] + 9\tau' \eta^2 \mathbb{E}[ h_t^i ] + 9\tau' \eta^2 \| \nabla_z F(\mathbf{x}_t) \|^2, \label{eq:T6}
\end{align}
where $h_t^i$ is defined in Lemma~\ref{lemma:mixed_het}.
The first inequality uses $1 + 2\tau' \leq 3\tau'$.
(\ref{eq:T6}) is obtained by adding and subtracting $\alpha \nabla_z F_i(\mathbf{x}_t) + (1-\alpha) \nabla_z F_k(\mathbf{x}_t)$ and $\nabla_z F(\mathbf{x}_t)$, and then applying $\|\mathbf{u}+\mathbf{v}+\mathbf{w}\|^2 \leq 3(\|\mathbf{u}\|^2 + \|\mathbf{v}\|^2 + \|\mathbf{w}\|^2)$ and Assumption 1 with $\mathbf{x}^i_{t,j} - \mathbf{x}_t = (\mathbf{0}, \mathbf{z}^i_{t,j} - \mathbf{z}_t)$.

By plugging in (\ref{eq:T6}) into (\ref{eq:split2}), we have
\begin{align}
    \mathbb{E}[\| \mathbf{z}_{t,j+1}^{i} - \mathbf{z}_t \|^2] &\leq \left( 1 + \frac{1}{2\tau'} + 9\tau'\eta^2 \mathcal{L}^2 \right) \mathbb{E}[\| \mathbf{z}_{t,j}^{i} - \mathbf{z}_t \|^2] + \Psi_t^i, \nonumber
\end{align}
where $\Psi_t^i := \eta^2 \sigma_z^2 + 4\eta^2 \alpha (1-\alpha) \sigma_G^2 + 9\tau' \eta^2 h_t^i + 9\tau' \eta^2 \| \nabla_z F(\mathbf{x}_t) \|^2$.
If $\eta \leq \frac{1}{3\sqrt{2}\mathcal{L}\tau'}$, we have $9\tau'\eta^2\mathcal{L}^2 \leq \frac{1}{2\tau'}$, and thus
\begin{align}
    \mathbb{E}[\| \mathbf{z}_{t,j+1}^{i} - \mathbf{z}_t \|^2] &\leq \left( 1 + \frac{1}{\tau'} \right) \mathbb{E}[\| \mathbf{z}_{t,j}^{i} - \mathbf{z}_t \|^2] + \Psi_t^i. \nonumber
\end{align}

Next, by unrolling the recursion from $j$ down to $0$ with $\mathbf{z}^i_{t,0} = \mathbf{z}_t$, we have
\begin{align}
    \mathbb{E}[\| \mathbf{z}_{t,j+1}^{i} - \mathbf{z}_t \|^2]
    &\leq \sum_{l=0}^{j} \left(1 + \frac{1}{\tau'} \right)^l \Psi_t^i
    = \tau' \left( \left(1 + \frac{1}{\tau'} \right)^{j+1} - 1\right) \Psi_t^i \nonumber \\
    &\leq \tau' \left( \left(1 + \frac{1}{\tau'} \right)^{\tau'} - 1\right) \Psi_t^i
    \leq (e - 1)\tau' \Psi_t^i \leq 2\tau' \Psi_t^i \nonumber \\
    &= 2\eta^2 \tau' \sigma_z^2 + 8\eta^2 \tau' \alpha (1-\alpha) \sigma_G^2 + 18\eta^2 (\tau')^2 h_t^i + 18\eta^2 (\tau')^2 \| \nabla_z F(\mathbf{x}_t) \|^2, \label{eq:e_bound}
\end{align}
where we use $j + 1 \leq \tau'$ and $(1 + 1/\beta)^\beta \leq e$ for any $\beta > 0$.

Finally, averaging (\ref{eq:e_bound}) over $i \in \mathcal{A}$ and applying Lemma~\ref{lemma:mixed_het} give the result. The case $j = 0$ holds trivially since $\mathbf{z}^i_{t,0} = \mathbf{z}_t$.
\end{proof}


\begin{lemma}[Full model discrepancy] \label{lemma:full_discrepancy}
Under the conditions of Lemma~\ref{lemma:partial}, for any step size satisfying $\eta \leq \frac{1}{3\sqrt{2}\mathcal{L}\max\{\tau, \tau'\}}$, we have, for any $j = \tau', \dots, H$,
\begin{align}
    \frac{1}{a} \sum_{i \in \mathcal{A}} \mathbb{E}[\| \mathbf{x}_{t,j}^{i} - \mathbf{x}_t \|^2]
    &\leq 2\tau\eta^2 \sigma_x^2 + 18\tau^2\eta^2 \sigma_G^2 + 18\tau^2\eta^2 \mathbb{E}[ \| \nabla F(\mathbf{x}_t) \|^2] \nonumber \\
    &\quad + 6\eta^2 \tau' \sigma_z^2 + 6\eta^2 \tau' \left( 4\alpha(1-\alpha) + 9\tau' \gamma \right) \sigma_G^2 + 54\eta^2 (\tau')^2 \| \nabla_z F(\mathbf{x}_t) \|^2. \nonumber
\end{align}
\end{lemma}

\begin{proof}
The full-model steps start from the output of the tail steps, $\mathbf{x}^i_{t,\tau'} = (\mathbf{y}_t, \mathbf{z}^i_{t,\tau'})$.
For any client $i \in [m]$ and $j = \tau', \dots, H-1$, we have
\begin{align}
    \mathbb{E}[\| \mathbf{x}_{t,j+1}^{i} - \mathbf{x}_t \|^2]
    &= \mathbb{E}[ \| \mathbf{x}_{t,j}^{i} - \mathbf{x}_t -\eta ( g_{t,j}^{i} - \nabla F_i(\mathbf{x}_{t,j}^{i}) ) - \eta \nabla F_i(\mathbf{x}_{t,j}^{i}) \|^2 ] \nonumber \\
    &= \eta^2 \mathbb{E}[\| g_{t,j}^{i} - \nabla F_i(\mathbf{x}_{t,j}^{i}) \|^2] + \mathbb{E}[ \| \mathbf{x}_{t,j}^{i} - \mathbf{x}_t -\eta \nabla F_i(\mathbf{x}_{t,j}^{i}) \|^2] \label{eq:meanzero} \\
    &\leq \eta^2 \sigma_x^2 + \left(1 + \frac{1}{2\tau} \right) \mathbb{E}[ \| \mathbf{x}_{t,j}^{i} - \mathbf{x}_t \|^2 ] + (1 + 2\tau) \eta^2  \| \nabla F_i(\mathbf{x}_{t,j}^{i}) \|^2 \label{eq:split} \\
    &\leq \eta^2 \sigma_x^2 + \left(1 + \frac{1}{2\tau} + 9\tau \eta^2 \mathcal{L}^2 \right) \mathbb{E}[ \| \mathbf{x}_{t,j}^{i} - \mathbf{x}_t \|^2 ] + 9\tau \eta^2 \sigma_G^2 + 9\tau \eta^2 \| \nabla F(\mathbf{x}_t) \|^2, \label{eq:full_three}
\end{align}
where (\ref{eq:meanzero}) holds because the cross term vanishes since $\mathbb{E}[g_{t,j}^{i}] = \nabla F_i(\mathbf{x}_{t,j}^{i})$ given $\mathbf{x}^i_{t,j}$, (\ref{eq:split}) follows from $\| \mathbf{u} + \mathbf{v} \|^2 \leq (1 + \frac{1}{c}) \| \mathbf{u} \|^2 + (1 + c) \| \mathbf{v} \|^2$ with $c = 2\tau$, and (\ref{eq:full_three}) follows from $1 + 2\tau \leq 3\tau$, adding and subtracting $\nabla F_i(\mathbf{x}_t)$ and $\nabla F(\mathbf{x}_t)$, and Assumptions 1 and 3.

If $\eta \leq \frac{1}{3\sqrt{2}\tau\mathcal{L}}$, we have $9\tau\eta^2\mathcal{L}^2 \leq \frac{1}{2\tau}$, and thus
\begin{align}
    \mathbb{E}[\| \mathbf{x}_{t,j+1}^{i} - \mathbf{x}_t \|^2] \leq \left(1 + \frac{1}{\tau} \right) \mathbb{E}[ \| \mathbf{x}_{t,j}^{i} - \mathbf{x}_t \|^2 ] + \Phi_t, \nonumber
\end{align}
where $\Phi_t := \eta^2 \sigma_x^2 + 9\tau \eta^2 \sigma_G^2 + 9\tau \eta^2 \| \nabla F(\mathbf{x}_t) \|^2$.
Unrolling the recursion down to $j = \tau'$, for any $s = 0, \dots, \tau$, we have
\begin{align}
    \mathbb{E}[\| \mathbf{x}_{t,\tau'+s}^{i} - \mathbf{x}_t \|^2]
    &\leq \left(1 + \frac{1}{\tau} \right)^{s} \mathbb{E}[\| \mathbf{x}_{t,\tau'}^{i} - \mathbf{x}_t \|^2] + \sum_{l=0}^{s-1} \left(1 + \frac{1}{\tau} \right)^{l} \Phi_t \nonumber \\
    &\leq e\, \mathbb{E}[\| \mathbf{z}_{t,\tau'}^{i} - \mathbf{z}_t \|^2] + \tau \left( \left(1 + \frac{1}{\tau} \right)^{\tau} - 1 \right) \Phi_t \nonumber \\
    &\leq 3\, \mathbb{E}[\| \mathbf{z}_{t,\tau'}^{i} - \mathbf{z}_t \|^2] + 2\tau \Phi_t, \label{eq:full_unroll}
\end{align}
where we use $s \leq \tau$ and $(1 + 1/\tau)^{\tau} \leq e$.

Finally, averaging (\ref{eq:full_unroll}) over $i \in \mathcal{A}$ and applying Lemma~\ref{lemma:partial} with $j = \tau'$ to the first term give the result.
\end{proof}

\clearpage

\begin{theorem}
Under Assumptions 1--4, suppose that $\mathcal{A}$ is a set of $a$ clients sampled uniformly at random without replacement from $[m]$, and each client $i \in \mathcal{A}$ selects a client $k$ uniformly at random from $\mathcal{A} \setminus \{i\}$. Let $H = \tau + \tau'$ and $\rho = \tau / H \geq \frac{1}{2}$. If $\eta \leq \frac{1}{24\mathcal{L}\tau}$, then ASFL satisfies:
\begin{align}
    \frac{1}{T}\sum_{t=0}^{T-1} \mathbb{E}\left[ \left\| \nabla F(\mathbf{x}_t) \right\|^2 \right] \leq \frac{4\left( F(\mathbf{x}_0) - F^* \right)}{\eta\tau T} + \frac{8\eta\mathcal{L}}{a\tau}A + \frac{8\eta^2\mathcal{L}^2}{\tau}B, \nonumber
\end{align}
where $\gamma = \alpha^2 + (1-\alpha)^2$, and
\begin{align}
    A &:= \tau\sigma_x^2 + \tau'\sigma_z^2 + \tau' \left( 4\alpha(1-\alpha) + 3\tau' \left( \frac{m-a}{m-1} + (1-\alpha)^2 \right) \right) \sigma_G^2, \nonumber \\
    B &:= \tau^2\sigma_x^2 + \tau'(3\tau+\tau')\sigma_z^2 + \left( 9\tau^3 + \tau'(3\tau+\tau') \left( 4\alpha(1-\alpha) + 9\tau'\gamma \right) \right) \sigma_G^2. \nonumber
\end{align}
\end{theorem}

\begin{proof}
We begin from Lemma~\ref{lemma:framework} as follows:
\begin{align}
    \mathbb{E}[F(\mathbf{x}_{t+1})] &\leq F(\mathbf{x}_t) - \frac{\eta\tau}{2} \left\| \nabla F(\mathbf{x}_t) \right\|^2 - \frac{\eta\tau'}{4} \left\| \nabla_z F(\mathbf{x}_t) \right\|^2 \nonumber \\
    &\qquad +\eta\mathcal{L}^2 \left( \sum_{j=\tau'}^{H-1} \underbrace{\frac{1}{a}\sum_{i\in\mathcal{A}} \mathbb{E} \left[ \left\| \mathbf{x}_{t,j}^i - \mathbf{x}_t \right\|^2 \right]}_{\text{full}} + \sum_{j=0}^{\tau'-1} \underbrace{\frac{1}{a} \sum_{i\in\mathcal{A}} \mathbb{E} \left[ \left\| \mathbf{z}_{t,j}^i - \mathbf{z}_t \right\|^2 \right]}_{\text{partial}} \right) \nonumber \\
    &\qquad + \frac{2\eta^2\mathcal{L}}{a} \left( \tau \sigma_x^2 + \tau' \sigma_z^2 + \tau' \left( 4\alpha(1-\alpha) + 3\tau' \left( \frac{m-a}{m-1} + (1-\alpha)^2 \right) \right) \sigma_G^2 \right). \nonumber
\end{align}
By plugging in the results of Lemma~\ref{lemma:partial} and ~\ref{lemma:full_discrepancy}, we have:
\begin{align}
    \mathbb{E}[F(\mathbf{x}_{t+1})] &\leq F(\mathbf{x}_t) - \frac{\eta\tau}{2} \left\| \nabla F(\mathbf{x}_t) \right\|^2 - \frac{\eta\tau'}{4} \left\| \nabla_z F(\mathbf{x}_t) \right\|^2 \nonumber \\
    &\qquad +\eta\mathcal{L}^2 \sum_{j=\tau'}^{H-1} \bigg( 2\tau\eta^2 \sigma_x^2 + 18\tau^2\eta^2 \sigma_G^2 + 18\tau^2\eta^2 \mathbb{E}[ \| \nabla F(\mathbf{x}_t) \|^2] \nonumber \\
    &\qquad \qquad \qquad + 6\eta^2 \tau' \sigma_z^2 + 6\eta^2 \tau' \left( 4\alpha(1-\alpha) + 9\tau' \gamma \right) \sigma_G^2 + 54\eta^2 (\tau')^2 \| \nabla_z F(\mathbf{x}_t) \|^2 \bigg) \nonumber \\
    &\qquad + \eta\mathcal{L}^2\sum_{j=0}^{\tau'-1} \left( 2\eta^2 \tau' \sigma_z^2 + 2\eta^2 \tau' \left( 4\alpha(1-\alpha) + 9\tau' \gamma \right) \sigma_G^2 + 18\eta^2 (\tau')^2 \| \nabla_z F(\mathbf{x}_t) \|^2 \right) \nonumber \\
    &\qquad + \frac{2\eta^2\mathcal{L}}{a} \left( \tau \sigma_x^2 + \tau' \sigma_z^2 + \tau' \left( 4\alpha(1-\alpha) + 3\tau' \left( \frac{m-a}{m-1} + (1-\alpha)^2 \right) \right) \sigma_G^2 \right), \nonumber
\end{align}
where $\gamma = \alpha^2 + (1-\alpha)^2$.

By rearranging the terms, we have:
\begin{align}
    \mathbb{E}[F(\mathbf{x}_{t+1})] &\leq F(\mathbf{x}_t) - \frac{\eta\tau}{2} \left\| \nabla F(\mathbf{x}_t) \right\|^2 - \frac{\eta\tau'}{4} \left\| \nabla_z F(\mathbf{x}_t) \right\|^2 \nonumber \\
    &\quad + \eta\mathcal{L}^2\tau \Big( 2\tau\eta^2\sigma_x^2 + 18\tau^2\eta^2\sigma_G^2 + 18\tau^2\eta^2 \left\| \nabla F(\mathbf{x}_t) \right\|^2 \nonumber \\
    &\qquad\qquad + 6\eta^2\tau'\sigma_z^2 + 6\eta^2\tau' \left( 4\alpha(1-\alpha) + 9\tau'\gamma \right) \sigma_G^2 + 54\eta^2(\tau')^2 \left\| \nabla_z F(\mathbf{x}_t) \right\|^2 \Big) \nonumber \\
    &\quad + \eta\mathcal{L}^2\tau' \Big( 2\eta^2\tau'\sigma_z^2 + 2\eta^2\tau' \left( 4\alpha(1-\alpha) + 9\tau'\gamma \right) \sigma_G^2 + 18\eta^2(\tau')^2 \left\| \nabla_z F(\mathbf{x}_t) \right\|^2 \Big) \nonumber \\
    &\quad + \frac{2\eta^2\mathcal{L}}{a} \left( \tau\sigma_x^2 + \tau'\sigma_z^2 + \tau' \left( 4\alpha(1-\alpha) + 3\tau' \left( \frac{m-a}{m-1} + (1-\alpha)^2 \right) \right) \sigma_G^2 \right). \nonumber
\end{align}
Then, by rearranging the terms, we have:
\begin{align}
    \mathbb{E}[F(\mathbf{x}_{t+1})] &\leq F(\mathbf{x}_t) - \frac{\eta\tau}{2} \left( 1 - 36\eta^2\mathcal{L}^2\tau^2 \right) \left\| \nabla F(\mathbf{x}_t) \right\|^2 - \frac{\eta\tau'}{4} \left( 1 - 72\eta^2\mathcal{L}^2\tau'(3\tau+\tau') \right) \left\| \nabla_z F(\mathbf{x}_t) \right\|^2 \nonumber \\
    &\quad + \frac{2\eta^2\mathcal{L}}{a} \left( \tau\sigma_x^2 + \tau'\sigma_z^2 + \tau' \left( 4\alpha(1-\alpha) + 3\tau' \left( \frac{m-a}{m-1} + (1-\alpha)^2 \right) \right) \sigma_G^2 \right) \nonumber \\
    &\quad + 2\eta^3\mathcal{L}^2 \Big( \tau^2\sigma_x^2 + \tau'(3\tau+\tau')\sigma_z^2 + \left( 9\tau^3 + \tau'(3\tau+\tau') \left( 4\alpha(1-\alpha) + 9\tau'\gamma \right) \right) \sigma_G^2 \Big). \nonumber
\end{align}
Finally, if $\eta \leq \frac{1}{24\mathcal{L}\max\{\tau,\tau'\}}$, which satisfies the step size conditions of Lemma~\ref{lemma:framework}, Lemma~\ref{lemma:partial}, and Lemma~\ref{lemma:full_discrepancy}, we have $36\eta^2\mathcal{L}^2\tau^2 \leq \frac{1}{16}$ and $72\eta^2\mathcal{L}^2\tau'(3\tau+\tau') \leq 288\eta^2\mathcal{L}^2\max\{\tau,\tau'\}^2 \leq \frac{1}{2}$. Thus,
\begin{align}
    \mathbb{E}[F(\mathbf{x}_{t+1})] &\leq F(\mathbf{x}_t) - \frac{\eta\tau}{4} \left\| \nabla F(\mathbf{x}_t) \right\|^2 - \frac{\eta\tau'}{8} \left\| \nabla_z F(\mathbf{x}_t) \right\|^2 \nonumber \\
    &\quad + \frac{2\eta^2\mathcal{L}}{a} \left( \tau\sigma_x^2 + \tau'\sigma_z^2 + \tau' \left( 4\alpha(1-\alpha) + 3\tau' \left( \frac{m-a}{m-1} + (1-\alpha)^2 \right) \right) \sigma_G^2 \right) \nonumber \\
    &\quad + 2\eta^3\mathcal{L}^2 \left( \tau^2\sigma_x^2 + \tau'(3\tau+\tau')\sigma_z^2 + \left( 9\tau^3 + \tau'(3\tau+\tau') \left( 4\alpha(1-\alpha) + 9\tau'\gamma \right) \right) \sigma_G^2 \right) \nonumber \\
    &\leq F(\mathbf{x}_t) - \frac{\eta\tau}{4} \left\| \nabla F(\mathbf{x}_t) \right\|^2 \nonumber \\
    &\quad + \frac{2\eta^2\mathcal{L}}{a} \left( \tau\sigma_x^2 + \tau'\sigma_z^2 + \tau' \left( 4\alpha(1-\alpha) + 3\tau' \left( \frac{m-a}{m-1} + (1-\alpha)^2 \right) \right) \sigma_G^2 \right) \nonumber \\
    &\quad + 2\eta^3\mathcal{L}^2 \left( \tau^2\sigma_x^2 + \tau'(3\tau+\tau')\sigma_z^2 + \left( 9\tau^3 + \tau'(3\tau+\tau') \left( 4\alpha(1-\alpha) + 9\tau'\gamma \right) \right) \sigma_G^2 \right), \label{eq:drop}
\end{align}
where (\ref{eq:drop}) holds since $-\frac{\eta\tau'}{8}\left\| \nabla_z F(\mathbf{x}_t) \right\|^2 \leq 0$.

For brevity, let
\begin{align}
    A &:= \tau\sigma_x^2 + \tau'\sigma_z^2 + \tau' \left( 4\alpha(1-\alpha) + 3\tau' \left( \frac{m-a}{m-1} + (1-\alpha)^2 \right) \right) \sigma_G^2, \nonumber \\
    B &:= \tau^2\sigma_x^2 + \tau'(3\tau+\tau')\sigma_z^2 + \left( 9\tau^3 + \tau'(3\tau+\tau') \left( 4\alpha(1-\alpha) + 9\tau'\gamma \right) \right) \sigma_G^2. \nonumber
\end{align}
By rearranging the terms in (\ref{eq:drop}) and dividing both sides by $\frac{\eta\tau}{4}$, we have:
\begin{align}
    \left\| \nabla F(\mathbf{x}_t) \right\|^2 \leq \frac{4\left( F(\mathbf{x}_t) - \mathbb{E}[F(\mathbf{x}_{t+1})] \right)}{\eta\tau} + \frac{8\eta\mathcal{L}}{a\tau}A + \frac{8\eta^2\mathcal{L}^2}{\tau}B. \nonumber
\end{align}
Taking the total expectation and summing over $t = 0, \ldots, T-1$, the first term on the right-hand side telescopes.
Then, dividing both sides by $T$ and using $F(\mathbf{x}_T) \geq F^* := \inf_{\mathbf{x}} F(\mathbf{x})$, we have:
\begin{align}
    \frac{1}{T}\sum_{t=0}^{T-1} \mathbb{E}\left[ \left\| \nabla F(\mathbf{x}_t) \right\|^2 \right] \leq \frac{4\left( F(\mathbf{x}_0) - F^* \right)}{\eta\tau T} + \frac{8\eta\mathcal{L}}{a\tau}A + \frac{8\eta^2\mathcal{L}^2}{\tau}B. \label{eq:final_general}
\end{align}
Finally, by setting $\eta = \sqrt{\frac{a}{\tau T}}$, which satisfies $\eta \leq \frac{1}{24\mathcal{L}\max\{\tau,\tau'\}}$ when $T \geq \frac{576\mathcal{L}^2\max\{\tau,\tau'\}^2 a}{\tau}$, we have:
\begin{align}
    \frac{1}{T}\sum_{t=0}^{T-1} \mathbb{E}\left[ \left\| \nabla F(\mathbf{x}_t) \right\|^2 \right] \leq \frac{4\left( F(\mathbf{x}_0) - F^* \right) + 8\mathcal{L}A/\tau}{\sqrt{a\tau T}} + \frac{8a\mathcal{L}^2 B}{\tau^2 T}. \nonumber
\end{align}
\end{proof}

\clearpage
\subsection{Communication pattern}\label{appendix:communication-pattern}
\subsubsection{Discussion of gossip-style communication}\label{appendix:gossip-communication}

We implement latent-space information sharing through a gossip-style communication pattern~\cite{boyd2006randomized}.
Rather than relying on a centralized collector or broadcasting latent information to all clients, each active client receives latent-space information from one randomly selected active client at each communication round.
Specifically, at round $r$, each active client $i$ uniformly samples one source client $j$, where $j \neq i$, and receives the latent-space information $I_j^{(r)}$ generated by client $j$.
Here, the directed communication link $j \rightarrow i$ indicates that latent-space information flows from the sampled source client $j$ to the receiving client $i$.

This communication pattern is particularly suitable for federated learning scenarios with a large number of edge devices. In such environments, clients are naturally distributed and can transfer latent-space information through local client-level connections rather than persistent global connectivity.
Our method organizes the transfer of $I_i$ over randomly formed directed communication links while preserving the standard federated training procedure.
Thus, the proposed method should be understood as latent-space information sharing over a directed gossip-style communication graph within the standard federated training procedure.

This gossip-style communication pattern has two practical advantages over centralized communication.
First, it does not incur blocking.
In contrast, when activations are aggregated at a central server and redistributed to clients, blocking is inevitable and increases with the number of clients.
Therefore, our design offers better scalability than typical centralized approaches.
Second, it enables more robust communication in FL environments.
Even if a client is disconnected, gossip-style latent information sharing does not stall the entire system, resulting in a more robust training process.
Nevertheless, exploring alternative communication patterns for latent information sharing remains a promising direction for future work.
\clearpage
\subsection {Additional discussion on privacy} \label{appendix:privacy}
While this paper primarily focuses on the proposed activation sharing mechanism for accelerating federated learning, privacy leakage remains a critical concern for practical deployment.
In this appendix, we provide additional experimental results and analysis on the differential privacy aspects of ASFL.

\subsubsection{Neighboring data and scope of the privacy guarantee}
\label{subsec:privacy_neighbors}
Similar to metric DP~\citep{chatzikokolakis2013broadening}, PixelDP~\citep{lecuyer2018certified} defines the neighboring relation based on the $p$-norm distance between inputs.
Specifically, two inputs $\xi$ and $\xi'$ are neighbors if
\begin{equation}
    \xi' = \xi + \delta', \qquad
    \delta' \in B_p(R), \qquad
    B_p(R) = \{\delta' \in \mathbb{R}^{n} : \|\delta'\|_p \le R\}.
    \label{eq:privacy_neighbors}
\end{equation}
PixelDP provides indistinguishability within this input radius.
We consider an honest-but-curious recipient that observes every activation it receives.

The sensitivity of the input-side subnetwork $\mathbf{y}$, which governs the required noise magnitude, is defined as:
\begin{equation}
\triangle^{\mathbf{y}}_{p,q} \;=\; \max_{\xi \neq \xi'} \frac{\bigl\|\mathbf{y}(\xi) - \mathbf{y}(\xi')\bigr\|_q}{\|\xi - \xi'\|_p},\nonumber
\end{equation}
where $\xi$ and $\xi'$ are two input samples of the corresponding client, and $p$ and $q$ are the degrees of input and output norms.
In addition, $R$ defines the radius of the threat model such that any adversarial perturbation confined within $B_p(R):= \{\delta' \in \mathbb{R}^{n} : \|\delta'\|_p \leq R\}$ is 
guaranteed to be masked by the injected noise and provides $(\epsilon, \delta)$-DP guarantees. 

We consider two noise mechanisms for privatizing the shared activations $I_i$, each offering a distinct privacy-utility trade-off. 
\begin{itemize}
\item \textbf{Laplace mechanism:} The Laplace mechanism adds independent noise drawn from a Laplace distribution to the activation vector: $ I^{P}_{i} \;:=\; \mathbf{y}(\mathcal{B}_{i,t}) \;+\; \mathrm{Lap}\!\left(\frac{\triangle^{\mathbf{y}}_{p,q}L}{\epsilon}\right), \label{eq:laplace_mech} $ where $\epsilon > 0$ is the privacy budget, and $I^{P}_{i}$ denotes the privatized activation vector. The Laplace mechanism satisfies \emph{pure} $(\epsilon, 0)$-DP.
\item \textbf{Gaussian mechanism:} The Gaussian mechanism adds isotropic Gaussian noise to the activation vector: $ I^{P}_{i} \;:=\; \mathbf{y}(\mathcal{B}_{i,t}) \;+\; \mathcal{N}\!\left(0,\, \sigma^{2} \mathbf{I}\right), \quad \sigma \;=\; \frac{\triangle^{\mathbf{y}}_{p,q}L\sqrt{2\ln(1.25/\delta)}}{\epsilon}. \label{eq:gaussian_mech} $ The Gaussian mechanism satisfies $(\epsilon, \delta)$-DP, where $\delta$ represents a small probability of the privacy guarantee failing. 
\end{itemize}
In both cases, by the post-processing immunity, the privatized activations $I^{P}_{i}$ forwarded to the paired client yield a received representation that is $(\epsilon, \delta)$-DP (or $(\epsilon, 0)$-DP for Laplace) with respect to the raw local data of the sending client.

\textbf{Impact of PixelDP on privacy leakage in ASFL} --
The following lemma, established in~\cite{lecuyer2018certified}, formalizes the robustness and privacy guarantee provided by PixelDP.

\begin{lemma}[PixelDP Stability Bound~\cite{lecuyer2018certified}]
\label{lemma:pixeldp}
Suppose a randomized function $\mathcal{M}$ with bounded output
$\mathcal{M} \in [0, b]$, $b \in \mathbb{R}^{+}$, satisfies
$(\epsilon, \delta)$-DP.
Then the expected value of its output satisfies:
\begin{equation}
    \forall\, \delta' \in B_p(1)\colon\quad
    \mathbb{E}\bigl[\mathcal{M}(x)\bigr]
    \;\leq\;
    e^{\epsilon}\cdot\mathbb{E}\bigl[\mathcal{M}(x + \delta')\bigr]
    \;+\; b\delta,
    \label{eq:pixeldp_lemma}
\end{equation}
where the expectation is taken over the randomness of $\mathcal{M}$.
\end{lemma}

Lemma~\ref{lemma:pixeldp} implies that the expected scoring function of the model is stable under bounded input perturbations, providing a certified robustness guarantee against adversarial examples in addition to privacy
protection.
Together, these properties make PixelDP a principled choice for protecting the activation sharing step in ASFL.

\subsubsection{Accumulation of privacy loss}
\label{subsec:privacy_composition}
Privacy loss accumulates when activations of the same data sample are repeatedly released with independent noise.
Let $k$ denote the total number of times the recipient receives an activation of the same sample over the entire training process.
If each release satisfies $(\epsilon,\delta)$-PixelDP, the $k$ releases together satisfy
\begin{equation}
    (\epsilon_{\mathrm{total}},\delta_{\mathrm{total}})
    = (k\epsilon,k\delta)
    \label{eq:privacy_composition}
\end{equation}
by sequential composition.
In contrast, reusing a received activation batch for $\tau'$ tail updates within a round is post-processing.
It incurs no additional privacy cost and is therefore not counted in $k$.

\subsubsection{Impact of activation sharing layer on performance}
\label{subsec:noise_layer_impact}
We evaluate the effect of the activation sharing layer under three configurations:
ASFL without noise, ASFL with Laplace noise ($\epsilon=0.2$), and ASFL with Gaussian noise ($\epsilon=0.5$, $\delta=0.001$).
Tables~\ref{tab:noise_layer_cifar10}--\ref{tab:noise_layer_femnist}
report accuracy across all datasets, together with Laplace Drop and Gaussian Drop, defined as the absolute accuracy degradation caused by each noise mechanism relative to the noise-free baseline.

Sharing activations at deeper layers generally reduces ASFL accuracy.
On CIFAR-10, it decreases from $79.6\%$ at the earliest partition point to $66.0\%$ at the deepest.
This trend suggests that deeper, lower-dimensional and more task-specific representations offer less diversity and transferable information for collaboration across heterogeneous clients.

Both Laplace Drop and Gaussian Drop generally decrease at deeper layers.
For example, on CIFAR-100, Gaussian Drop falls from $5.04\%$ to $0.89\%$.
These results suggest a trade-off: earlier layers provide greater collaboration benefits but require more careful noise calibration to limit utility loss, whereas deeper layers show smaller noise-induced degradation but offer diminishing collaboration benefits.

\begin{table*}[ht!]
\caption{Impact of activation sharing layer on performance of CIFAR-10 with Laplace ($\epsilon=0.2$) and Gaussian ($\delta=0.001, \epsilon=0.5$). Laplace Drop and Gaussian Drop denote the accuracy reduction (\%) of ASFL + Laplace and ASFL + Gaussian relative to the ASFL baseline, respectively.}
\label{tab:noise_layer_cifar10}
\begin{center}
\resizebox{\textwidth}{!}{\begin{tabular}{ccccccccccc}  
    \toprule   
        & $l=0$ & $l=1$ & $l=2$ & $l=3$ & $l=4$ & $l=5$ & $l=6$ & $l=7$ & $l=8$ & $l=9$   \\
    \midrule
    ASFL & $79.6 \pm \text{\scriptsize 0.71}$ & $76.74 \pm \text{\scriptsize 0.47}$ & $76.87 \pm \text{\scriptsize 0.38}$ &  $77.38 \pm \text{\scriptsize 0.26}$ & $75.24 \pm \text{\scriptsize 0.03}$ & $73.85 \pm \text{\scriptsize 0.31}$ & $70.14 \pm \text{\scriptsize 0.52}$ &  $69.38 \pm \text{\scriptsize 0.44}$ & $65.68 \pm \text{\scriptsize 0.21}$ &  $66.00 \pm \text{\scriptsize 0.58}$ \\
    ASFL + Laplace & $67.59 \pm \text{\scriptsize 0.95}$ & $69.77 \pm \text{\scriptsize 1.32}$ & $67.05 \pm \text{\scriptsize 0.87}$ &  $70.53 \pm \text{\scriptsize 0.95}$ & $69.83 \pm \text{\scriptsize 0.97}$ & $70.86 \pm \text{\scriptsize 0.40}$ & $67.93 \pm \text{\scriptsize 0.38}$ &  $66.81 \pm \text{\scriptsize 0.95}$ & $65.09 \pm \text{\scriptsize 0.37}$ &  $65.34 \pm \text{\scriptsize 0.95}$ \\
    ASFL + Gaussian & $66.59 \pm \text{\scriptsize 1.20}$ & $68.54 \pm \text{\scriptsize 0.18}$ & $68.91 \pm \text{\scriptsize 1.31}$ &  $69.63 \pm \text{\scriptsize 0.74}$ & $68.55 \pm \text{\scriptsize 0.70}$ & $70.16 \pm \text{\scriptsize 0.79}$ & $68.10 \pm \text{\scriptsize 0.46}$ &  $67.60 \pm \text{\scriptsize 0.48}$ & $65.48 \pm \text{\scriptsize 0.29}$ &  $65.27 \pm \text{\scriptsize 0.94}$ \\
    Laplace Drop & $11.47\downarrow$ & $6.97\downarrow$ & $9.82\downarrow$ &  $6.85\downarrow$ & $5.41\downarrow$ & $2.99\downarrow$ & $2.21\downarrow$ &  $2.57\downarrow$ & $0.59\downarrow$ &  $0.66\downarrow$ \\
    Gaussian Drop & $12.47\downarrow$ & $8.20\downarrow$ & $7.96\downarrow$ &  $7.75\downarrow$ & $6.69\downarrow$ & $3.69\downarrow$ & $2.04\downarrow$ &  $1.78\downarrow$ & $0.20\downarrow$ &  $0.73\downarrow$ \\
    \bottomrule
\end{tabular}}
\end{center}
\end{table*}

\begin{table*}[ht!]
\caption{Impact of activation sharing layer on performance of CIFAR-100 with Laplace ($\epsilon=0.2$) and Gaussian ($\delta=0.001, \epsilon=0.5$). Laplace Drop and Gaussian Drop denote the accuracy reduction (\%) of ASFL + Laplace and ASFL + Gaussian relative to the ASFL baseline, respectively.}
\label{tab:noise_layer_cifar100}
\begin{center}
\resizebox{\textwidth}{!}{\begin{tabular}{ccccccccccc}  
    \toprule   
        & $l=0$ & $l=1$ & $l=2$ & $l=3$ & $l=4$ & $l=5$ & $l=6$ & $l=7$ & $l=8$ & $l=9$   \\
    \midrule
    ASFL & $50.88 \pm \text{\scriptsize 0.78}$ & $49.91 \pm \text{\scriptsize 0.40}$ & $49.12 \pm \text{\scriptsize 0.73}$ &  $47.81 \pm \text{\scriptsize 0.24}$ & $40.95 \pm \text{\scriptsize 0.95}$ & $31.82 \pm \text{\scriptsize 1.13}$ & $32.44 \pm \text{\scriptsize 2.18}$ & $39.21 \pm \text{\scriptsize 0.85}$ & $44.08 \pm \text{\scriptsize 1.22}$ &  $47.23 \pm \text{\scriptsize 0.58}$ \\
    ASFL + Laplace & $47.16 \pm \text{\scriptsize 0.59}$ & $46.20 \pm \text{\scriptsize 0.54}$ & $45.11 \pm \text{\scriptsize 1.02}$ &  $43.91 \pm \text{\scriptsize 0.92}$ & $37.48 \pm \text{\scriptsize 0.91}$ & $29.92 \pm \text{\scriptsize 0.57}$ & $28.60 \pm \text{\scriptsize 1.45}$ & $38.06 \pm \text{\scriptsize 0.83}$ & $41.50 \pm \text{\scriptsize 0.80}$ &  $44.80 \pm \text{\scriptsize 0.48}$ \\
    ASFL + Gaussian & $45.84 \pm \text{\scriptsize 1.00}$ & $46.40 \pm \text{\scriptsize 0.24}$ & $45.19 \pm \text{\scriptsize 0.05}$ &  $43.48 \pm \text{\scriptsize 0.94}$ & $39.61 \pm \text{\scriptsize 2.41}$ & $30.90 \pm \text{\scriptsize 0.57}$ & $29.75 \pm \text{\scriptsize 0.75}$ & $38.03 \pm \text{\scriptsize 0.05}$ & $41.17 \pm \text{\scriptsize 0.48}$ & $46.34 \pm \text{\scriptsize 0.17}$ \\
    Laplace Drop & $3.72\downarrow$ & $3.71\downarrow$ & $4.01\downarrow$ &  $3.9\downarrow$ & $3.47\downarrow$ & $1.90\downarrow$ & $3.84\downarrow$ &  $1.15\downarrow$ & $2.58\downarrow$ &  $2.43\downarrow$ \\
    Gaussian Drop & $5.04\downarrow$ & $3.71\downarrow$ & $3.93\downarrow$ &  $4.33\downarrow$ & $1.34\downarrow$ & $0.92\downarrow$ & $2.69\downarrow$ &  $1.18\downarrow$ & $2.91\downarrow$ &  $0.89\downarrow$ \\
    \bottomrule
\end{tabular}}
\end{center}
\end{table*}

\begin{table*}[ht!]
\caption{Impact of activation sharing layer on performance of SVHN with Laplace ($\epsilon=0.2$) and Gaussian ($\delta=0.001, \epsilon=0.5$). Laplace Drop and Gaussian Drop denote the accuracy reduction (\%) of ASFL + Laplace and ASFL + Gaussian relative to the ASFL baseline, respectively.}
\label{tab:noise_layer_svhn}
\begin{center}
\resizebox{\textwidth}{!}{\begin{tabular}{cccccccccccc}  
    \toprule   
        & $l=0$ & $l=1$ & $l=2$ & $l=3$ & $l=4$ & $l=5$ & $l=6$ & $l=7$ & $l=8$ & $l=9$ & $l=10$  \\
    \midrule
    ASFL & $89.20 \pm \text{\scriptsize 0.16}$ & $89.37 \pm \text{\scriptsize 0.10}$ & $89.11 \pm \text{\scriptsize 0.03}$ &  $89.40 \pm \text{\scriptsize 0.18}$ & $89.34 \pm \text{\scriptsize 0.11}$ & $89.12 \pm \text{\scriptsize 0.10}$ & $88.62 \pm \text{\scriptsize 0.28}$ & $89.43 \pm \text{\scriptsize 0.55}$ & $88.39 \pm \text{\scriptsize 0.15}$ &  $86.61 \pm \text{\scriptsize 0.41}$ &  $86.02 \pm \text{\scriptsize 0.59}$ \\
    ASFL + Laplace & $19.66 \pm \text{\scriptsize 0.09}$ & $51.61 \pm \text{\scriptsize 24.99}$ & $19.61 \pm \text{\scriptsize 0.02}$ &  $41.58 \pm \text{\scriptsize 31.07}$ & $64.19 \pm \text{\scriptsize 31.53}$ & $87.07 \pm \text{\scriptsize 0.41}$ & $87.56 \pm \text{\scriptsize 0.40}$ & $87.71 \pm \text{\scriptsize 0.88}$ & $86.76 \pm \text{\scriptsize 0.65}$ &  $85.97 \pm \text{\scriptsize 0.30}$ &  $85.31 \pm \text{\scriptsize 0.90}$ \\
    ASFL + Gaussian & $22.73 \pm \text{\scriptsize 4.44}$ & $32.99 \pm \text{\scriptsize 18.77}$ & $19.60 \pm \text{\scriptsize 0.01}$ &  $54.13 \pm \text{\scriptsize 25.37}$ & $64.76 \pm \text{\scriptsize 31.94}$ & $85.95 \pm \text{\scriptsize 0.81}$ & $85.82 \pm \text{\scriptsize 0.92}$ & $86.67 \pm \text{\scriptsize 0.90}$ & $88.00 \pm \text{\scriptsize 0.25}$ &  $85.73 \pm \text{\scriptsize 0.73}$ &  $85.36 \pm \text{\scriptsize 1.23}$ \\
    Laplace Drop & $69.54\downarrow$ & $37.76\downarrow$ & $69.50\downarrow$ &  $47.82\downarrow$ & $25.15\downarrow$ & $2.05\downarrow$ & $1.06\downarrow$ &  $1.72\downarrow$ & $1.63\downarrow$ & $0.64\downarrow$ & $0.71\downarrow$ \\
    Gaussian Drop & $66.47\downarrow$ & $56.38\downarrow$ & $69.51\downarrow$ &  $35.27\downarrow$ & $25.48\downarrow$ & $3.17\downarrow$ & $2.80\downarrow$ &  $2.76\downarrow$ & $0.39\downarrow$ & $0.88\downarrow$ & $0.66\downarrow$ \\
    \bottomrule
\end{tabular}}
\end{center}
\end{table*}

\begin{table}[ht!]
\caption{Impact of activation sharing layer on performance of FEMNIST with Laplace ($\epsilon=0.2$) and Gaussian ($\delta=0.001, \epsilon=0.5$). Laplace Drop and Gaussian Drop denote the accuracy reduction (\%) of ASFL + Laplace and ASFL + Gaussian relative to the ASFL baseline, respectively.}
\label{tab:noise_layer_femnist}
\begin{center}
\scriptsize
\resizebox{0.8\textwidth}{!}{\begin{tabular}{cccccc}  
    \toprule   
        & $l=0$ & $l=1$ & $l=2$ & $l=3$ & $l=4$  \\
    \midrule
    ASFL & $86.38 \pm \text{\scriptsize 0.17}$ & $86.27 \pm \text{\scriptsize 0.57}$ & $86.01 \pm \text{\scriptsize 0.20}$ &  $85.56 \pm \text{\scriptsize 0.35}$ & $85.76 \pm \text{\scriptsize 0.24}$  \\
    ASFL + Laplace & $83.66 \pm \text{\scriptsize 0.14}$ & $83.25 \pm \text{\scriptsize 0.51}$ & $84.77 \pm \text{\scriptsize 0.31}$ &  $85.03 \pm \text{\scriptsize 0.08}$ & $85.18 \pm \text{\scriptsize 0.25}$  \\
    ASFL + Gaussian & $83.66 \pm \text{\scriptsize 0.19}$ & $83.55 \pm \text{\scriptsize 0.40}$ & $84.96 \pm \text{\scriptsize 0.08}$ &  $85.05 \pm \text{\scriptsize 0.24}$ & $85.18 \pm \text{\scriptsize 0.16}$  \\
    Laplace Drop & $2.72\downarrow$ & $3.02\downarrow$ & $1.24\downarrow$ &  $0.53\downarrow$ & $0.58\downarrow$  \\
    Gaussian Drop & $2.72\downarrow$ & $2.72\downarrow$ & $1.05\downarrow$ &  $0.51\downarrow$ & $0.58\downarrow$  \\
    \bottomrule
\end{tabular}}
\end{center}
\end{table}
\subsubsection{Model inversion attack on shared activations}
\label{subsec:model-inversion-attack}

We examine the privacy behavior of ASFL under an optimization-based model inversion attack to assess whether Gaussian perturbation reduces visually recoverable information in shared latent representations.
An adversary observing these representations may attempt to reconstruct the corresponding private inputs.

Model inversion attacks in the image domain are commonly categorized into optimization-based and generative approaches.
Optimization-based methods directly optimize an input image so that its model response matches either a target intermediate representation~\cite{mahendran2015understanding} or a target model output, such as a class confidence score~\cite{fredrikson2015model}.
In contrast, generative methods introduce a learned image prior and perform inversion in the latent space of a generator, thereby constraining the reconstruction toward the natural image manifold~\cite{zhang2020secret,peng2025generative}.

In this experiment, we consider an optimization-based model inversion attack that reconstructs inputs by matching shared intermediate features. Given an observed shared activation $\tilde{I}$, the attacker reconstructs an input image by solving:
\begin{align}
    x^* =
    \arg\min_{x'} \left\| \mathbf{y}(x') - \tilde{I} \right\|_2^2 + \lambda R_{\mathrm{TV}}(x') \nonumber
\end{align}
where $\mathbf{y}(x')$ denotes the intermediate representation at the shared activation layer, $R_{\mathrm{TV}}(\cdot)$ is the total variation regularizer that suppresses high-frequency artifacts, and $\lambda$ controls its strength.

We evaluate the attack with Gaussian perturbation using $(\epsilon,\delta)=(0.5,0.05)$. Figure~\ref{fig:inversion-attack} shows the shared activations used for inversion and the corresponding reconstructed images.

\begin{figure}[ht!]
    \centering
    \includegraphics[width=1.0\columnwidth]{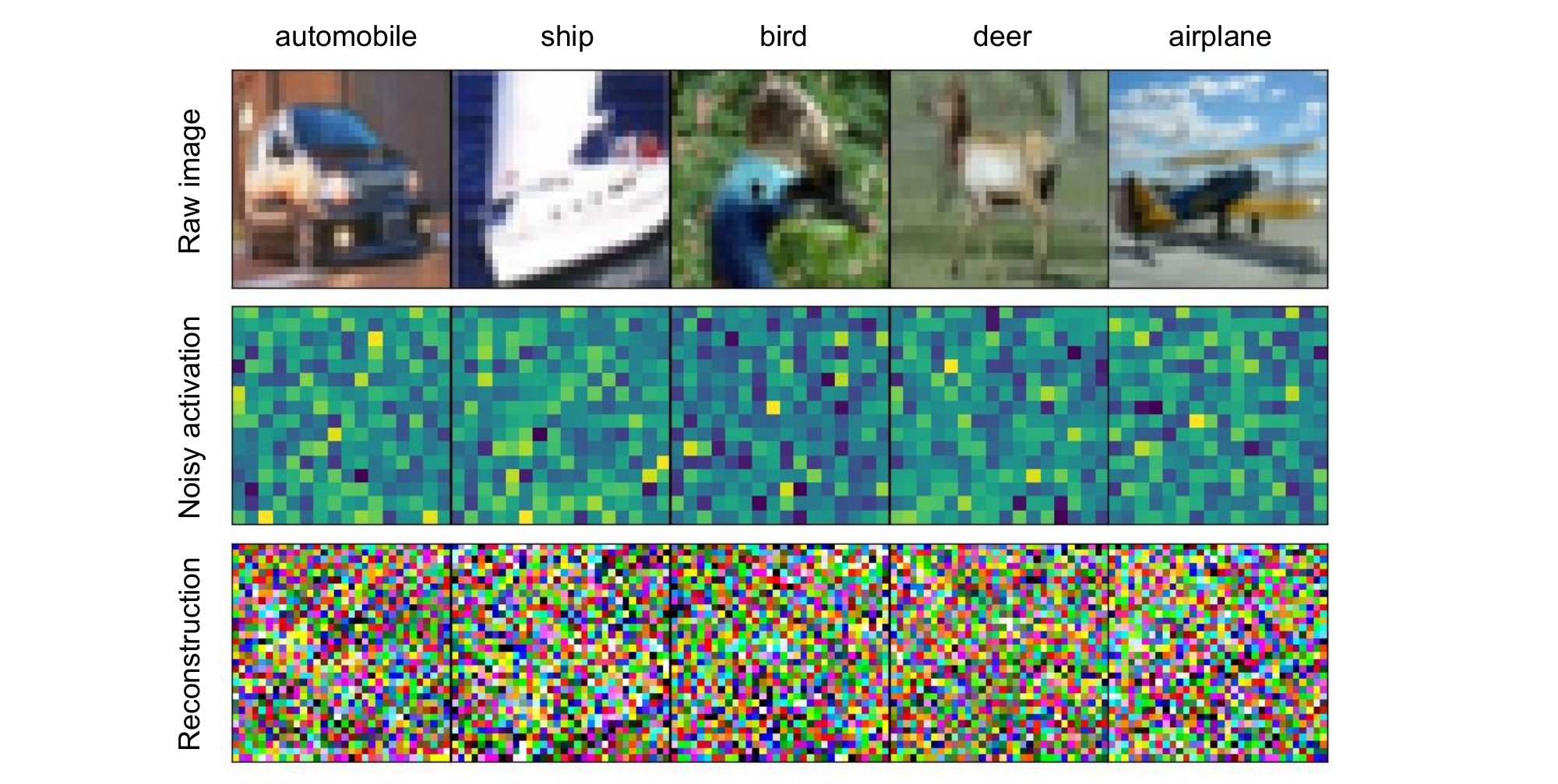}
    \caption{Optimization-based feature inversion attack on shared activations under Gaussian perturbation. The top row shows the original private input images for reference, the middle row visualizes the corresponding perturbed shared activations used for inversion, and the bottom row shows the reconstructed images.}
    \label{fig:inversion-attack}
\end{figure}

As shown in Figure~\ref{fig:inversion-attack}, the reconstructed images do not visually reveal the corresponding private inputs under this attack setting.

\textbf{Quantitative inversion results} --
We additionally evaluate the learning-based inversion attack of~\citet{he2019model} and DRAG~\citep{lei2025drag} on CIFAR-10.
Table~\ref{tab:inversion_pixeldp} reports PSNR, SSIM, and distance correlation (dCor) with and without PixelDP.
We observe that PixelDP substantially degrades all three metrics for both types of feature inversion attacks.

\begin{table}[t!]
\caption{Inversion attacks on ASFL with and without PixelDP on CIFAR-10. Laplace noise uses $\epsilon=0.2$, and Gaussian noise uses $(\epsilon,\delta)=(0.5,0.01)$.}
\label{tab:inversion_pixeldp}
\begin{center}
\scriptsize
\begin{tabular}{llrrr}
\toprule
Attack & Noise & PSNR & SSIM & dCor \\
\midrule
Learning-based & None & 29.84 & 0.9508 & 0.9419 \\
 & Laplace & 24.47 & 0.8560 & 0.9128 \\
 & Gaussian & 24.67 & 0.8645 & 0.8884 \\
\midrule
DRAG & None & 29.04 & 0.9442 & 0.9181 \\
 & Laplace & 20.02 & 0.7181 & 0.9032 \\
 & Gaussian & 19.95 & 0.7331 & 0.8706 \\
\bottomrule
\end{tabular}
\end{center}
\vspace{-0.4cm}
\end{table}

\textbf{Comparison with SplitFed} --
Table~\ref{tab:inversion_splitfed} compares ASFL and SplitFed without DP under the same two inversion attacks.
ASFL shows information leakage comparable to SplitFed under both evaluated attacks.

\begin{table}[t!]
\caption{Inversion attacks on ASFL and SplitFed without DP on CIFAR-10.}
\label{tab:inversion_splitfed}
\begin{center}
\vspace{-0.2cm}
\scriptsize
\begin{tabular}{llrrr}
\toprule
Attack & Method & PSNR & SSIM & dCor \\
\midrule
Learning-based & SplitFed & 28.79 & 0.9399 & 0.9149 \\
 & ASFL & 29.84 & 0.9508 & 0.9419 \\
\midrule
DRAG & SplitFed & 29.21 & 0.9493 & 0.9120 \\
 & ASFL & 29.04 & 0.9442 & 0.9181 \\
\bottomrule
\end{tabular}
\end{center}
\vspace{-0.4cm}
\end{table}

\subsubsection{Gradient perturbation and membership inference}
\label{subsec:gradient_perturbation}

PixelDP provides indistinguishability within an input radius, rather than record-level DP on the local data.
We therefore additionally evaluate ASFL combined with gradient perturbation.
Table~\ref{tab:gradient_perturbation_mia} reports accuracy and membership inference attack (MIA) results for ASFL and SplitFed.
ASFL retains higher accuracy than SplitFed with similar MIA resistance under the evaluated settings.

\begin{table}[h]
\caption{Validation accuracy and membership inference results for ASFL and SplitFed with gradient perturbation.}
\label{tab:gradient_perturbation_mia}
\begin{center}
\scriptsize
\begin{tabular}{cclccc}
\toprule
\multirow{2}{*}{$\epsilon$} & \multirow{2}{*}{$\delta$} & \multirow{2}{*}{Method} & Valid. acc. & \multirow{2}{*}{AUC} & TPR at 1\% FPR \\
 & & & (\%) & & (\%) \\
\midrule
20 & 0.1 & SplitFed & $70.0 \pm 0.1$ & $0.51 \pm 0.002$ & $1.2 \pm 0.3$ \\
 & & ASFL & $72.8 \pm 0.5$ & $0.52 \pm 0.000$ & $1.1 \pm 0.2$ \\
\midrule
20 & 0.01 & SplitFed & $68.0 \pm 1.0$ & $0.51 \pm 0.001$ & $1.0 \pm 0.2$ \\
 & & ASFL & $72.1 \pm 0.9$ & $0.52 \pm 0.001$ & $1.2 \pm 0.1$ \\
\midrule
10 & 0.1 & SplitFed & $65.4 \pm 0.9$ & $0.51 \pm 0.003$ & $1.2 \pm 0.2$ \\
 & & ASFL & $70.0 \pm 0.8$ & $0.51 \pm 0.001$ & $1.2 \pm 0.1$ \\
\midrule
10 & 0.01 & SplitFed & $62.7 \pm 0.7$ & $0.51 \pm 0.003$ & $0.9 \pm 0.1$ \\
 & & ASFL & $65.7 \pm 0.5$ & $0.51 \pm 0.003$ & $1.0 \pm 0.1$ \\
\bottomrule
\end{tabular}
\end{center}
\end{table}
\clearpage
\subsection{Additional experimental results}
\label{appendix:supplementary}

\subsubsection {Effect of the number of exchanged activations}\label{appendix:activation}
Table~\ref{tab:batches} presents ablation results with varying numbers of shared batches across clients.
The accuracy generally improves as more batches are exchanged, although communication overhead also increases accordingly.
However, the gains are not always consistent, possibly due to randomness in activation sampling.
If the remote data used to compute activations are similar to the local data, activation sharing may fail to reduce data heterogeneity while still incurring additional communication cost, and less informative shared activations may even degrade performance.
These results suggest that exchanging one or two randomly selected batches is sufficient to mitigate data heterogeneity with only marginal communication overhead.

\begin{table}[ht]
\caption{Ablation study on activation sharing quantity (the number of shared activation batches).}
\label{tab:batches}
\begin{center}
\setlength{\tabcolsep}{2pt}
\scriptsize
\begin{tabular}{lcccccccccc}
\toprule
\textbf{Shared}
    & \multicolumn{2}{c}{\textbf{CIFAR-10}}
    & \multicolumn{2}{c}{\textbf{CIFAR-100}}
    & \multicolumn{2}{c}{\textbf{SVHN}}
    & \multicolumn{2}{c}{\textbf{FEMNIST}} \\
\cmidrule(lr){2-3}\cmidrule(lr){4-5}\cmidrule(lr){6-7}\cmidrule(lr){8-9}
\textbf{Batches} & \textbf{Acc (\%)} \accup & \textbf{Comm (MB)} \commdown
& \textbf{Acc (\%)} \accup & \textbf{Comm (MB)} \commdown
& \textbf{Acc (\%)} \accup & \textbf{Comm (MB)} \commdown
& \textbf{Acc (\%)} \accup & \textbf{Comm (MB)} \commdown \\
\midrule
$8\times$
    & 74.15 $\pm$ 1.17 & 64
    & 53.27 $\pm$ 0.86 & 32
    & 89.06 $\pm$ 0.27 & 8
    & 84.82 $\pm$ 0.28 & 1.92 \\
$4\times$
    & 76.01 $\pm$ 0.50 & 32
    & 52.34 $\pm$ 1.12 & 16
    & 88.99 $\pm$ 0.49 & 4
    & 84.59 $\pm$ 0.21 & 0.96 \\
$2\times$
    & 75.81 $\pm$ 0.70 & 16
    & 51.18 $\pm$ 0.69 & 8
    & 89.26 $\pm$ 0.65 & 2
    & 84.61 $\pm$ 0.07 & 0.48 \\
$1\times$
    & 75.80 $\pm$ 0.55 & 8
    & 50.24 $\pm$ 0.54 & 4
    & 89.29 $\pm$ 0.31 & 1
    & 84.62 $\pm$ 0.28 & 0.24 \\
\bottomrule
\end{tabular}
\end{center}
\vspace{-0.4cm}
\end{table}

\subsubsection{Effect of the Dirichlet Concentration Parameter}

We conduct additional experiments on CIFAR-10 to evaluate ASFL under different degrees of data heterogeneity.
Specifically, we use Dirichlet concentration parameters $\alpha = 0.05$ and $\alpha = 0.5$, corresponding to stronger and weaker non-IID conditions, respectively, relative to the original setting of $\alpha = 0.1$.

\begin{table}[h]
\caption{Validation accuracy (\%) on CIFAR-10 with different
Dirichlet concentration parameters. Smaller $\alpha$ corresponds
to stronger data heterogeneity.}
\label{tab:dirichlet}
\begin{center}
\scriptsize
\begin{tabular}{lcc}
    \toprule
    Method & $\alpha = 0.05$ & $\alpha = 0.5$ \\
    \midrule
    FedAvg & $58.93 \pm 0.3$ & $85.77 \pm 0.3$ \\
    FedProx & $56.16 \pm 0.1$ & $86.20 \pm 0.1$ \\
    FedProto & $59.44 \pm 1.0$ & $86.86 \pm 0.2$ \\
    MOON & $58.69 \pm 1.1$ & $86.58 \pm 0.1$ \\
    SCAFFOLD & $55.37 \pm 1.8$ & $87.59 \pm 0.1$ \\
    FedPVR & $58.45 \pm 0.9$ & $86.27 \pm 0.1$ \\
    FedVARP & $61.14 \pm 0.4$ & $85.93 \pm 0.1$ \\
    FedImpro & $50.88 \pm 0.7$ & $83.11 \pm 0.5$ \\
    SplitFed & $58.21 \pm 0.5$ & $85.75 \pm 0.2$ \\
    \rowcolor{gray!30} \textbf{ASFL (Ours)} & $\mathbf{67.14} \pm 0.5$ & $\mathbf{87.60} \pm 0.1$ \\
    \bottomrule
\end{tabular}
\end{center}
\vspace{-0.3cm}
\end{table}

As shown in Table~\ref{tab:dirichlet}, ASFL still achieves strong validation accuracy under more severe non-IID conditions ($\alpha = 0.05$).
It achieves $67.14\%$ accuracy, exceeding FedAvg and the strongest baseline, FedVARP, by $8.21$ and $6.00$ percentage points, respectively.
These results are consistent with our original findings and support the effectiveness of ASFL under strong data heterogeneity.

Under weaker non-IID conditions ($\alpha = 0.5$), all methods achieve higher validation accuracy, and the performance gaps across methods are substantially smaller, leaving limited additional benefit from ASFL in this setting.
Nevertheless, ASFL retains the highest validation accuracy at $87.60\%$.
This pattern is consistent with the motivation of ASFL's activation-sharing mechanism, which is designed to mitigate client drift caused by data heterogeneity: its benefit is more pronounced when heterogeneity is severe and diminishes as client data distributions become more similar.

\subsubsection{Isolated performance gain from activation sharing}
\label{appendix:activation-sharing-ablation}

Attributing ASFL's accuracy gains to activation sharing requires separating its contribution from that of partial training.
To make this distinction, we evaluate a variant on CIFAR-10 that retains ASFL's partial training procedure, including the same 16 tail update steps in each tail round, but replaces the received remote activations with the client's own local activations.

\begin{table}[ht]
\caption{Ablation study on activation sharing on CIFAR-10. Both ASFL variants use the same partial training procedure with 16 tail update steps per tail round.}
\label{tab:activation_sharing_ablation}
\begin{center}
\scriptsize
\begin{tabular}{lc}
    \toprule
    Method & Validation accuracy (\%) \\
    \midrule
    FedAvg & $67.70 \pm 0.33$ \\
    ASFL (without remote activations) & $66.23 \pm 0.41$ \\
    ASFL (with remote activations) & $75.27 \pm 0.88$ \\
    \bottomrule
\end{tabular}
\end{center}
\vspace{-0.3cm}
\end{table}

Replacing remote activations with local ones reduces validation accuracy from $75.27\%$ to $66.23\%$, a drop of $9.04$ percentage points despite using the same partial training procedure and tail update budget (Table~\ref{tab:activation_sharing_ablation}).
The local-only variant also falls below FedAvg ($67.70\%$) by $1.47$ percentage points.
Thus, partial training alone does not explain ASFL's improvement over FedAvg in this setting.
the gain depends on incorporating remote activations.
\end{document}